\documentclass{article}
\PassOptionsToPackage{numbers,sort&compress}{natbib}

\usepackage[preprint]{neurips_2026}

\usepackage[utf8]{inputenc} %
\usepackage[T1]{fontenc}    %
\usepackage{hyperref}       %
\usepackage{url}            %
\usepackage{booktabs}       %
\usepackage{amsfonts}       %
\usepackage{nicefrac}       %
\usepackage{microtype}      %
\usepackage{xcolor}         %

\usepackage{amsmath}
\usepackage{amssymb}
\usepackage{mathtools}
\usepackage{amsthm}
\usepackage{multirow}
\usepackage{makecell}
\usepackage{caption}
\usepackage{subcaption}
\usepackage{wrapfig}
\usepackage{longtable}
\usepackage{tcolorbox}
\usepackage{arydshln}
\usepackage{siunitx}
\usepackage{siunitx}
\usepackage{enumitem}
\usepackage[whole]{bxcjkjatype}
\usepackage{CJKutf8}
\newcommand{\Chinese}[1]{{\begin{CJK*}{UTF8}{gbsn}#1\end{CJK*}}}

\usepackage{amsmath,amsfonts,bm}

\def\eqref#1{equation~\ref{#1}}

\def\1{\bm{1}}

\def\vh{{\bm{h}}}

\def\vx{{\bm{x}}}

\def\mA{{\bm{A}}}

\def\mH{{\bm{H}}}

\def\mK{{\bm{K}}}
\def\mL{{\bm{L}}}

\def\mW{{\bm{W}}}

\DeclareMathAlphabet{\mathsfit}{\encodingdefault}{\sfdefault}{m}{sl}
\SetMathAlphabet{\mathsfit}{bold}{\encodingdefault}{\sfdefault}{bx}{n}

\def\gG{{\mathcal{G}}}

\def\gT{{\mathcal{T}}}

\def\sA{{\mathbb{A}}}
\def\sB{{\mathbb{B}}}

\def\sK{{\mathbb{K}}}
\def\sL{{\mathbb{L}}}

\def\sS{{\mathbb{S}}}

\def\sV{{\mathbb{V}}}

\def\emA{{A}}

\def\emW{{W}}

\newcommand{\R}{\mathbb{R}}

\newcommand{\parents}{Pa} %

\DeclareMathOperator*{\argmax}{arg\,max}

\usepackage{todonotes}

\theoremstyle{plain}
\newtheorem{theorem}{Theorem}[section]
\newtheorem{proposition}[theorem]{Proposition}

\theoremstyle{definition}

\theoremstyle{remark}

\title{StructLens: A Structural Lens for Language Models via Maximum Spanning Trees}

\author{%
Haruki Sakajo
\quad
Frederikus Hudi
\quad
Yusuke Sakai
\AND
Hidetaka Kamigaito
\quad
Taro Watanabe\\\\
Nara Institute of Science and Technology (NAIST)\\
\texttt{sakajo.haruki.sd9@naist.ac.jp}
\\
\texttt{\{frederikus.hudi.fe7,sakai.yusuke.sr9,kamigaito.h,taro\}@is.naist.jp}
}

\newcommand\structlensname{\textsc{StructLens}}

\begin{document}

\maketitle

\begin{abstract}
Language exhibits inherent structures, a property that explains both language acquisition and language change.
Given this characteristic, we expect language models to manifest their own internal structures as well.
While interpretability research has investigated how models compute representations mechanistically through attention patterns and Sparse Autoencoders, the organization of the resulting representations is overlooked.
To address this gap, we introduce \textsc{StructLens}, a framework to analyze representations through a holistic structural view.
\textsc{StructLens} constructs maximum spanning trees based on the semantic representations in residual streams, inspired by tree representation in dependency parsing, and provides summaries of token relationships in representation space.
We analyze how contiguous tokens are also nearby in representation space and find that middle layers show the strongest local-span organization.
Moreover, analysis of pre-training checkpoints reveals that smaller local units become detectable earlier in pre-training, and larger units later.
Our findings demonstrate that \textsc{StructLens} provides insights into how models organize token representations across layers and training.
Our code is available at \url{https://github.com/naist-nlp/structlens}.
\end{abstract}

\section{Introduction}
Language possesses structure.
Linguistic phenomena, such as language acquisition and language change, have been explained through underlying structural frameworks~\citep{chomsky1962ss, TOMASELLO2005, bybee2006, Bybee_2010}.
Given language's structural nature, we expect that language models (LMs), which are designed to computationally model language, should similarly exhibit their own structural properties~\citep{lee-etal-2025-geometric}.

While language exhibits such structural properties, research on LMs, e.g., interpretability and pruning, has frequently overlooked these structures.
For example, existing interpretability tools primarily analyze individual tokens or features, e.g., logit lens~\citep{logitlens} and Sparse Autoencoders (SAEs)~\citep{huben2024sparse}.
Similarly, cosine similarity that is employed for inter-layer analysis~\citep{men-etal-2025-shortgpt, jiang2025tracing} is fundamentally based on token-to-token comparisons at corresponding positions, making it challenging to capture the holistic structural pattern formed within specific layers.
To facilitate global analysis of representations at each layer, approaches that incorporate inter-token relationships and provide comprehensive structural insights are expected to make valuable contributions to LM analysis.

Several studies have utilized parsing techniques developed in Natural Language Processing (NLP) to conduct inter-layer analysis based on inter-token relationships from a linguistic, particularly generative linguistic, perspective.
These investigations have demonstrated that attention weights reflect syntactic structures~\citep{raganato-tiedemann-2018-analysis, clark-etal-2019-bert, ravishankar-etal-2021-attention, zhang2025the}, representations encode syntactic information~\citep{hewitt-manning-2019-structural, andreas2018measuring, li-eisner-2019-specializing, murty2023characterizing, hudi-etal-2024-disentangling}, and the syntactic structures emerge in a bottom-up manner~\citep{someya-etal-2025-derivational}.
Although these studies have revealed that LMs possess and utilize structures, their focus has centered on static, generative grammatical structures that presuppose a certain ground truth structure synthesized by linguists.
However, given that language exhibits dynamic structures~\citep{TOMASELLO2005, bybee2006, Bybee_2010} formed through bottom-up processes and LMs are unable to introspect their internal mechanisms, the approaches of bottom-up construction and analysis should be more appropriate to assess the LMs' own internal structure.

\begin{figure}
    \centering
    \includegraphics[width=\linewidth]{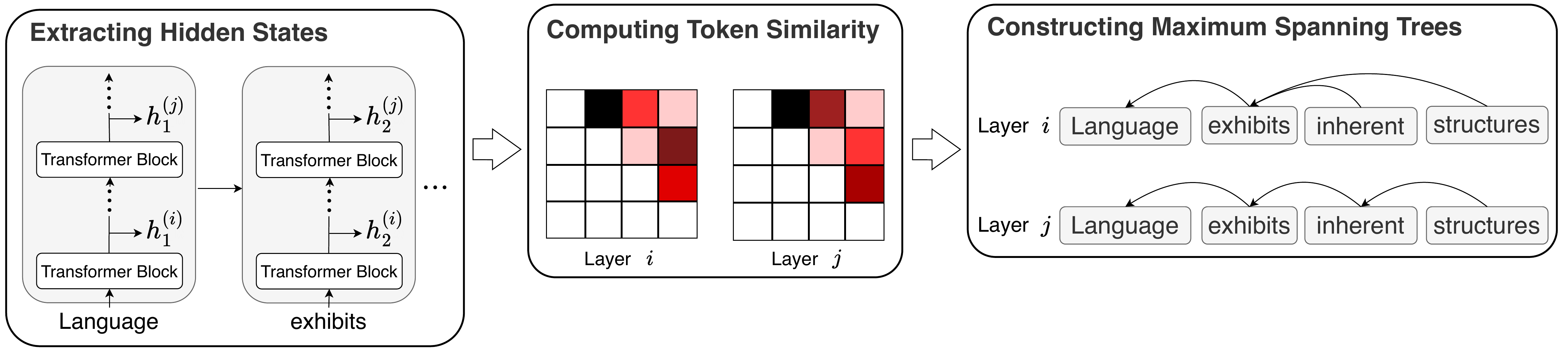}
    \caption{Overview of \structlensname. \structlensname\ extracts hidden states and then computes inter-token similarity. Finally, it constructs Maximum Spanning Trees using the adjacency similarity matrices.}
    \label{fig:overview}
\end{figure}

To address this gap, we propose \structlensname, a framework that constructs Maximum Spanning Trees (MSTs), i.e., a tree structure connecting all the nodes in a graph with the maximum total edge weight, using LM internal representations, analogous to those employed in dependency parsing in the NLP field~\citep{eisner-1996-three, yamada-matsumoto-2003-statistical, mcdonald-etal-2005-online,  mcdonald-etal-2005-non}.
Although graph structures are more general, arbitrary complex structures are often difficult to interpret.
On the other hand, tree structures provide interpretable representations of high-dimensional data~\cite{Probst2020}.
We therefore leverage tree structures to analyze LM representations.
Our approach analyzes residual streams at each layer's output, computing L2 distance between token semantic representations to construct an MST at each layer (see Figure~\ref{fig:overview}).
\structlensname\ provides an order-constrained tree summary of token relationships in representation space.
To analyze the trees, we introduce the local-span connectivity metrics that quantify whether connected tree fragments are concentrated on contiguous token intervals.
\structlensname\ finds that LMs exhibit a middle-layer local-span signature, where smaller local spans spread across tokens earlier in pre-training before consolidating into denser, larger-span structures during later steps of pre-training.
We also provide a pilot study showing that structure-aware metrics computed over \structlensname\ improve layer pruning in some settings.
Our findings highlight that structure-aware perspectives are potentially useful for LM analysis and optimization.

\section{Background and Related Work}
\label{sec:background}

\paragraph{Structures in language.}
In the study of language, researchers have assumed that language possesses structures, conceived as either static, e.g., generative grammar \citep{chomsky1962ss}, or dynamic, e.g., usage-based theory \citep{TOMASELLO2005, bybee2006}.
Traditional generative grammar, i.e., transformational grammar, assumes formal rules, while usage-based approaches hypothesize that instances of use shape language representations, allowing analysis of gradual language change.

\paragraph{Residual stream.}
The Transformer~\citep{vaswaniAttentionAllYou2017} updates internal representations gradually by utilizing residual connections.
This work assumes a variant of the Transformer with pre-layer normalization architecture~\citep{xiong-etal-onlayer}, which forms a \emph{residual stream}~\citep{mathematicalframework}.
Formally, given the input features of length $n$, let $d$ be a hidden dimension and $f_\theta^{(\ell)}(\cdot): \R^{n \times d} \rightarrow \R^{n \times d}$ be the Transformer block at layer $\ell$ that comprises Multi-Head Attention and Multi-Layer Perceptron blocks. The hidden state of the input at the $\ell$-th layer is referred to as the residual stream $\mH^{(\ell)} \in \R^{n \times d}$, defined as follows:
\begin{equation}
\label{eq:residual_connection}
    \mH^{(\ell + 1)} \,=\, \mH^{(\ell)} + f_\theta^{(\ell + 1)}(\mH^{(\ell)}).
\end{equation}
Residual streams in LMs have provided insights into both interpretability work~\citep{kamigaito2025diversitytransformerlayersaspect, minegishi2025topology, viswanathan2026the, li2026on} and layer pruning methods~\citep{yang-etal-2024-laco, men-etal-2025-shortgpt}.

\paragraph{Mechanistic interpretability for language models.}
Mechanistic interpretability is an interpretable framework, employing bottom-up methods to reveal models' computational processes and behavior~\citep{bereska2024mechanistic}.
Research on LM interpretability has examined both activations on the residual stream and the modules that transform them (e.g., Multi-Head Attention, Multi-Layer Perceptron), uncovering the nature of encoded information and functions~\citep{olsson2022incontextlearninginductionheads, kobayashi2024analyzing, rai2025practicalreviewmechanisticinterpretability, cheng2025emergence}.
Research on mechanistic interpretability has also identified models' computational circuits, which reflect underlying behaviors of LMs~\citep{ameisen-etal-circuit, NEURIPS2023_efbba771, marks2025sparse}.
Logit lens~\citep{logitlens} is a technique to analyze intermediate states by projecting intermediate representations into a vocabulary space through the final prediction layer.
Previous studies trained probes and evaluated whether targeted information (e.g., syntactic trees) is encoded in representations~\citep{tenney2018what, hewitt-manning-2019-structural, andreas2018measuring, hall-maudslay-etal-2020-tale, stanczak-etal-2022-neurons, brinkmann-etal-2025-large}. Sparse Autoencoders (SAEs) are used to identify interpretable features and causal circuits within models, addressing the challenge of superposition, where representations exhibit polysemantic properties~\citep{huben2024sparse, brinkmann-etal-2025-large, hanna-mueller-2025-incremental}.
Building on these approaches, we focus on inter-token relationships within individual layers and construct tree structures for each layer, enabling us to provide global views beyond token-level interpretations.

\section{Method}
\label{sec:method}
The primary objective of \structlensname\ is to analyze Transformer layers in terms of the organization of token representations in the representation space from a holistic perspective.
Given the residual stream representations of all tokens in an input sequence, we first construct a complete graph whose edge weights reflect pairwise representation similarity.
Since arbitrary complex structures of the graph are often hard to analyze and the edge weights vary across layers (see Appendix~\ref{appendix:token-sim}), we derive a different graph, i.e., a Maximum Spanning Tree (MST), from each layer of the Transformer.
The resulting MST can be interpreted as a structure of the layer-wise representation geometry and enables comparison of the structures across layers.

Formally, let $\vh^{(\ell)}_i$ denote the residual stream of the $i$-th token immediately after layer $\ell$. 
For the layer $\ell$, given an input token sequence $\vx$ of length $n$, we first define a pairwise similarity matrix $\mW^{(\ell)} \in \mathbb{R}^{n \times n}$ as:
\begin{equation}
    \label{eq:symmetric_similarity}
    \emW^{(\ell)}_{i,j}
    =
    \begin{cases}
        \dfrac{1}{1 + \lVert \vh_i^{(\ell)} - \vh_j^{(\ell)} \rVert}, & \text{if } i \neq j \\[6pt]
        0 & \text{otherwise}.
    \end{cases}
\end{equation}
This matrix is symmetric and defines a complete graph $\gG$ without self-loops, where each node corresponds to a token in $\vx$ and each edge encodes a relation between two tokens.
We use the reciprocal to convert a distance into a similarity for the sake of numerical stability; details are provided in Appendix~\ref{appendix:token-sim}.

Given that tokens have a canonical order and autoregressive LMs have left-to-right information flow, we define a forward-constrained adjacency matrix $\mA^{(\ell)} \in \mathbb{R}^{n \times n}$ as:
\begin{equation}
\label{eq:forward_projection}
    \emA^{(\ell)}_{i,j}
    =
    \begin{cases}
        \emW^{(\ell)}_{i,j} & \text{if } i < j, \\[6pt]
        0 & \text{otherwise}.
    \end{cases}
\end{equation}
The forward constraint has two useful consequences, formalized below: it preserves the undirected pairwise similarities in $\mW^{(\ell)}$ because $\mW^{(\ell)} = \mA^{(\ell)} + {\mA^{(\ell)}}^\top$, and it makes the directed MST equivalent to selecting the most similar preceding token for each non-root token under a unique-predecessor assumption.
\begin{proposition}
\label{proposition:nearest-predecessor}
    Under the forward constraint, StructLens selects the most similar predecessor for each non-root token.
\end{proposition}
\begin{proof}
    Under the forward constraint, each non-root token $i$ chooses one preceding token $j < i$ as its parent.
    Thus, any such parent selection forms a single-root spanning tree.
    
    Let $\parents(i)$ denote the parent token of token $i$.
    Since the total weight of the tree is $\sum^n_{i=2}\mA_{\parents(i), i}$, and the choice of $\parents(i)$ is independent of the choice of $\parents(k)$ for $k \neq i$, this sum is maximized independently for each token $i$.
    Thus, $\parents(i) \in \argmax_{j < i} \mA_{j, i}.$
    
    Therefore, the directed MST under the forward constraint is equivalent to selecting the most similar preceding token for each non-root token.
\end{proof}

Under the forward constraint, \structlensname\ has a particularly simple interpretation.
Each token selects the most similar predecessor in representation space, producing an acyclic directed graph.
This makes \structlensname\ scalable and easy to interpret.

For each layer, we build a single-root, order-respecting MST on $\mA^{(\ell)}$.
Although the optimal tree can be obtained by selecting the highest-weight predecessor, we use the algorithm introduced by \cite{Tarjan1977}, which runs in $\mathcal{O}(n^2)$ time for a dense graph and is based on Chu-Liu/Edmonds' algorithm~\cite{chuliu1965shortest, edmonds1967optimum}, as it provides a unified implementation for the present setting and for future \structlensname\ variants with other edge constraints.

\structlensname\ summarizes the geometry of token residual-stream representations as an order-constrained directed tree.
Each edge connects a token to a similar preceding token, yielding an interpretable backbone of local and non-local representational relationships within a layer.
For simplicity, we refer to this tree as an MST.

\section{Analyzing Layers through \structlensname}
\label{sec: layer-analysis}

\subsection{Local-span Connectivity in Pretrained Language Models}
\label{sec:chunk_analysis}
\begin{wrapfigure}{r}{0.3\textwidth}
    \centering
    \vspace{-1.0em}
    \includegraphics[width=\linewidth]{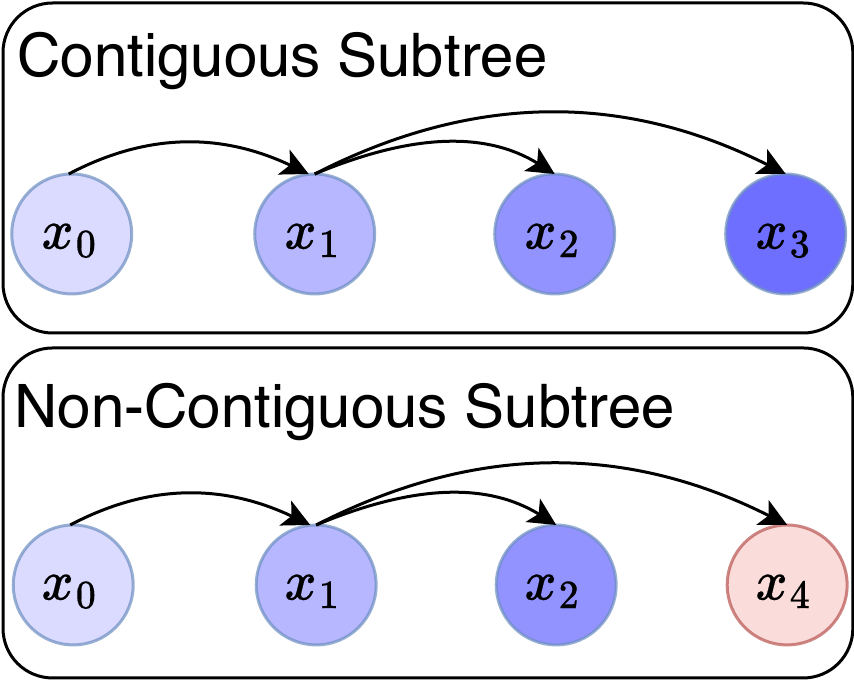}
    \vspace{-0.5em}
    \caption{Contiguous and non-contiguous 4-node subtrees.}
    \label{fig:contiguous_subtree_sample}
    \vspace{-1.0em}
\end{wrapfigure}
Inspired by context-dependent units seen in linguistic structures for humans, i.e., chunks,
we study whether \structlensname\ trees exhibit local-span connectivity, in which connected subtrees have token indices that form contiguous intervals.
We use this as an operational measure of local organization in representation space.
Specifically, we focus on the $k$-node subtrees spanning over contiguous position tokens without any skips, as shown in Figure~\ref{fig:contiguous_subtree_sample}, and we refer to them as \emph{contiguous subtrees}.

We measure the \emph{contiguous subtree ratio}, which quantifies the density of contiguous subtrees, indicating how strongly the $k$-node subtrees concentrate on local token spans, and the \emph{contiguous token ratio}, which measures the coverage of such tokens.
These ratios characterize token relationship trees in representation space with respect to local connectivity.

Formally, a $k$-node subtree is a connected induced subgraph of a tree $\gT$ containing $k$ nodes.
A contiguous $k$-node subtree is a $k$-node subtree whose token indices form an interval $\{a, a+1, \dots, a + k - 1\}$.
Let $C^k_{cont}(\gT)$ and $C^k_{all}(\gT)$ denote the number of contiguous $k$-node subtrees of a tree $\gT$ and the total number of $k$-node subtrees of $\gT$, respectively.
The \emph{contiguous subtree ratio} of $\gT$ is defined as:
$\frac{C^k_{cont}(\gT)}{C^k_{all}(\gT)}$.
We also investigate the number of tokens covered by contiguous subtrees.
Let $\mathcal{S}_k(\mathcal{T})$ be the set of tokens in the contiguous $k$-node subtrees of $\mathcal{T}$.
Let $\mathcal{N}_k(\mathcal{T}) = \bigcup_{S \in \mathcal{S}_k(\mathcal{T})} V(S)$ be the set of node positions appearing in any such subtree, where $V(S)$ denotes the node set of subtree $S$.
The \emph{contiguous token ratio} is defined as:
$\frac{|\mathcal{N}_k(\gT)|}{n}$.

\begin{wrapfigure}{r}{0.35\textwidth}
    \centering
    \vspace{-1.0em}
    \includegraphics[width=\linewidth]{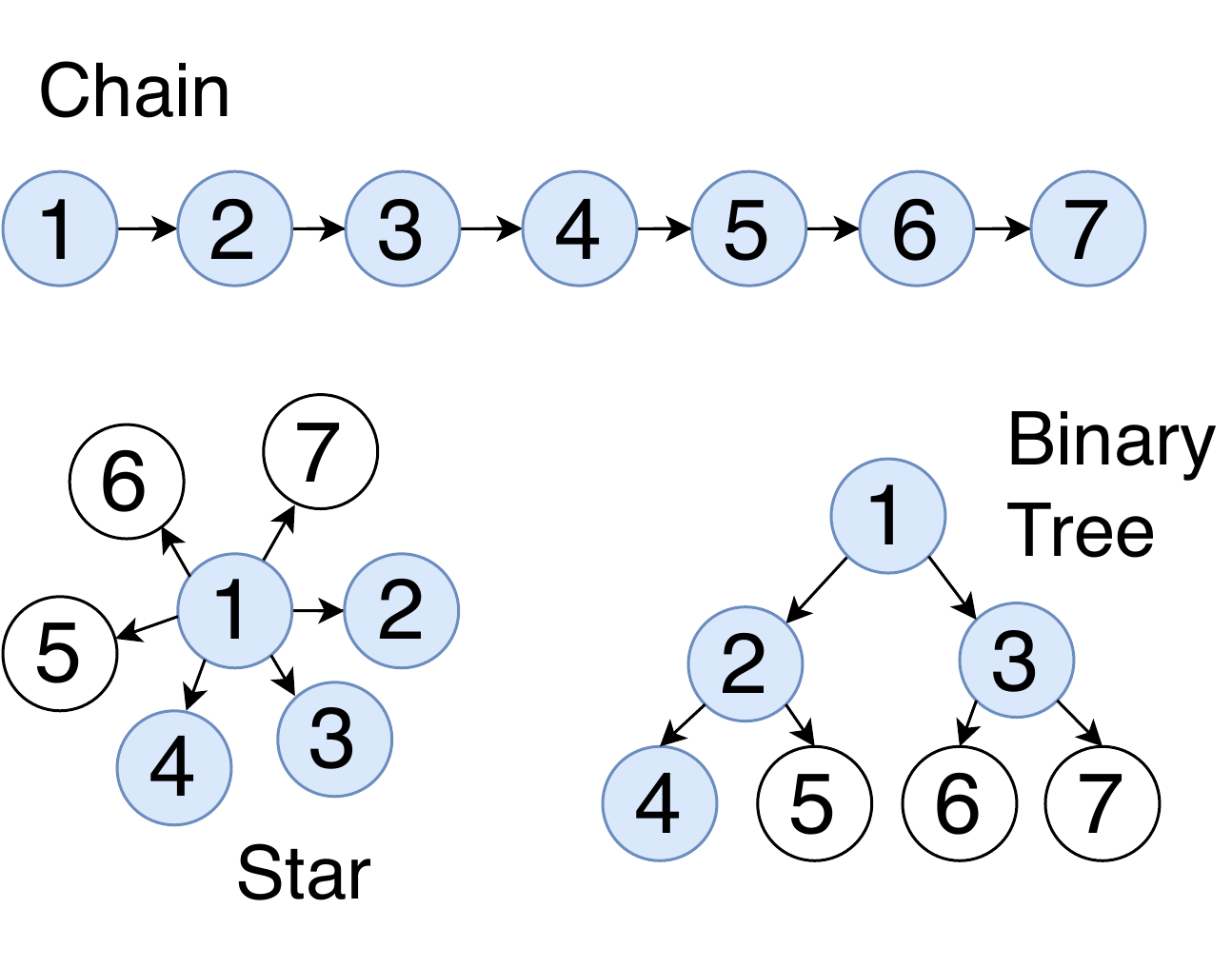}
    \vspace{-0.5em}
    \caption{Illustration of trees. Colored nodes form contiguous 4-node subtrees.}
    \label{fig:tree_sample}
    \vspace{-1.0em}
\end{wrapfigure}
To illustrate the difference between the contiguous subtree ratio and the contiguous token ratio, we consider 4-node contiguous subtrees in three simple 7-node trees, as illustrated in Figure~\ref{fig:tree_sample}.
We present rooted trees by strict S-expressions, where the first element is the parent and the following elements are its children.
For the chain tree $(1(2(3(4(5(6(7)))))))$, the contiguous subtree ratio and the contiguous token ratio are 1.
For the star tree $(1(2)(3)(4)(5)(6)(7))$, the total number of 4-node subtrees is 20, the total number of contiguous subtrees is 1, and the total number of contiguous tokens is 4.
Therefore, the contiguous subtree ratio is $\frac{1}{20}$ and the contiguous token ratio is $\frac{4}{7}$.
For the binary tree $(1(2(4)(5))(3(6)(7)))$, the total number of 4-node subtrees is 6, and the total number of contiguous subtrees is 1.
Therefore, the contiguous subtree ratio is $\frac{1}{6}$ and the contiguous token ratio is $\frac{4}{7}$.

Furthermore, to determine whether contiguous subtrees are derived from implicit positional information or patterns dependent on natural token order, we compare the original input order with the shuffled input order.
We also shuffle input tokens at the token or word n-gram levels to examine whether models consider or ignore word boundaries and semantic coherence.
If LMs exhibit contiguous subtrees only when the input is ordered as in the original order, this indicates that they are sensitive to natural token order beyond the tested shuffled baseline.

\paragraph{Experimental settings.}
\label{sec: layer-analysis-exp-settings}
We employ Llama3.1 8B~\citep{grattafiori2024llama3herdmodels} and Qwen2.5 7B~\citep{qwen2025qwen25technicalreport} for our experiments.
The evaluation datasets are MMLU~\citep{hendrycks2021measuring}, which is an English multiple-choice Question-Answering dataset with four choices, and MultiNews~\citep{fabbri-etal-2019-multi}, which is a summarization dataset in English.
We randomly sample instances from each dataset and employ prompt templates with five-shot examples from the development set of each dataset, as used in the MMLU paper and LongBench~\citep{bai-etal-2024-longbench} for MultiNews.
As a case study, we investigate 4-node contiguous subtrees.
We also test 3- and 5-node contiguous subtrees on MMLU.
A detailed description of the experimental settings is provided in Appendix~\ref{appendix:experimental-settings}.
When shuffling input tokens, we keep the position of the begin-of-sentence token and shuffle other tokens. 
We run the shuffling experiment with three seeds \{0, 1, 2\} and report the average.

\begin{figure}
    \centering
    \includegraphics[width=\linewidth]{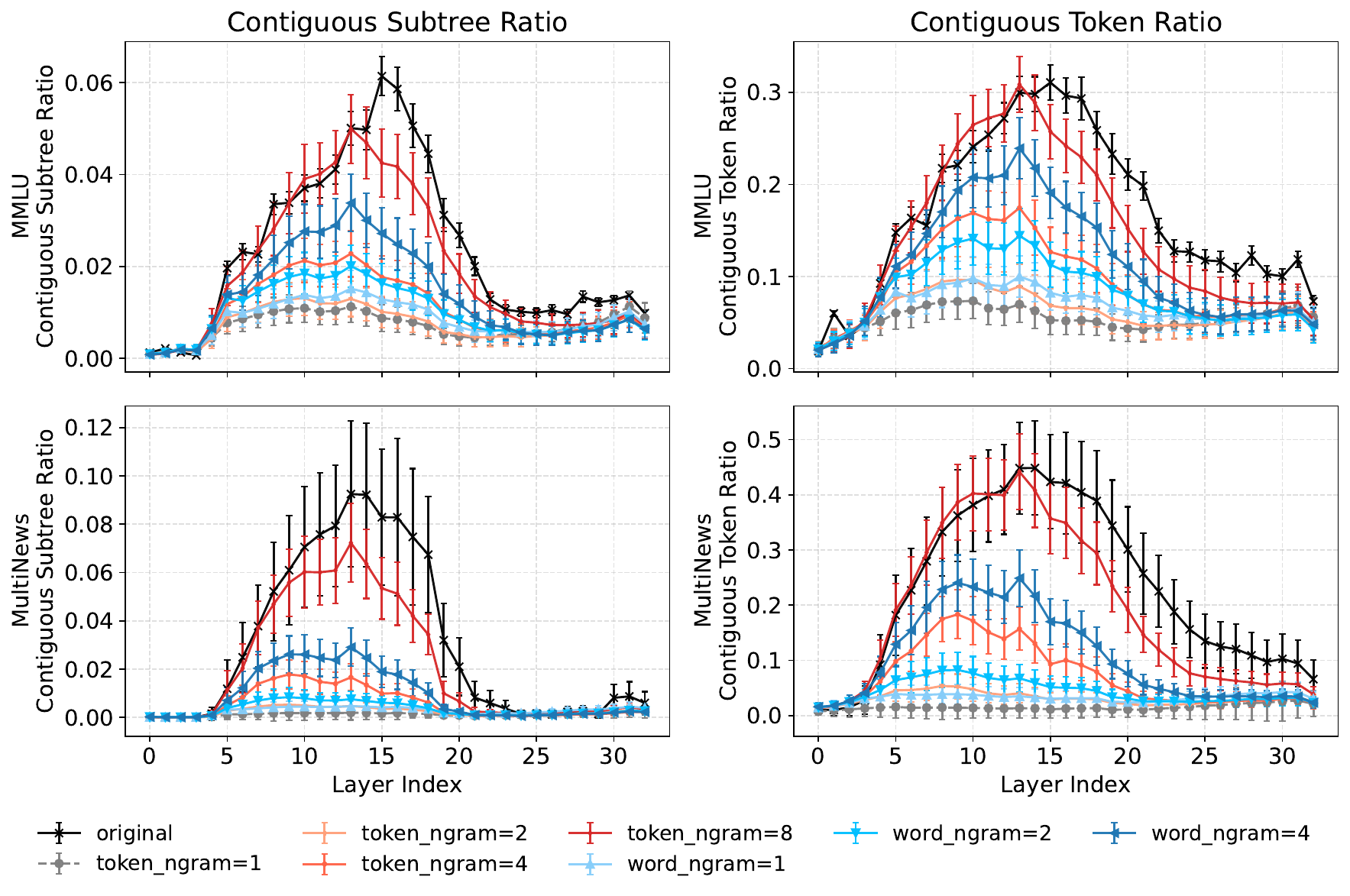}
    \caption{Visualization of the layer-wise evolution of the contiguous 4-node subtrees and tokens in them for the original input order and the shuffled one for Llama3.1 8B. The x-axis denotes layer depth, and the y-axis denotes the contiguous subtree ratio and the contiguous token ratio.}
    \label{fig:contiguous_subtree_llama}
\end{figure}

\paragraph{Results.}
Figure~\ref{fig:contiguous_subtree_llama} shows the contiguous subtree and token ratios for all layers by averaging over the samples.
This figure reveals that, from lower to middle layers, i.e., 0\% to 50\% of the layer percentage, both ratios become larger, indicating that contiguous local structures become more frequent and more widely distributed across tokens.
In higher layers, the ratios decrease, suggesting that this local organization is reduced.
This pattern, where the intermediate layers behave differently from other layers, has also been observed in geometric metrics and embedding tasks~\cite{hosseini2023large, skean2025layer}, supporting that the structures derived by \structlensname\ reflect meaningful properties of layer-wise representations.%
We also conducted the same experiments on MMLU and multinews with the original input order for 3-node and 5-node subtrees for both Llama3.1 8B and Qwen2.5 7B in Appendix~\ref{appendix:contiguous-subtree-3-5-node}, which show similar trends to those observed for 4-node subtrees.
This indicates that this middle-layer pattern is not only for 4-node subtrees but also for another local-span pattern.
Shuffling results show that models organize fewer contiguous subtrees at the smaller unit levels, indicating that the contiguous subtrees observed in the original setting depend on the original token order.
These results suggest that LMs organize token representations into contiguous structures in intermediate layers before distributing them or forming non-local structures in higher layers.

\subsection{Training Dynamics through \structlensname}
\label{sec:training-dynamics}

During pre-training, downstream task performance improves gradually, while mid-training yields substantial performance gains across several tasks~\citep{blakeney2024does, walsh2025}.
We further examine when LMs acquire the behavior in the middle layers during pre-training and mid-training by measuring contiguous subtree and token ratios across multiple training checkpoints, revealing how and when they do so.

\paragraph{Experimental settings.}
We employ OLMo2 7B~\citep{walsh2025} and its checkpoints during pre-training.
We use checkpoints at steps \{1,000, 10,000, 101,000, 500,000, 750,000, 928,646\} of stage 1, and at 1,000 steps of one of the checkpoints of stage 2 called Mid-training~\cite{walsh2025}.
We refer to them as stage1-1k, stage1-10k, stage1-101k, stage1-500k, stage1-750k, stage1-928k, and stage2-1k, respectively.
The evaluation dataset is MMLU, and we compute the contiguous subtree and token ratios for the 3-, 4-, and 5-node subtrees, and test for the original input order and the shuffled order.
We use the same settings as reported in Section~\ref{sec:chunk_analysis}.

\begin{figure}[t]
    \centering
    \includegraphics[width=\linewidth]{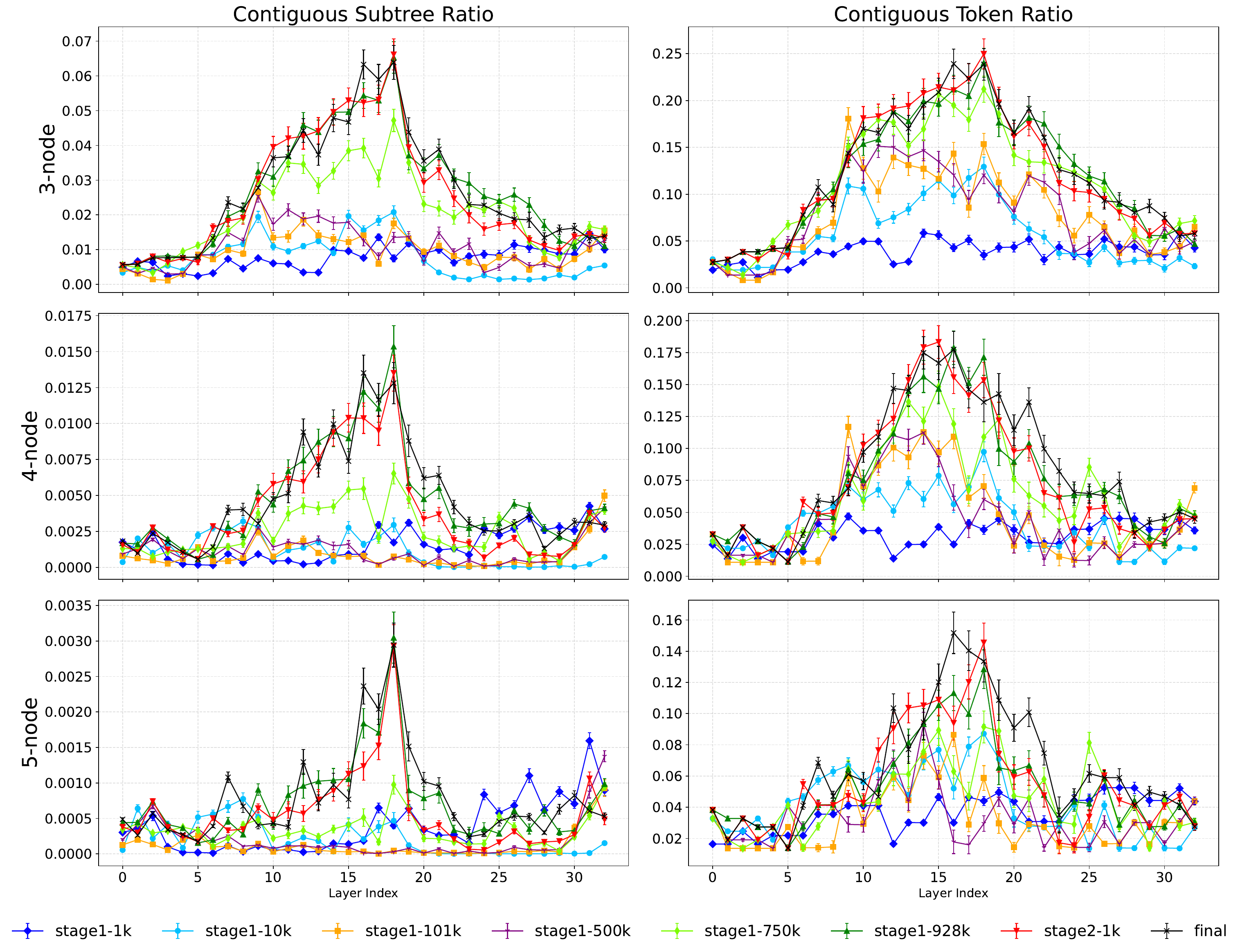}
    \caption{Visualization of the layer-wise evolution of the 3-, 4-, and 5-node contiguous subtrees and tokens in them across multiple checkpoints of OLMo2 7B for MMLU. In each plot, the x-axis denotes the layer index, and the y-axis denotes the contiguous subtree and token ratio.}
    \label{fig:contiguous_subtree_olmo}
\end{figure}

\paragraph{Results.}
Figure~\ref{fig:contiguous_subtree_olmo} shows the contiguous subtree and token ratios for multiple checkpoints as reported in Section~\ref{sec:chunk_analysis}.
The shuffling results are reported in Appendix~\ref{appendix:contiguous-subtree-olmo-shuffled}.
This figure shows that the subtree ratio in the middle layers differs between later checkpoints (stage1-928k and stage2-1k) and earlier ones, while the model behaves similarly across checkpoints in the lower layers.
By analyzing different subtree sizes $k \in \{3, 4, 5\}$, \structlensname\ reveals a scale-dependent emergence of local-span connectivity.
Smaller subtrees become detectable earlier in training, especially in the token-coverage metric, whereas larger contiguous subtrees emerge later in the middle layers.
This suggests that the model first develops short-range local connectivity, and later training consolidates this connectivity into denser, larger-span tree fragments.
The contiguous token ratio also suggests that such structures are distributed broadly across tokens at earlier steps.
Notably, this emergence is not observed in training, optimizer, or downstream-task metrics\footnote{Available at\\\url{https://wandb.ai/ai2-llm/OLMo-2-1124-7B/reports/OLMo-2-7B-Nov-2024--VmlldzoxMDUzMzE1OA}.}.
Comparing the structural signal with loss, optimizer, and downstream metrics can be an important direction of future work.

\subsection{Token Organization in Language Models through \structlensname}

Layer-wise analysis shows that the contiguous local structures become more dominant in the middle layers and are reduced in the higher layers.
Training dynamics further show that this behavior emerges at later checkpoints.
The contiguous token ratio increases during earlier steps of pre-training, indicating that contiguous local structures become more broadly distributed across tokens.
In contrast, the contiguous subtree ratio exhibits a sharper increase at later checkpoints, indicating that such structures become more dominant among tree fragments.
These trends suggest a two-stage process during pre-training: first, smaller contiguous structures spread across the sequence, and later, they become more concentrated and larger within the representation geometry.
In summary, \structlensname\ reveals how token representations are organized across layers and during pre-training.

\section{Inter-Layer Similarity through \structlensname}

\subsection{Measuring Inter-Layer Similarity through \structlensname}
For analyzing layer redundancy in LMs, established methods, e.g., cosine similarity, are employed to quantify layer similarity~\citep{jiang2025tracing, men-etal-2025-shortgpt}.
These conventional approaches measure similarity between representations at corresponding positions, capturing local pairwise relationships.
However, these methods lack a global perspective encompassing intra-layer token relationships and do not provide a holistic view of layer-level interaction.
In this study, we compute layer similarity with three structure-aware similarity metrics to measure comprehensive and holistic relationships using \structlensname.

\paragraph{Cosine Similarity for \structlensname~(Cos-Struct).}

For each subtree of depth~2, we compute the average of the hidden representations of the parent and its children, yielding a flattened subtree of depth~1. This process is applied recursively until only a single representation remains at the root node. Let $\sK_i$ be the set of child nodes for $i$, defined as $\sK_i = \{\, j \in \{1,\dots,n\} \,\mid\, \parents(j) = i \,\}$. The aggregated representation by averaging at node $i$ is defined recursively as:
\begin{equation}
    \bar{\vh}_i = \frac{1}{|\sK_i|+1} \left( \frac{\vh_i}{\|\vh_i\|} + \sum\nolimits_{j \in \sK_i} \bar{\vh}_j \right).
\end{equation}
The aggregated representation at the root node of layer $\ell$ is denoted by $\bar{\vh}^{(\ell)}$. The structural similarity between two layers, $\ell_a$ and $\ell_b$, is then measured by the cosine similarity of their aggregated root representations:
\begin{equation}
\label{eq: cos-struct}
    \text{score}_\text{Cos-Struct}(\ell_a, \ell_b) = 
    \cos\!\left(\bar{\vh}^{(\ell_a)}, \bar{\vh}^{(\ell_b)}\right).
\end{equation}
Cos-Struct operates in $O(n)$ time.
Although Cos-Struct incorporates structural aggregation, it still does not directly measure the structural similarity between trees induced by \structlensname.

\paragraph{Tree Edit Distance~(Tree-Edit).}
The Tree Edit Distance~\cite{zhangshasha} has been widely applied and studied as a method for quantifying dissimilarity between two ordered labeled trees~\citep{benjamin2022revisitingtreeeditdistance}. Here, we explore its utility as a negative similarity score.

For a given graph $\gG$, let $\gT_a$ denote the \structlensname\ maximum spanning tree corresponding to layer $\ell_a$. Let $\mathcal{P}(\gT_a,\gT_b)$ be the set of edit scripts that transform $\gT_a$ into $\gT_b$, and let $c(o)$ denote the cost for an edit operation $o$ in an edit script of $\pi$.
The Tree-Edit score is defined as:
\begin{equation}
\label{eq: ted}
    \text{score}_\text{Tree-Edit}(\ell_a,\ell_b) =
    -\left( \min_{\pi \in \mathcal{P}(\gT_a,\gT_b)} \sum\nolimits_{o \in \pi} c(o) \right).
\end{equation}
Tree-Edit operates in $O(n^4)$ time.
Tree-Edit, however, is unable to move an entire subtree since such changes require deletion and insertion operations recursively for all nodes and edges in the subtree.

\paragraph{Edge Edit Distance~(Edge-Edit).}
We employ a more straightforward edge-based edit distance metric, which mitigates score variations caused by the subtree movement between layers, providing a direct and more stable structural comparison. %
We define the edge-set $\sS_{\gT}$ of a tree $\gT$ as the set of parent-child pairs.
Let $r$ be the root node, and let $\parents_{T}\left(i\right)$ denote the parent node of any node $i$ other than the root $r$. The edge-set is then given by:
\begin{equation}
\label{eq: edge-set}
    \sS_{\gT} \,=\, \{\, \left(i,\parents_{T}(i)\right) \;\mid\; i \in \{1,\dots,n\},\; i \neq r \,\}.
\end{equation}
Let $\gT_a$ and $\gT_b$ be the spanning trees corresponding to layer $\ell_a$ and $\ell_b$, respectively, and let $\sS_{\gT_a}$ and $\sS_{\gT_b}$ be the respective edge-set as defined in Equation \ref{eq: edge-set}.
As the two trees have the same set of nodes and the same number of edges, the Edge-Edit score equals their edge~set difference:
\begin{equation}
\label{eq: edge-edit}
    \text{score}_\text{Edge-Edit}(\ell_a,\ell_b) = 
    -\left( |\sS_{\gT_a} \backslash \sS_{\gT_b}| + |\sS_{\gT_b} \backslash \sS_{\gT_a}| \right).
\end{equation}
This metric directly counts edge insertions and deletions, avoids inflated costs from subtree movements, and more stably measures structural similarity across layers in $O(n)$ time.

\subsection{Non-Structural Metrics}

To investigate how structural metrics capture inter-layer similarity, we compare them with non-structural metrics, i.e., Centered Kernel Alignment~\citep{pmlr-v97-kornblith19a} and Cosine Similarity.

\paragraph{Centered Kernel Alignment (CKA).}
Centered Kernel Alignment~(CKA) can be used to compute global inter-layer similarity.
We compute CKA using the unbiased estimator of the Hilbert-Schmidt Independence Criterion~(HSIC)~\citep{song-supervised-feature}.
Formally, the inter-layer similarity with CKA is defined as:
\begin{equation}
\label{eq:cka}
    \text{score}_\text{CKA}(\ell_a, \ell_b) = \frac{\text{HSIC}(\mK, \mL)}{\sqrt{\text{HSIC}(\mK, \mK) \text{HSIC}(\mL, \mL)}},
\end{equation}
where $\mK = \mH^{(\ell_a)}{\mH^{(\ell_a)}}^\top$ and $\mL = \mH^{(\ell_b)}{\mH^{(\ell_b)}}^\top$ denote the linear Gram matrices. To mitigate the statistical bias caused by the finite sample size $n$, we employ the unbiased estimator of HSIC~\cite{song-supervised-feature}.

\paragraph{Cosine similarity~(Cos-Base).}
Cosine similarity is a widely used metric for comparing vector representations.
We measure inter-layer similarity using Cos-Base as:
\begin{equation}
\label{eq: cos-sim}
    \text{score}_\text{Cos-Base}(\ell_a, \ell_b) = \sum\nolimits_i^n \cos\! \left(\vh_i^{(\ell_a)}, \vh_i^{(\ell_b)}\right).
\end{equation}

\begin{figure}[t]
    \centering
    \begin{subfigure}[b]{0.16\textwidth}
        \includegraphics[width=\textwidth,clip,trim={0 0 24em 0}]{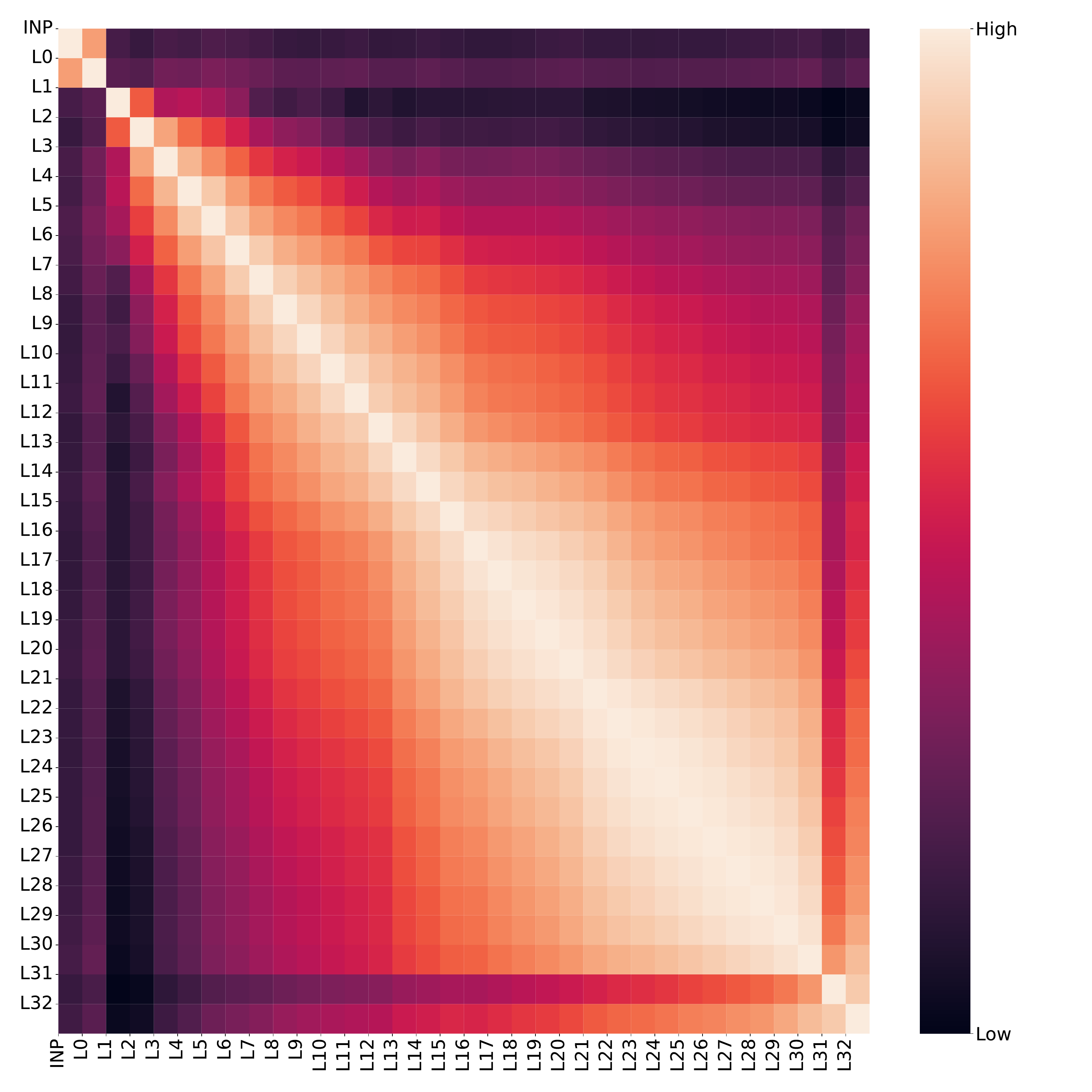}
        \caption{CKA}
    \end{subfigure}
    \begin{subfigure}[b]{0.16\textwidth}
        \includegraphics[width=\textwidth,clip,trim={0 0 24em 0}]{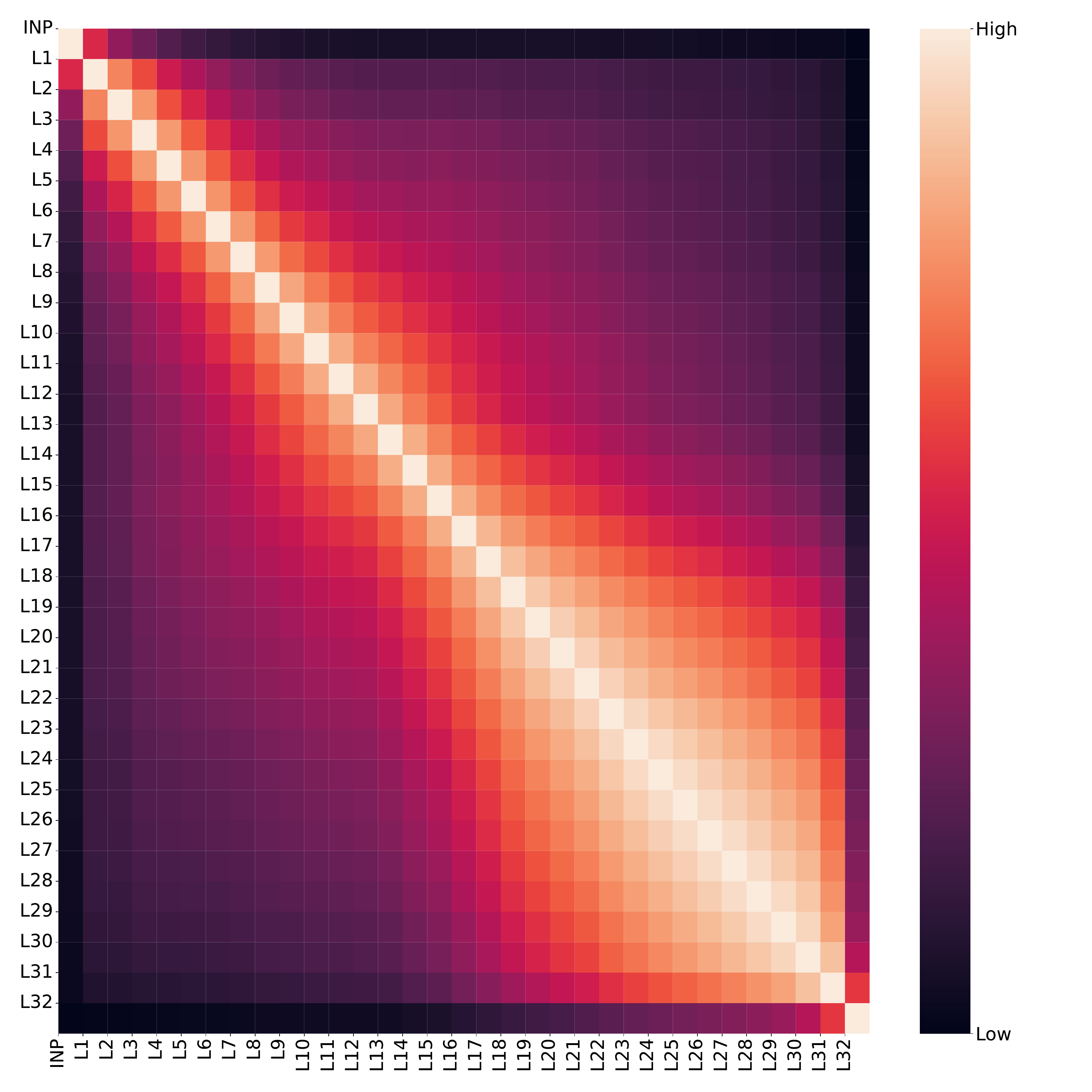}
        \caption{Cos-Base}
    \end{subfigure}
    \begin{subfigure}[b]{0.16\textwidth}
        \includegraphics[width=\textwidth,clip,trim={0 0 24em 0}]{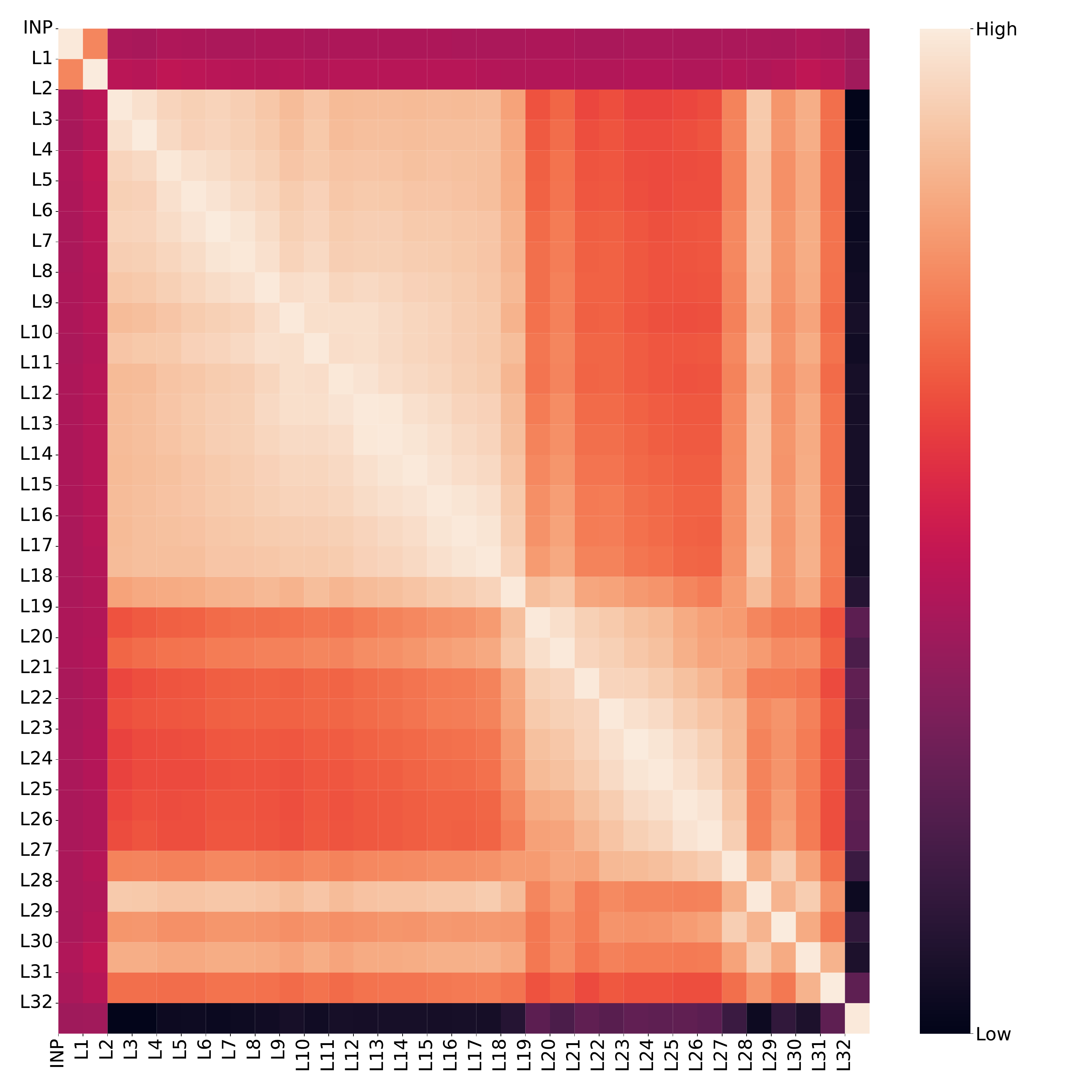}
        \caption{Cos-Struct}
    \end{subfigure}
    \begin{subfigure}[b]{0.16\textwidth}
        \includegraphics[width=\textwidth,clip,trim={0 0 24em 0}]{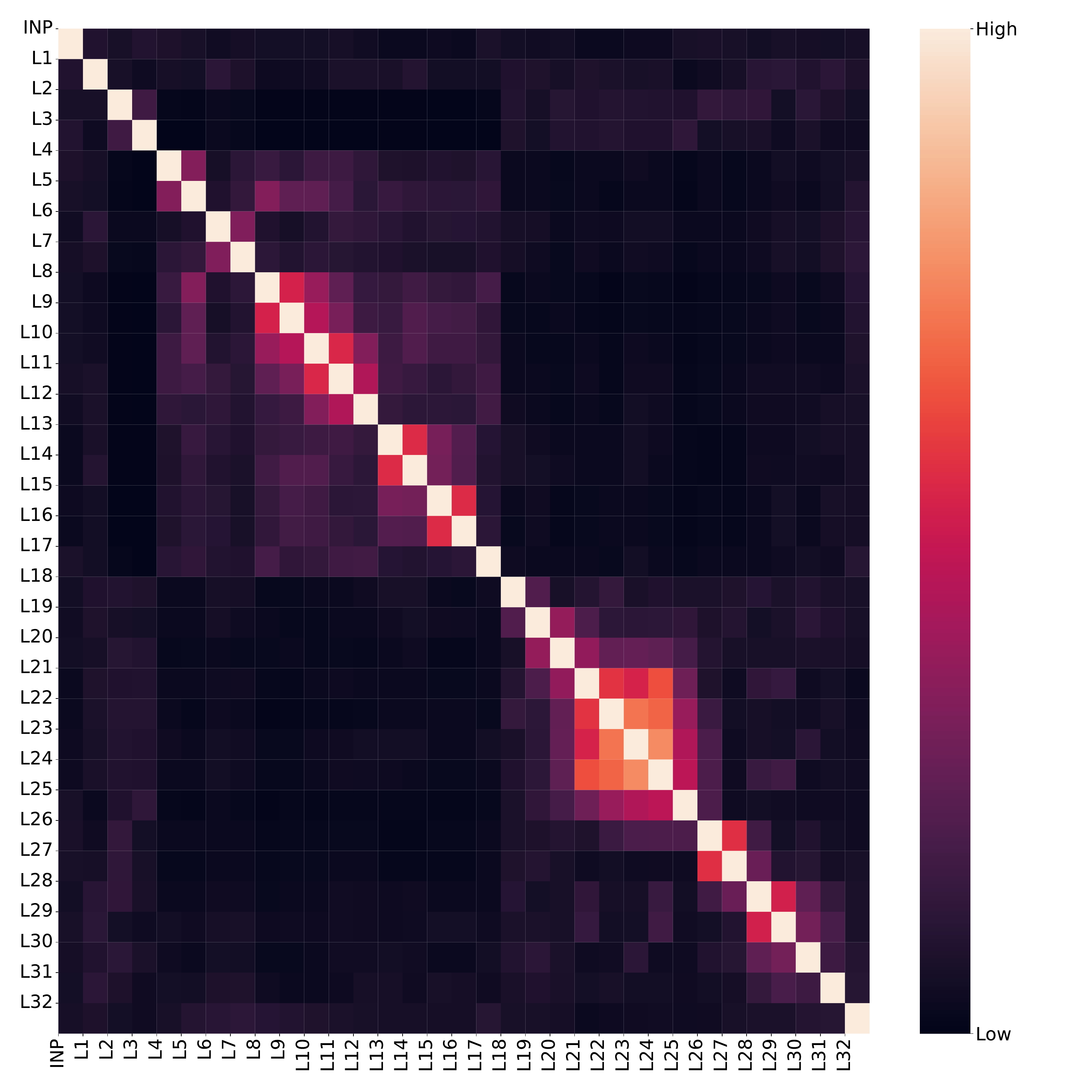}
        \caption{Tree-Edit}
    \end{subfigure}
    \begin{subfigure}[b]{0.16\textwidth}
        \includegraphics[width=\textwidth,clip,trim={0 0 24em 0}]{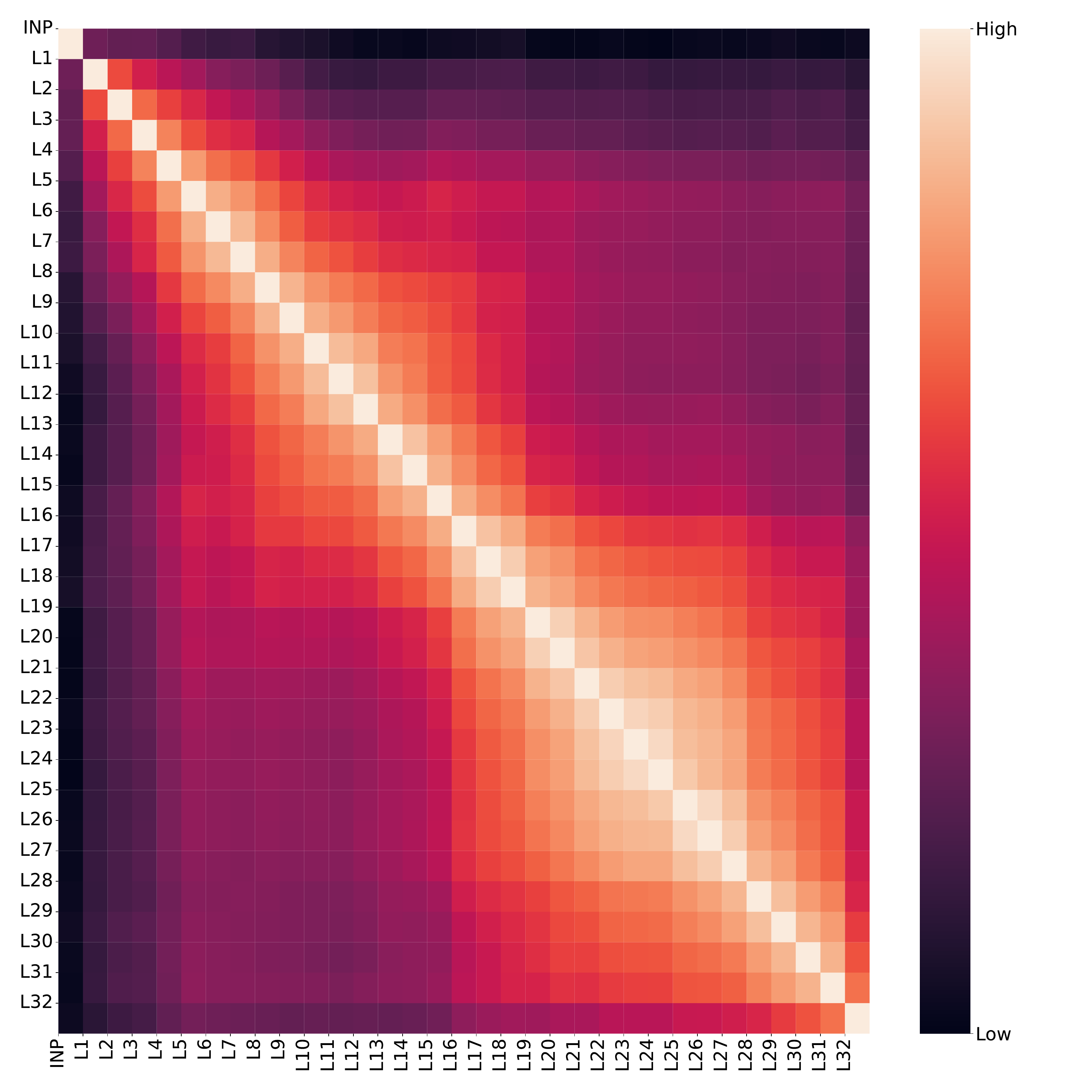}
        \caption{Edge-Edit}
    \end{subfigure}
    \caption{Inter-layer similarity samples of Llama3.1 8B for each metric on MMLU. Bright color represents high similarity, while dark color represents low similarity.}
    \label{fig:sim_mmlu_llama}
\end{figure}

\subsection{Inter-Layer Similarity Analysis}
\label{sec:layer-sim-experiment}
We analyze language models using the similarity of tree structures across layers obtained via \structlensname.
We apply \structlensname\ to representations of each sampled instance in the datasets and then compute inter-layer similarity using Equations~\ref{eq:cka}, \ref{eq: cos-sim}, \ref{eq: cos-struct}, \ref{eq: ted}, and \ref{eq: edge-edit}.

\paragraph{Experimental settings.}
We employ Llama3.1 8B~\citep{grattafiori2024llama3herdmodels} and Qwen2.5 7B~\citep{qwen2025qwen25technicalreport} for our experiments.
The evaluation datasets are MMLU~\citep{hendrycks2021measuring}, which is a multiple-choice Question-Answering dataset with four choices, and its Chinese equivalent, CMMLU~\citep{li-etal-2024-cmmlu}.
We randomly sample instances and compute inter-layer similarity.
Details of experimental settings are provided in Appendix~\ref{appendix:experimental-settings}.

\paragraph{Results.}
The inter-layer similarity for Llama3.1 8B on MMLU is illustrated in Figure~\ref{fig:sim_mmlu_llama}, where the x-axis and y-axis represent the layer indices.
This figure shows that Edge-Edit exhibits diagonal clustering patterns, forming discrete groupings characterized by high inter-layer similarity, which we refer to as islands, and that these patterns are observed with the k-NN metric~\cite{wolfram2025layers}.
We report the similarity for CMMLU and Qwen2.5 7B in Appendix~\ref{appendix:layer-sim}, where Edge-Edit also exhibits diagonal clustering patterns, and these islands remain consistent across model family and size (see Appendix~\ref{appendix:layer-sim}).
We evaluate the clustering consistency across samples and the clustering quality in Appendix~\ref{appendix:clustering_consistency}.
The results show that the clusterings produced by these metrics are consistent with $k = 3$ or $k = 4$, except for Tree-Edit.
We investigate these islands through frequent subtree mining and logit lens in Appendices~\ref{appendix:freqt} and ~\ref{appendix:logitlens-structlens}, which suggest that islands correspond to meaningful groupings in the LMs.

\subsection{Layer Pruning through \structlensname}
\label{sec: layer-pruning}

As an exploratory application of \structlensname, we test whether \structlensname-derived layer similarity can provide signals for layer pruning.

\paragraph{Layer pruning algorithm.}
Layer pruning algorithms for Transformer LMs~\citep{yang-etal-2024-laco,men-etal-2025-shortgpt, gromov2025the} identify and prune layers that produce relatively small modifications to representations based on their similarity across layers.
This approach leverages the residual connections in the Transformer as described in Equation~\ref{eq:residual_connection} and quantifies layer importance to determine layers for removal.
We employ the metric introduced in ShortGPT~\cite{men-etal-2025-shortgpt} and assess layer influence through \structlensname.

\paragraph{Layer influence.}
ShortGPT~\citep{men-etal-2025-shortgpt} computes layer influence (importance) using inter-layer cosine similarity, and subsequently removes layers from the model in ascending order of importance.
In ShortGPT, the $i$-th layer influence, referred to as Block Influence (BI), is defined using Equation~\ref{eq: cos-sim} as:
\begin{align}
\label{eq:cos_bi}
    \text{CosBase-BI}_i &= 1 - \text{score}_{\text{Cos-Base}}(\ell_i, \ell_{i-1})
\end{align}
Additionally, we calculate influence using another non-structural metric, CKA (Equation~\ref{eq:cka}), and three structure-aware \structlensname\ similarity metrics, namely Cos-Struct (Equation~\ref{eq: cos-struct}), Tree-Edit (Equation~\ref{eq: ted}), and Edge-Edit (Equation~\ref{eq: edge-edit}).
CKA-BI, CosStruct-BI, TreeEdit-BI, and EdgeEdit-BI are computed similarly to Equation~\ref{eq:cos_bi} by replacing the Cos-Base score with CKA, Cos-Struct, Tree-Edit, and Edge-Edit, respectively.
Since the score ranges for TreeBI and EdgeBI depend on inputs, we normalize them on a per-sample basis, taking into account the theoretical bounds of each metric.

\paragraph{Experimental settings.}
We employ Llama3.1 8B Instruct and Qwen2.5 7B Instruct.
For each model, we remove approximately 10\% of layers in ascending order of BI scores according to each metric.
For evaluation datasets, we use ARC-Easy~\citep{clark2018thinksolvedquestionanswering}, MMLU, and CMMLU as Question-Answering datasets, GSM8K~\citep{cobbe2021trainingverifierssolvemath} as a QA dataset with multi-step reasoning, and Multinews and VCSUM~\citep{wu-etal-2023-vcsum} as summarization datasets.
We measure accuracy for the QA tasks and Rouge-L F1~\citep{lin-2004-rouge} for summarization tasks.
We employ McNemar's test~\citep{McNemar1947} to assess the statistical significance of differences in accuracy, and the paired bootstrap test for ROUGE-L F1.
Layer removal calibration utilizes 10 samples from the English Wikipedia dataset~\citep{wikidump}, as performed in LaCo~\cite{yang-etal-2024-laco}.
Details of experimental settings are provided in Appendix~\ref{appendix:experimental-settings}.

\paragraph{Results.}
The removed layers are presented in Table~\ref{tab:result-pruning}, and we show the layer importance of each layer in Appendix~\ref{appendix:pruning-result-pretrained}.
The layer-pruning results, as shown in Table~\ref{tab:result-pruning}, indicate that tree-aware metrics outperform non-structural metrics for Llama3.1 8B Instruct, suggesting that structural metrics are helpful for layer pruning in this setting.
While each metric shows inconsistent results across datasets for Qwen2.5 7B Instruct, CosStruct-BI exhibits better or comparable performance to the baselines on tasks that require generation, i.e., GSM8K, Multinews, and VCSUM.
We also evaluate pre-trained models on multiple QA tasks, and these models exhibit a similar tendency to instruct-tuned models, as shown in Appendix~\ref{appendix:pruning-result-pretrained}.
These results demonstrate that there is room to improve layer-pruning metrics and that structural metrics potentially provide insights into them.

\begin{table}[t]
\centering
\caption{Pruning results. Each value denotes the accuracy or the correctness of an answer in a task, with higher values indicating better performance. Values of each dataset result denoted by $\dagger$ are statistically significant ($p < 0.05$) compared to CosBase-BI.}
\label{tab:result-pruning}
\resizebox{\linewidth}{!}{
\begin{tabular}{llrrrrrrr}
\toprule
Metric & Layers removed & ARC-Easy & MMLU & CMMLU & GSM8K & Multinews & VCSUM & Avg.  \\
\midrule
\multicolumn{3}{l}{Llama3.1 8B Instruct} \\
\midrule
Dense & - & 85.8\:\: & 68.4\:\: & 55.9\:\: & 70.4\:\: & 26.1\:\: & 16.6\:\: & 53.8 \\
\cmidrule{1-9} \vspace{-1em} \\
CosBase-BI & 24 25 26 27 & \underline{78.3}\:\: & 66.8\:\: & 54.4\:\: & \underline{43.0}\:\: & \underline{23.1}\:\: & 8.7\:\: & 45.7 \\
CKA-BI & 24 25 26 31 & 75.3$^{\dagger}$ & 66.9\:\: & 54.2\:\: & 5.8$^{\dagger}$ & 18.2$^{\dagger}$ & 1.9$^{\dagger}$ & 37.0 \\
\cdashline{1-9} \vspace{-.8em} \\
CosStruct-BI & 23 24 25 26 & \textbf{79.3}\:\: & \textbf{67.8$^{\dagger}$} & \underline{55.1$^{\dagger}$} & \textbf{57.1$^{\dagger}$} & \textbf{25.1$^{\dagger}$} & \textbf{16.2$^{\dagger}$} & \textbf{50.1} \\
TreeEdit-BI & 23 24 26 27 & 78.1 & \underline{66.9}\:\: & \textbf{55.8$^{\dagger}$} & 40.9$^{\dagger}$ & 22.7\:\: & \underline{10.7$^{\dagger}$} & \underline{45.8} \\
EdgeEdit-BI & 23 24 25 26 & \textbf{79.3}\:\: & \textbf{67.8$^{\dagger}$} & \underline{55.1$^{\dagger}$} & \textbf{57.1$^{\dagger}$} & \textbf{25.1$^{\dagger}$} & \textbf{16.2$^{\dagger}$} & \textbf{50.1} \\
\midrule
\multicolumn{3}{l}{Qwen2.5 7B Instruct} \\
\midrule
Dense & - & 87.6\:\: & 74.3\:\: & 81.1\:\: & 80.4\:\: & 23.6\:\: & 17.7\:\: & 60.8 \\
\cmidrule{1-9} \vspace{-1em} \\
CosBase-BI & 15 16 17 & \underline{83.7}\:\: & 56.1\:\: & 56.8\:\: & \underline{35.9}\:\: & 21.6\:\: & \textbf{17.5}\:\: & \underline{45.3} \\
CKA-BI & 4 25 27 & 80.1$^{\dagger}$ & \textbf{70.0$^{\dagger}$} & \textbf{75.7$^{\dagger}$} & 26.3$^{\dagger}$ & 12.2$^{\dagger}$ & 10.4$^{\dagger}$ & \textbf{45.8} \\
\cdashline{1-9} \vspace{-.8em} \\
CosStruct-BI & 13 14 15 & 82.8\:\: & 55.7\:\: & 55.9\:\: & \textbf{37.5}\:\: & \textbf{23.1$^{\dagger}$} & 16.2\:\: & 45.2 \\
TreeEdit-BI & 13 16 17 & 82.5$^{\dagger}$ & \underline{59.7$^{\dagger}$} & \underline{63.1$^{\dagger}$} & 17.1$^{\dagger}$ & 20.0$^{\dagger}$ & 15.7$^{\dagger}$ & 43.0 \\
EdgeEdit-BI & 14 15 16 & \textbf{84.2}\:\: & 53.4$^{\dagger}$ & 54.0$^{\dagger}$ & 31.1$^{\dagger}$ & \underline{22.8$^{\dagger}$} & \underline{16.9}\:\: & 43.7 \\
\bottomrule
\end{tabular}
}
\end{table}

\section{Limitation}
\label{sec:limitation}
\paragraph{Comparison with linguistic chunks.}
In Section~\ref{sec:chunk_analysis}, while we analyze subtrees inspired by chunks, which are properties of language, the analysis is limited to observing the subtree features rather than comparing them with linguistic chunks because chunks in language depend on language users.
However, our primary goal is to analyze how LMs organize input tokens in the representation space at each layer; our findings provide insights into this.

\paragraph{Token similarity metrics.}
This study measures inter-token similarity using reciprocal L2 distance.
While other metrics, e.g., cosine similarity, can be used to measure it, our objective is to analyze the organization of token representations in Euclidean space, where the magnitude carries meaningful information.
Therefore, we employ the L2 distance-based metric for \structlensname\ as a first step.
Investigating other metrics can be a future direction for expanding \structlensname.

\paragraph{Forward constraint.}
Given the left-to-right information flow of autoregressive LMs, we apply the forward constraint before constructing MSTs, which may limit the analysis of tree features, e.g., the root node.
However, as we demonstrate in Section~\ref{sec:method}, this constraint preserves pairwise token similarity.
Rather, this constraint makes the tree construction more efficient and yields a simple tree, making the tree causal order-respecting and easier to analyze.

\paragraph{Layer pruning results.}
The layer pruning results show that the structural metrics are inconsistently robust across datasets and models, limiting their applicability to model optimization.
These results also indicate that there is room to improve layer pruning metrics and that structural metrics may provide insights into them.

\section{Conclusion}
\label{sec:conclusion}
Our experimental results reveal that \structlensname\ trees exhibit a middle-layer increase in local-span connectivity under natural token order.
This pattern is reduced under token shuffling, suggesting that the observed connectivity is sensitive to natural sequence organization rather than arising solely from the tree construction.
Pre-training checkpoint analysis reveals the evolution of models' token representation organization in the representation space.
Although, as discussed in Section~\ref{sec:limitation}, there are still limitations, e.g., layer pruning, our findings demonstrate that \structlensname\ provides valuable insights into interpretability and model analysis and has the potential to expand research in this field. 

\section*{Acknowledgement}
We thank the anonymous reviewers and the area chair for their valuable comments and suggestions.
This work was supported by JST BOOST, Grant Number JPMJBS2423.

{
\small
\bibliographystyle{unsrt}
\bibliography{ref}
}

\newpage

\appendix

\section{Inter-Token Similarity}
\label{appendix:token-sim}

We compute inter-token L2 distance for Llama3.1 8B and Qwen2.5 7B using 10 samples in the Wikipedia dataset as a preliminary study.
Figure~\ref{fig:avg_token_dist} shows that the average in a higher layer of Qwen2.5 7B reaches approximately 600.
Given this result, we convert distance into similarity by reciprocal transformation rather than exponentiation to avoid numerical underflow, as used in Equation~\ref{eq:symmetric_similarity}.
\begin{figure}[ht]
    \centering
    \begin{subfigure}[b]{0.45\linewidth}
        \includegraphics[width=\textwidth]{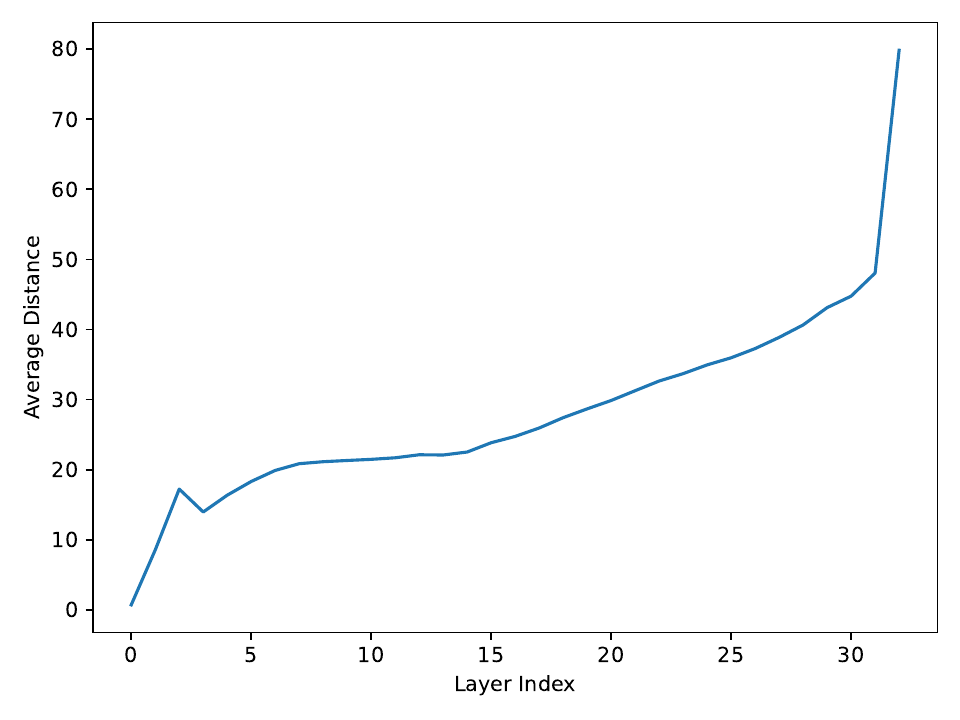}
    \caption{Llama3.1 8B}
    \end{subfigure}
    \begin{subfigure}[b]{0.45\linewidth}
        \includegraphics[width=\textwidth]{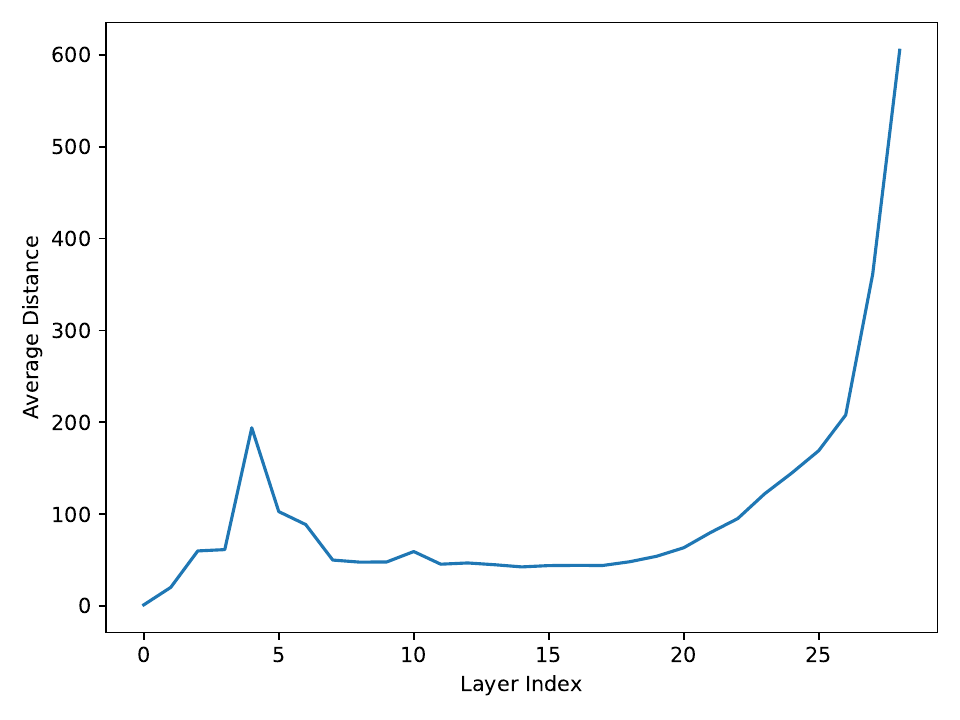}
    \caption{Qwen2.5 7B}
    \end{subfigure}
    \caption{Average L2 distance between tokens.}
    \label{fig:avg_token_dist}
\end{figure}

\section{Experimental Settings (Detail)}
\label{appendix:experimental-settings}

\paragraph{Datasets and Prompts.}
We randomly sample 100 instances from each dataset and use the following prompt template with five-shot examples for each dataset, using the development set samples for MMLU and CMMLU.
For Multinews and VCSUM, we use the prompts used in LongBench~\citep{bai-etal-2024-longbench}.

\begin{tcolorbox}[title=MMLU, boxrule=1pt]
The following are multiple-choice questions about \{subject\}. Respond with either A, B, C, or D as your answer.

\{Question of Example1\}

(A) \{Choice A of Example1\}

(B) \{Choice B of Example1\}

(C) \{Choice C of Example1\}

(D) \{Choice D of Example1\}

Answer: \{Answer of Example1\}

...

\{Question\}

(A) \{Choice A\}

(B) \{Choice B\}

(C) \{Choice C\}

(D) \{Choice D\}

Answer:
\end{tcolorbox} 

\begin{tcolorbox}[title=CMMLU, boxrule=1pt]
\Chinese{以下是关于（" \{subject\} "）的单项选择题，请直接给出正确答案的选项。}

\{Question of Example1\}

(A) \{Choice A of Example1\}

(B) \{Choice B of Example1\}

(C) \{Choice C of Example1\}

(D) \{Choice D of Example1\}

Answer: \{Answer of Example1\}

...

\{Question\}

(A) \{Choice A\}

(B) \{Choice B\}

(C) \{Choice C\}

(D) \{Choice D\}

Answer:
\end{tcolorbox}

\begin{tcolorbox}[title=Multinews, boxrule=1pt]
    You are given several news passages. Write a one-page summary of all news.
    
    News:
    
    \{context\}
    
    Now, write a one-page summary of all the news.
    
    Summary:
\end{tcolorbox}

\begin{tcolorbox}[title=VCSUM, boxrule=1pt]
    \Chinese{下面有一段会议记录，请你阅读后，写一段总结，总结会议的内容。}
    
    \Chinese{会议记录：}
    
    \{context\}
    
    \Chinese{会议总结：}
\end{tcolorbox}

\paragraph{Implementations.}
In this study, we use models and datasets via HuggingFace, and Tables~\ref{tab:model_names} and \ref{tab:dataset_names} show the HuggingFace IDs of each model and dataset, respectively.
Table~\ref{tab:checkpoints} provides model revisions for each checkpoint of OLMo2 7B.
To run Llama3.1 70B and Qwen2.5 72B, we quantize models to 4bit with QLoRA~\citep{dettmers2023qlora} via Transformers~\citep{wolf-etal-2020-transformers}.
We use Transformers to use LMs with PyTorch~\citep{pytorch}, TensorFlow Text to build MSTs, scikit-learn~\citep{scikit-learn} for spectral clustering and computing ARI, and SciPy~\citep{2020SciPy-NMeth} to compute correlation.
To run ShortGPT~\cite{men-etal-2025-shortgpt}, we employ its official implementation.
Hyperparameters used in experiments are provided in Table~\ref{tab:hyperparameters}.
We use a single NVIDIA GeForce RTX 3090 GPU, a single NVIDIA A100-SXM4-40GB GPU, one or two NVIDIA RTX A6000 or NVIDIA RTX 6000 Ada Generation GPUs for extracting hidden states.
In the layer pruning experiments, we use lm-evaluation-harness~\citep{eval-harness} and vLLM~\cite{vllm} on a single NVIDIA H100 GPU or a single NVIDIA RTX Pro 6000 GPU.
Code is provided in the supplement materials, and we will release the \structlensname\ code as an installable Python library.

\begin{table}[t]
\centering
\caption{HuggingFace ID and Hyperparameters.}
\begin{subtable}[t]{0.45\textwidth}
    \centering 
    \caption{Models}
    \label{tab:model_names}
    \resizebox{\linewidth}{!}{
        \begin{tabular}{ll}
        \toprule
           Model  & HuggingFace ID \\
           \midrule
           Llama3.1 8B  & meta-llama/Llama-3.1-8B  \\
           Llama3.1 8B Instruct  & meta-llama/Llama-3.1-8B-Instruct  \\
           Llama3.1 70B  & meta-llama/Llama-3.1-70B  \\
           Qwen2.5 7B  & Qwen/Qwen2.5-7B \\
           Qwen2.5 7B Instruct  & Qwen/Qwen2.5-7B-Instruct \\
           Qwen2.5 72B  & Qwen/Qwen2.5-72B \\
           OLMo2 7B & allenai/OLMo-2-1124-7B \\
        \bottomrule
        \end{tabular}
        }
\end{subtable}
\hfill
\begin{subtable}[t]{0.5\textwidth}
    \centering 
    \small
    \caption{Datasets}
    \label{tab:dataset_names}
        \begin{tabular}{ll}
        \toprule
           Dataset  & HuggingFace ID \\
           \midrule
           MMLU  & cais/mmlu \\
           CMMLU  & lmlmcat/cmmlu \\
           Wikipedia & wikimedia/wikipedia \\
           Multinews & zai-org/LongBench \\
           VCSUM & zai-org/LongBench \\
        \bottomrule
        \end{tabular}
\end{subtable}
\hfill
\begin{subtable}[t]{0.45\textwidth}
    \centering 
    \caption{Checkpoints}
    \label{tab:checkpoints}
    \resizebox{\linewidth}{!}{
        \begin{tabular}{ll}
        \toprule
            Checkpoint & Revision (HuggingFace) \\
             \midrule
            stage1-1k & stage1-step1000-tokens5B \\
            stage1-10k & stage1-step10000-tokens42B \\
            stage1-101k & stage1-step101000-tokens424B \\
            stage1-500k & stage1-step500000-tokens2098B \\
            stage1-750k & stage1-step750000-tokens3146B \\
            stage1-928k & stage1-step928646-tokens3896B \\
            stage2-1k & stage2-ingredient1-step1000-tokens5B \\
        \bottomrule
        \end{tabular}
        }
\end{subtable}
\hfill
\begin{subtable}[t]{0.49\textwidth}
    \centering
    \small
    \caption{Hyperparameters}
    \label{tab:hyperparameters}
        \begin{tabular}{lr}
        \toprule
            Parameter & Value \\
            \midrule
            Decoding & Greedy \\
            Precision & BF16 \\
            Seed & 42 \\
        \bottomrule
        \end{tabular}
\end{subtable}
\end{table}

\begin{figure}[!t]
    \centering
    \begin{subfigure}[b]{\linewidth}
        \includegraphics[width=\linewidth]{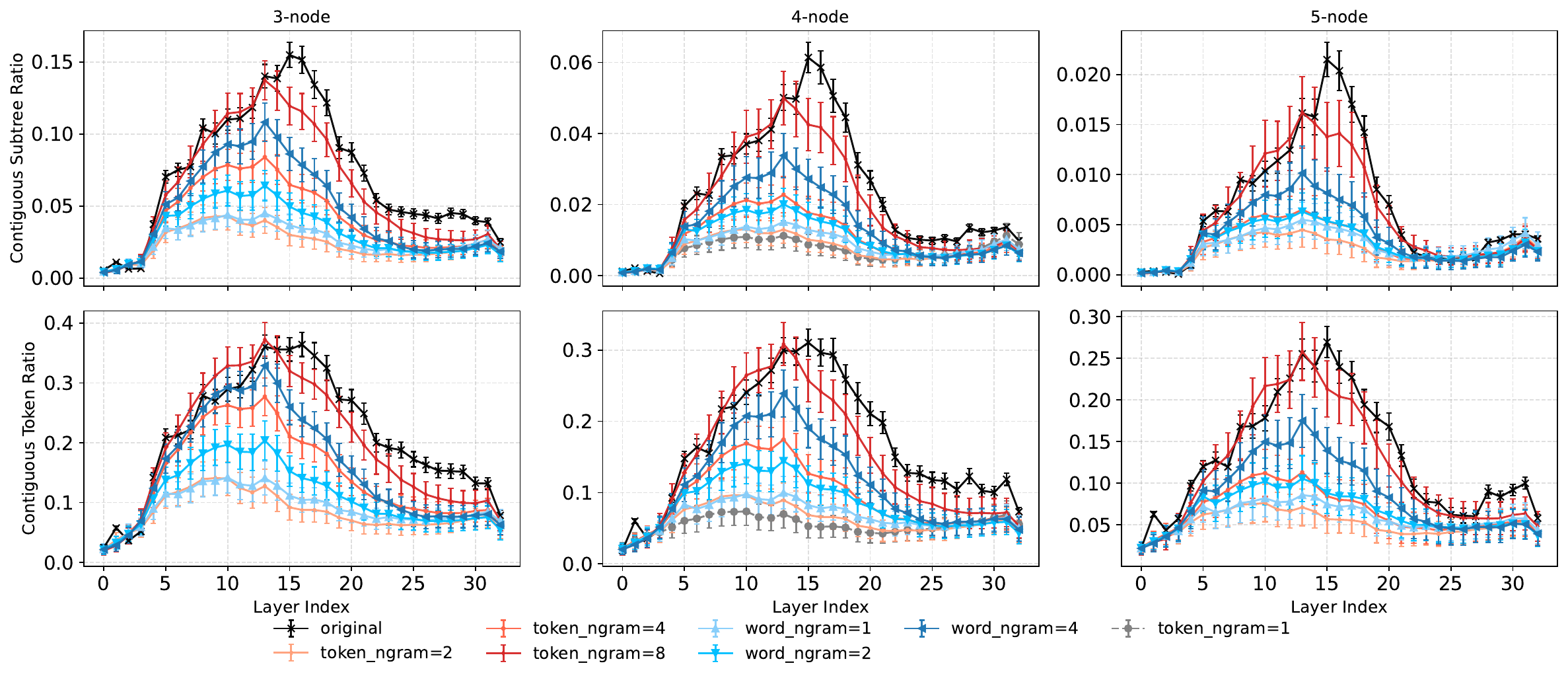}
        \caption{MMLU}
    \end{subfigure}
    \begin{subfigure}[b]{\linewidth}
        \includegraphics[width=\linewidth]{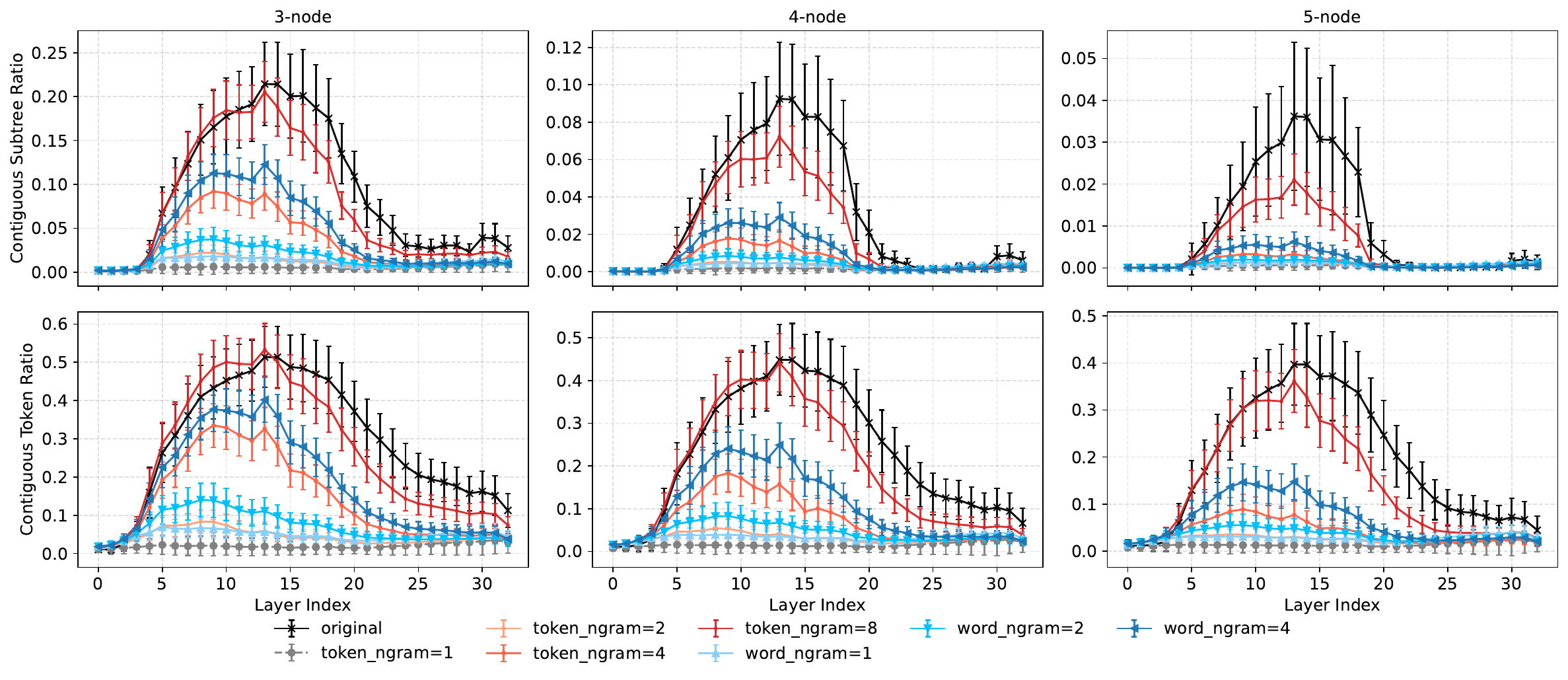}
        \caption{Multinews}
    \end{subfigure}
    \caption{Contiguous subtree and token ratios for 3-, 4-, and 5-node subtrees of Llama3.1 8B.}
    \label{fig:contiguous-subtree-3-5-node-llama}
\end{figure}

\begin{figure}[!t]
    \centering
    \begin{subfigure}[b]{\linewidth}
        \includegraphics[width=\linewidth]{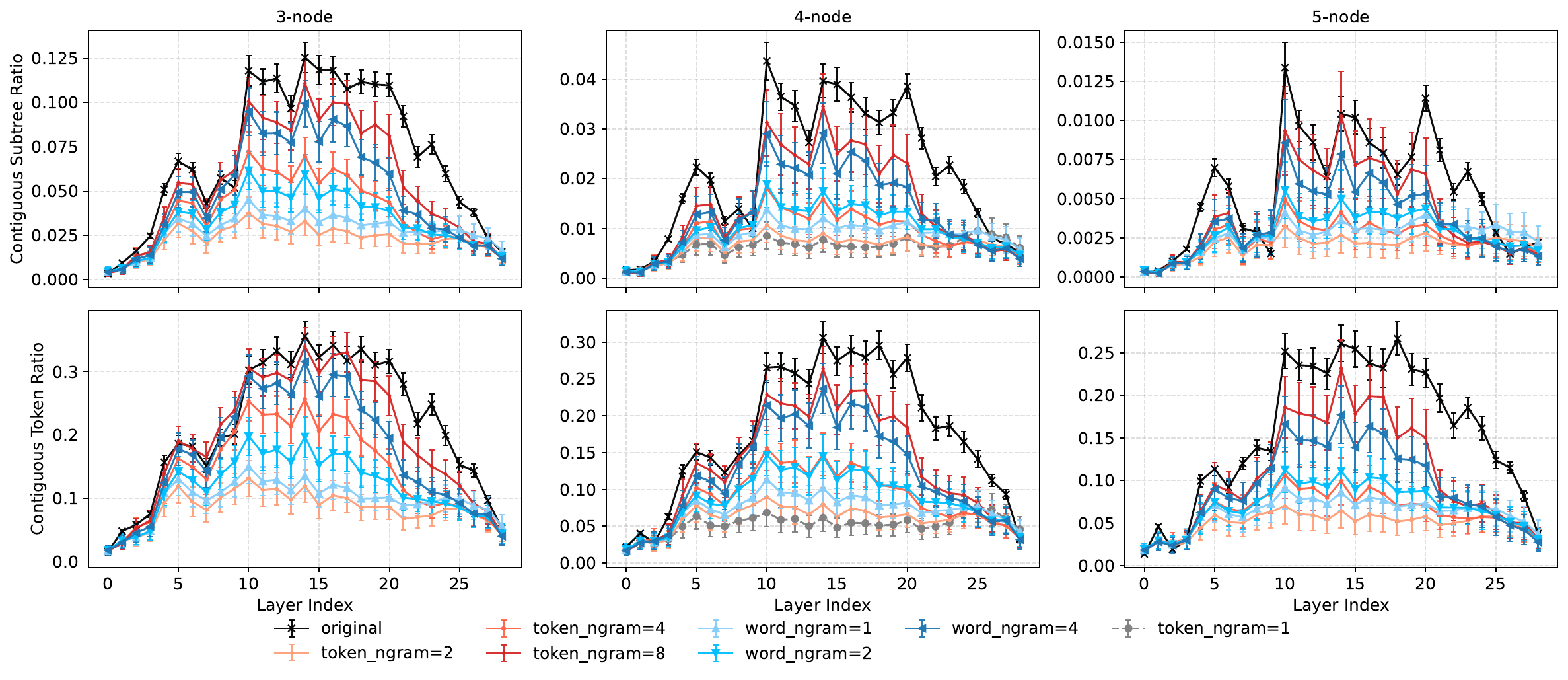}
        \caption{MMLU}
    \end{subfigure}
    \begin{subfigure}[b]{\linewidth}
        \includegraphics[width=\linewidth]{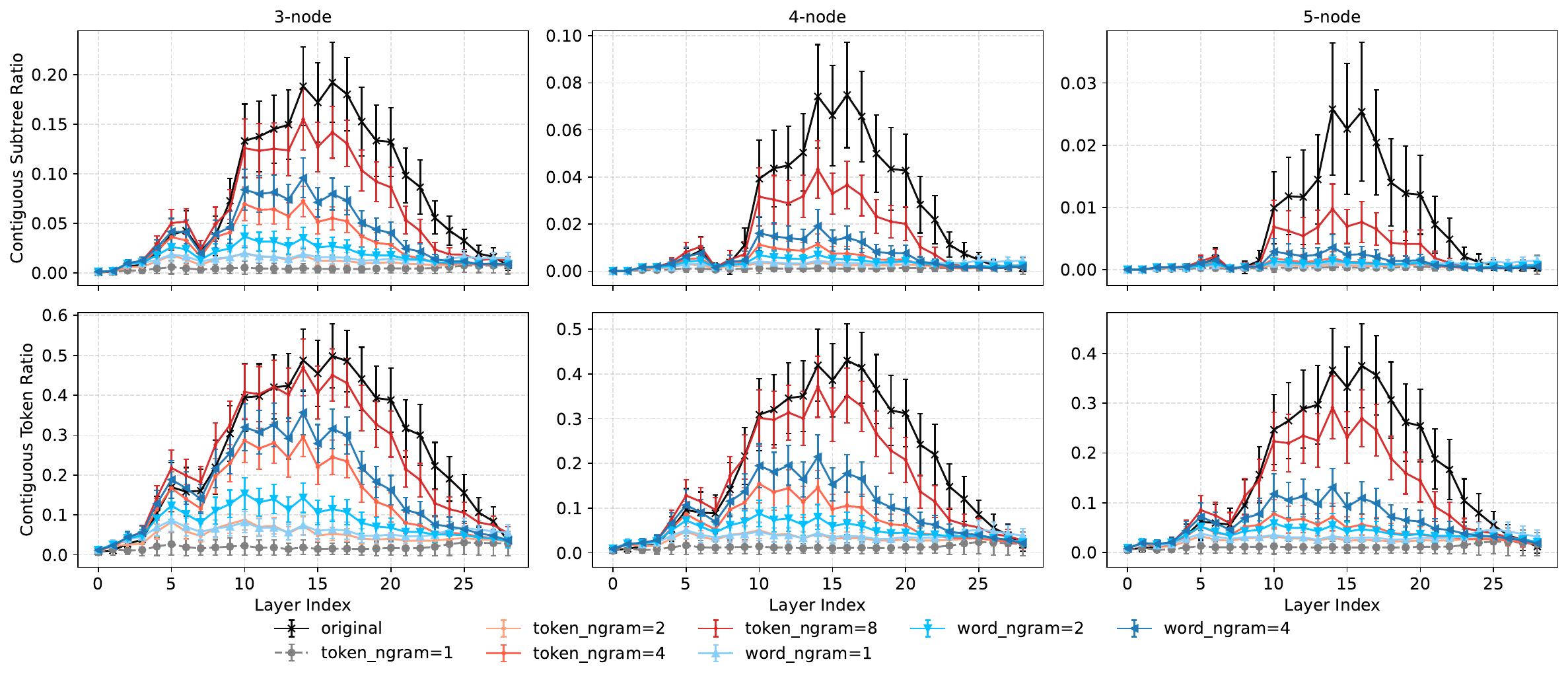}
        \caption{Multinews}
    \end{subfigure}
    \caption{Contiguous subtree and token ratios for 3-, 4-, and 5-node subtrees of Qwen2.5 7B.}
    \label{fig:contiguous-subtree-3-5-node-qwen}
\end{figure}

\section{Contiguous Subtree Analysis}

\subsection{Contiguous Subtrees of 3-, 4-, and 5-node}
\label{appendix:contiguous-subtree-3-5-node}

To examine whether the middle-layer signature in Section~\ref{sec:chunk_analysis} is a specific feature of 4-node subtrees, we measure the contiguous subtree ratio and the contiguous token ratio for 3- and 5-node subtrees.
Figures~\ref{fig:contiguous-subtree-3-5-node-llama}~and~\ref{fig:contiguous-subtree-3-5-node-qwen} show the layer-wise evolution of both the contiguous subtree ratio and the contiguous token ratio for 3-, 4-, and 5-node subtrees.
The models exhibit the middle-layer signature across subtree sizes, suggesting that they show high contiguous token connectivity within local spans in the middle layers' representation space.
\vspace{50em}

\subsection{Contiguous Subtrees during Pre-training}
\label{appendix:contiguous-subtree-olmo-shuffled}

\begin{wrapfigure}{r}{0.45\textwidth}
    \centering
    \vspace{-3.5em}
    \includegraphics[width=\linewidth]{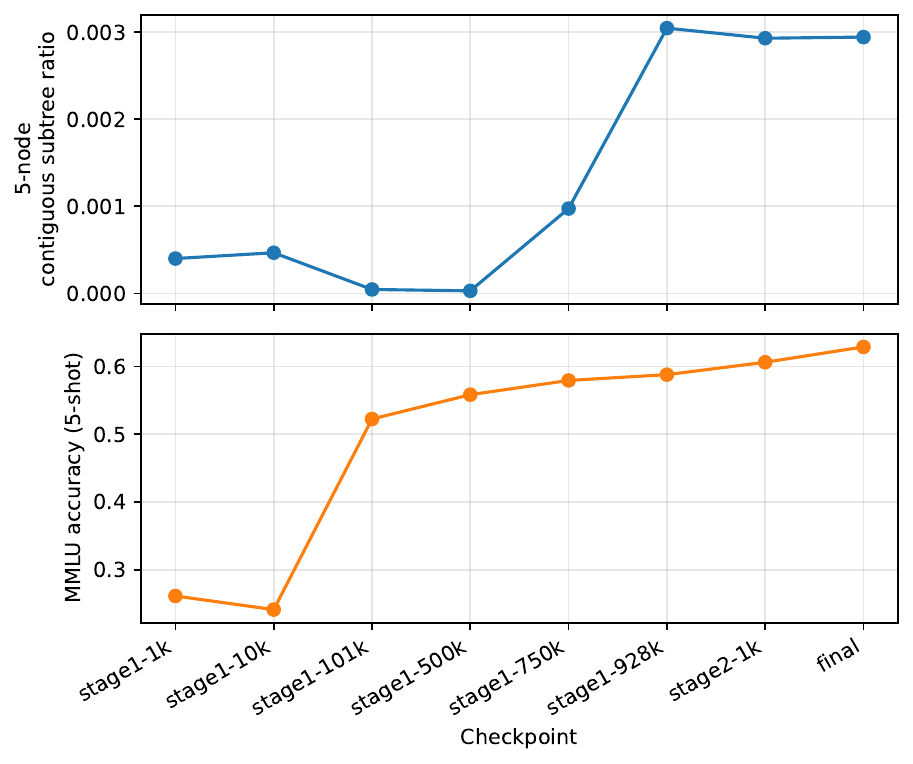}
    \vspace{-0.5em}
    \caption{5-node contiguous subtree ratio at Layer 18 and MMLU performance across checkpoints.}
    \label{fig:5-node-csr-mmlu}
    \vspace{-1.0em}
\end{wrapfigure}
We observe the pre-training step-wise evolution of the local-span connectivity in Section~\ref{sec:training-dynamics}.
Figure~\ref{fig:5-node-csr-mmlu} shows the 5-node contiguous subtree ratio at Layer 18 and 5-shot MMLU performance across checkpoints, indicating that the contiguous subtree ratio increases sharply between stage1-500k and stage1-928k, while MMLU performance improves less. 
For further investigation, we shuffle input tokens and compute the contiguous subtree ratio and the contiguous token ratio as in Section~\ref{sec:chunk_analysis}.
Figure~\ref{fig:contiguous_subtree_olmo_shuffled} shows the shuffling results along with those of the original order shown in Figure~\ref{fig:contiguous_subtree_olmo} in Section~\ref{sec:training-dynamics}.
This figure suggests that both ratios in the middle and higher layers increase across checkpoints, indicating that the observed contiguous structures depend on the original token order rather than arising solely from implicit positional biases.

\begin{figure}[ht]
    \centering
    \includegraphics[width=0.9\linewidth]{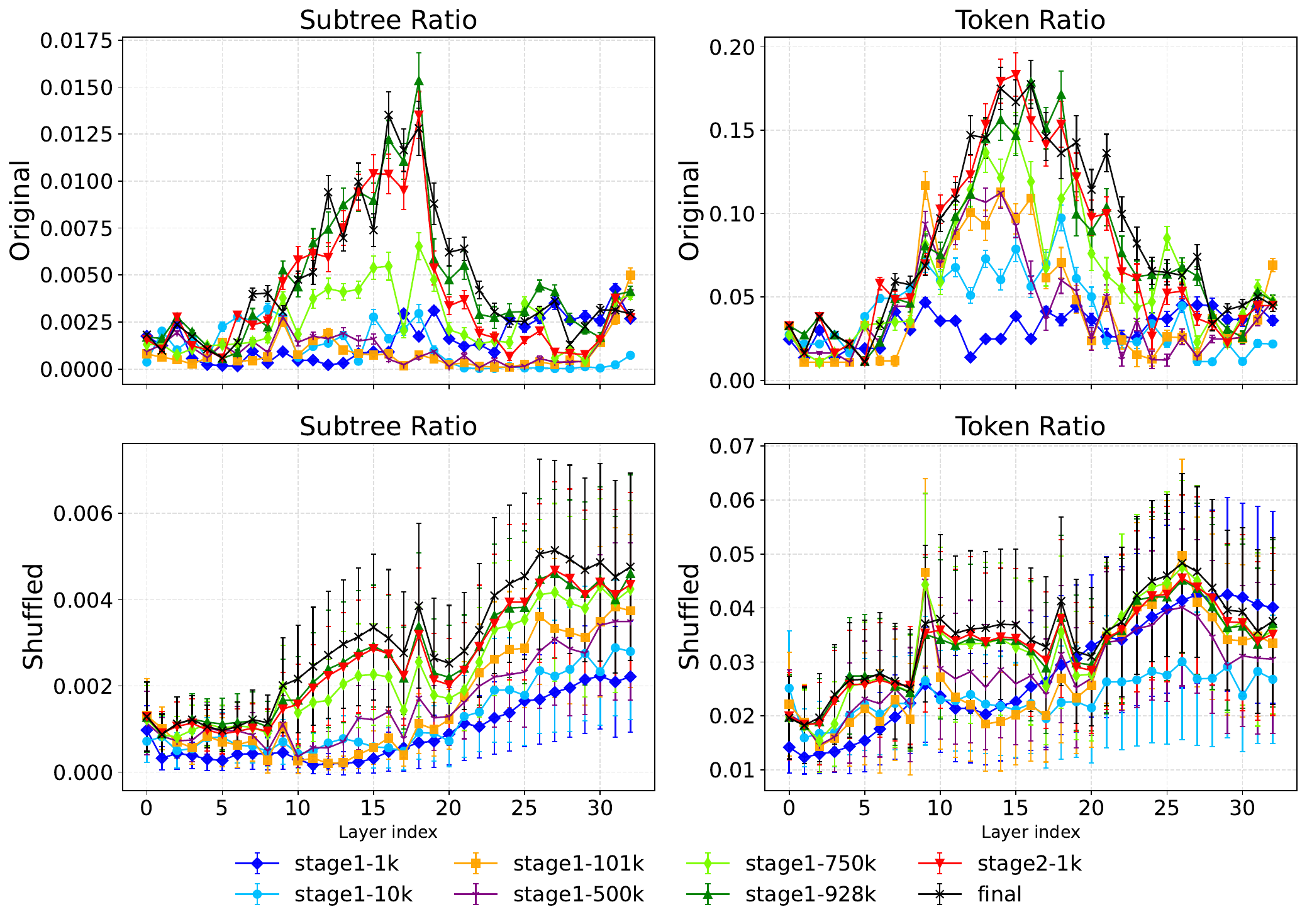}
    \caption{Visualization of the layer-wise evolution of the contiguous subtrees and tokens in them of multiple checkpoints of OLMo2 7B for MMLU. The x-axis denotes the layer index, and the y-axis denotes the contiguous subtree and token ratio.}
    \label{fig:contiguous_subtree_olmo_shuffled}
\end{figure}

\section{Inter-Layer Similarity in Diverse Settings}
\label{appendix:layer-sim}

\begin{figure*}[t]
    \centering
    \begin{subfigure}[b]{0.17\textwidth}
        \includegraphics[width=\textwidth,clip,trim={0 0 24em 0}]{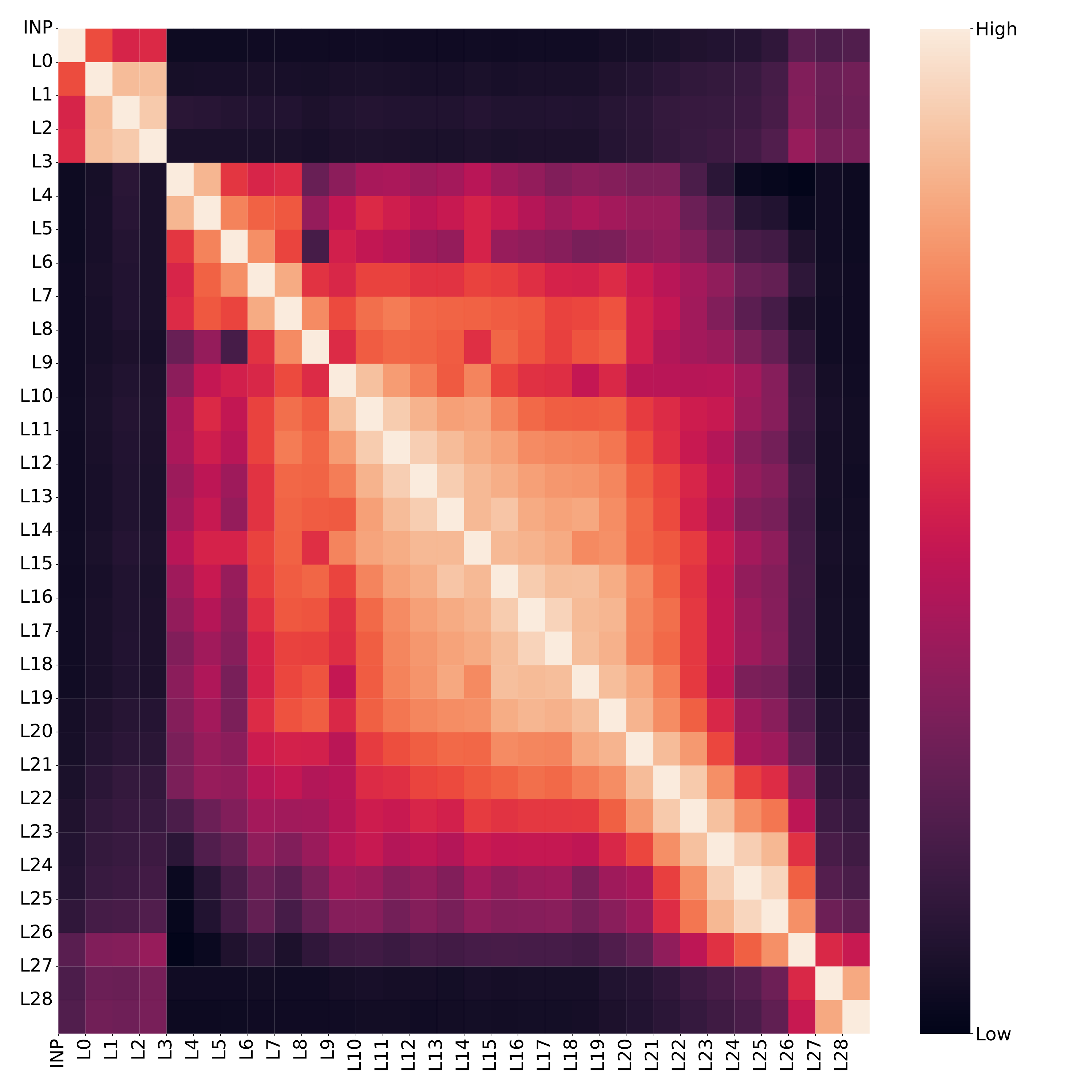}
        \caption{CKA}
    \end{subfigure}
    \begin{subfigure}[b]{0.17\textwidth}
        \includegraphics[width=\textwidth,clip,trim={0 0 24em 0}]{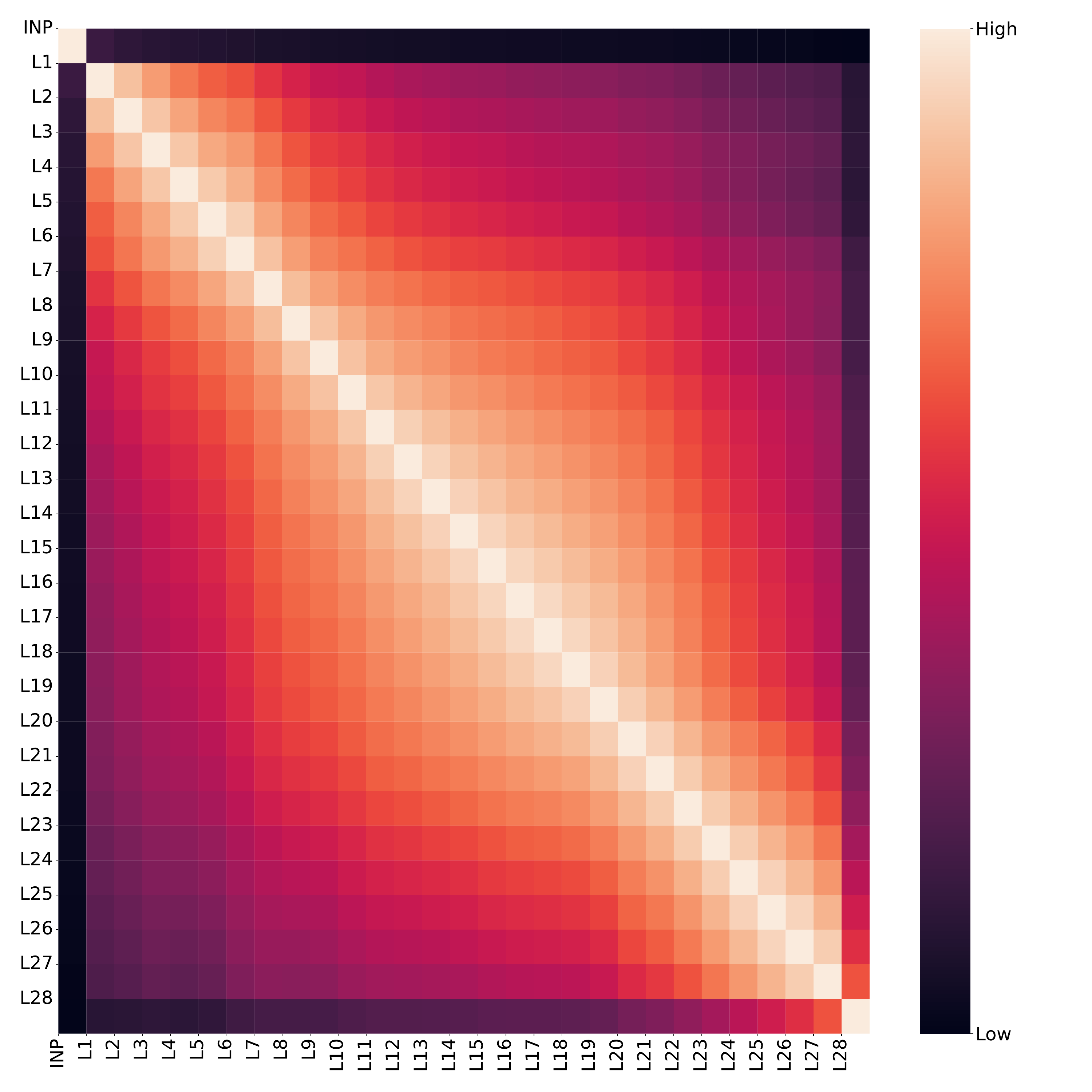}
        \caption{Cos-Base}
    \end{subfigure}
    \begin{subfigure}[b]{0.17\textwidth}
        \includegraphics[width=\textwidth,clip,trim={0 0 24em 0}]{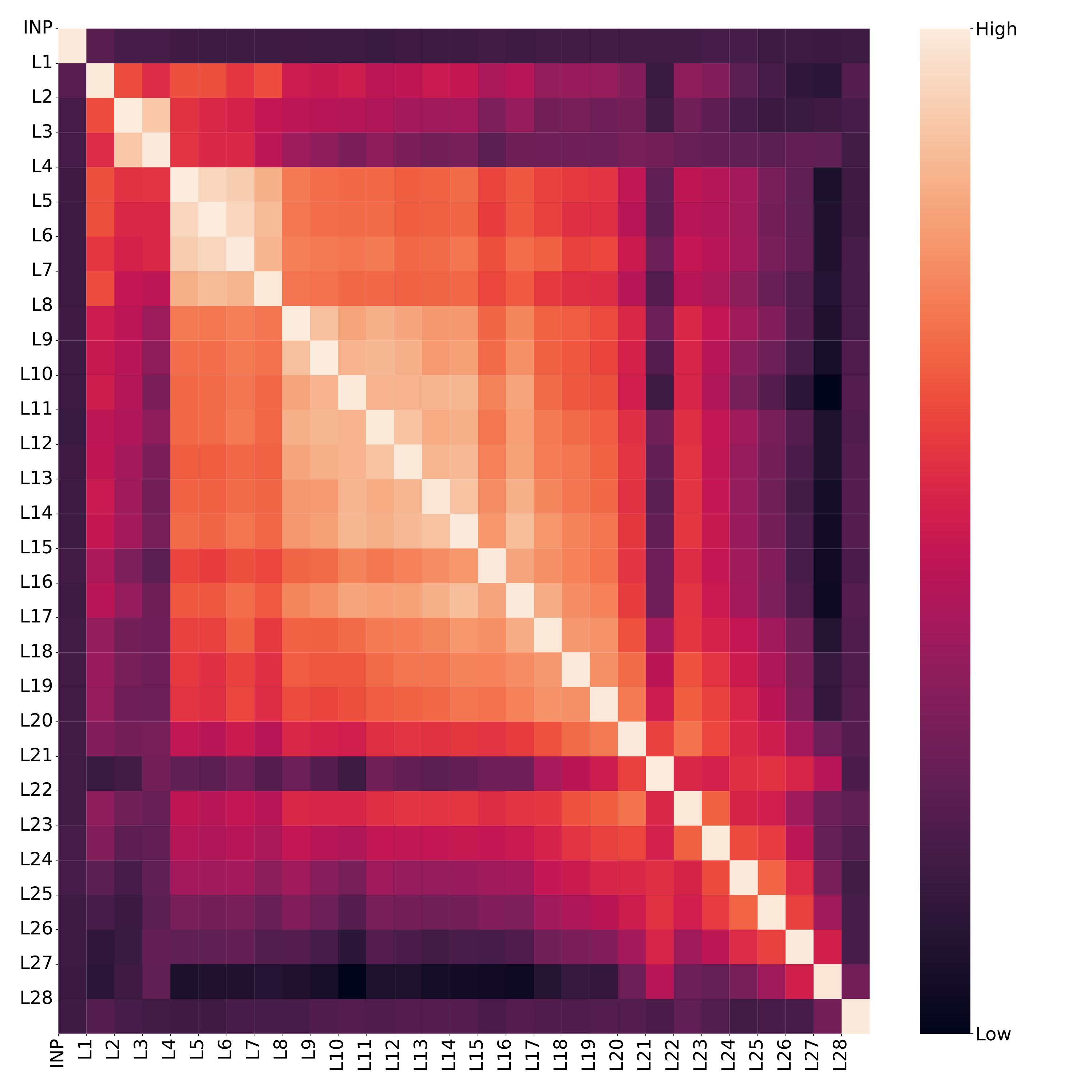}
        \caption{Cos-Struct}
    \end{subfigure}
    \begin{subfigure}[b]{0.17\textwidth}
        \includegraphics[width=\textwidth,clip,trim={0 0 24em 0}]{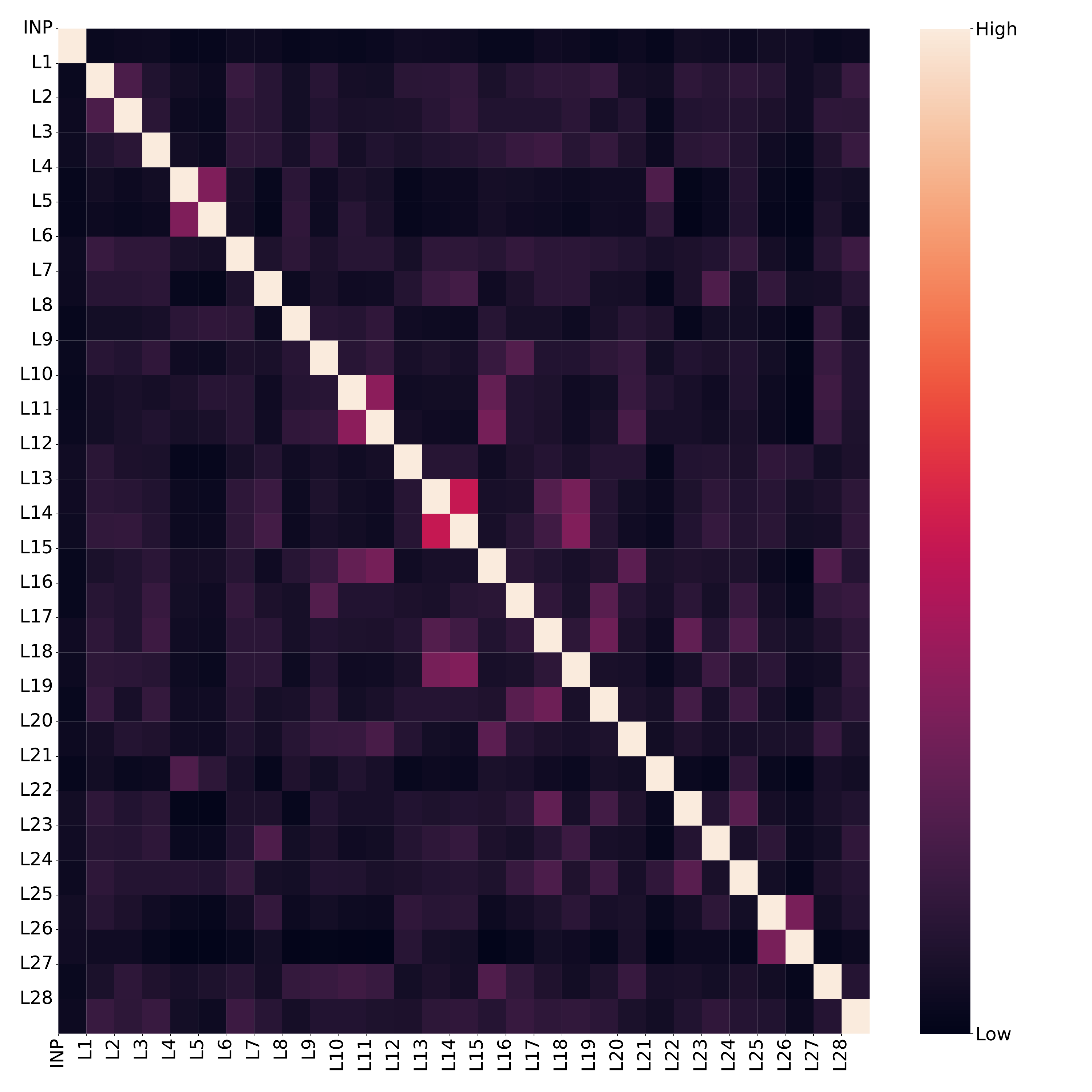}
        \caption{Tree-Edit}
    \end{subfigure}
    \begin{subfigure}[b]{0.17\textwidth}
        \includegraphics[width=\textwidth,clip,trim={0 0 24em 0}]{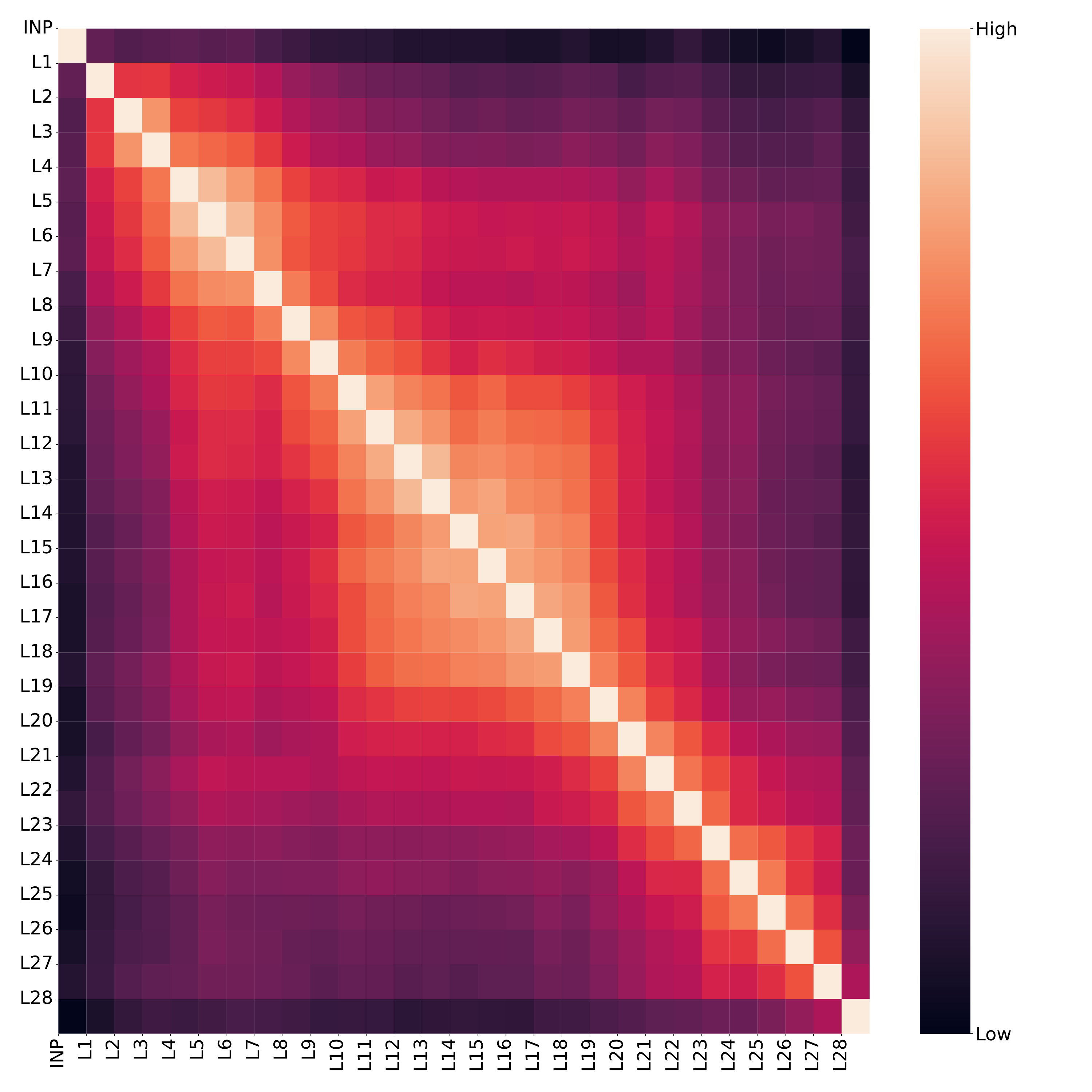}
        \caption{Edge-Edit}
    \end{subfigure}
    \caption{Inter-layer similarity samples of Qwen2.5 7B for each metric on MMLU. Bright color represents high similarity, while dark color represents low similarity.}
    \label{fig:sim_mmlu_qwen}
\end{figure*}

\paragraph{Inter-layer similarity in Qwen2.5 7B on MMLU.}
We measure inter-layer similarity of CKA, Cos-Base, Cos-Struct, Tree-Edit, and Edge-Edit for Qwen2.5 7B on MMLU and investigate if it differs from that for Llama3.1 8B reported in Section~\ref{sec:layer-sim-experiment}.
Figure~\ref{fig:sim_mmlu_qwen} shows that Edge-Edit exhibits diagonal clustering patterns that we refer to as islands, as well as Llama3.1 8B, while CKA and Cos-Struct show different patterns.
This suggests that Edge-Edit provides potentially insightful patterns of inter-layer similarity across models.

\paragraph{Edge-Edit across datasets, model families, and model sizes.}
To examine whether inter-layer similarity patterns by Edge-Edit shown in Section~\ref{sec:layer-sim-experiment} vary across datasets, model families, and model sizes, we measure the similarity for Llama3.1 70B, Qwen2.5 7B, and Qwen2.5 72B on both MMLU and CMMLU.
Figures~\ref{fig:layer-sim-llama}~and~\ref{fig:layer-sim-qwen} show that the similarity patterns are close within a model family, but slightly differ across model families.
This result suggests that the island patterns derived by Edge-Edit are an inherent property of a model family.

\begin{figure*}[t]
    \centering
    \begin{subfigure}[b]{0.17\textwidth}
        \includegraphics[width=\textwidth,clip,trim={0 0 24em 0}]{assets/mmlu/llama/ted_0_l2_edge_edit.pdf}
        \caption{8B, MMLU}
    \end{subfigure}
    \begin{subfigure}[b]{0.17\textwidth}
        \includegraphics[width=\textwidth,clip,trim={0 0 24em 0}]{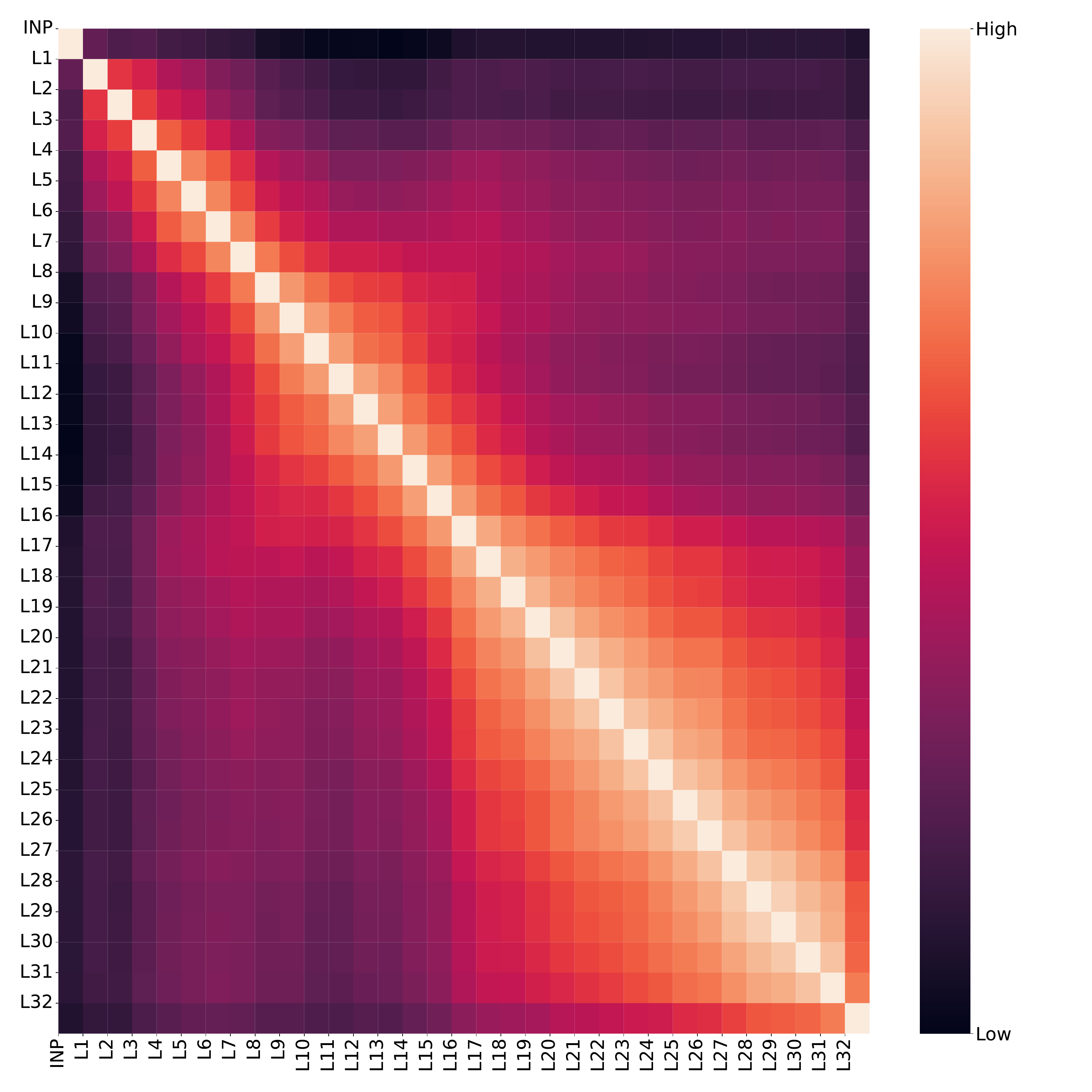}
        \caption{8B, CMMLU}
    \end{subfigure}
    \begin{subfigure}[b]{0.17\textwidth}
        \includegraphics[width=\textwidth,clip,trim={0 0 24em 0}]{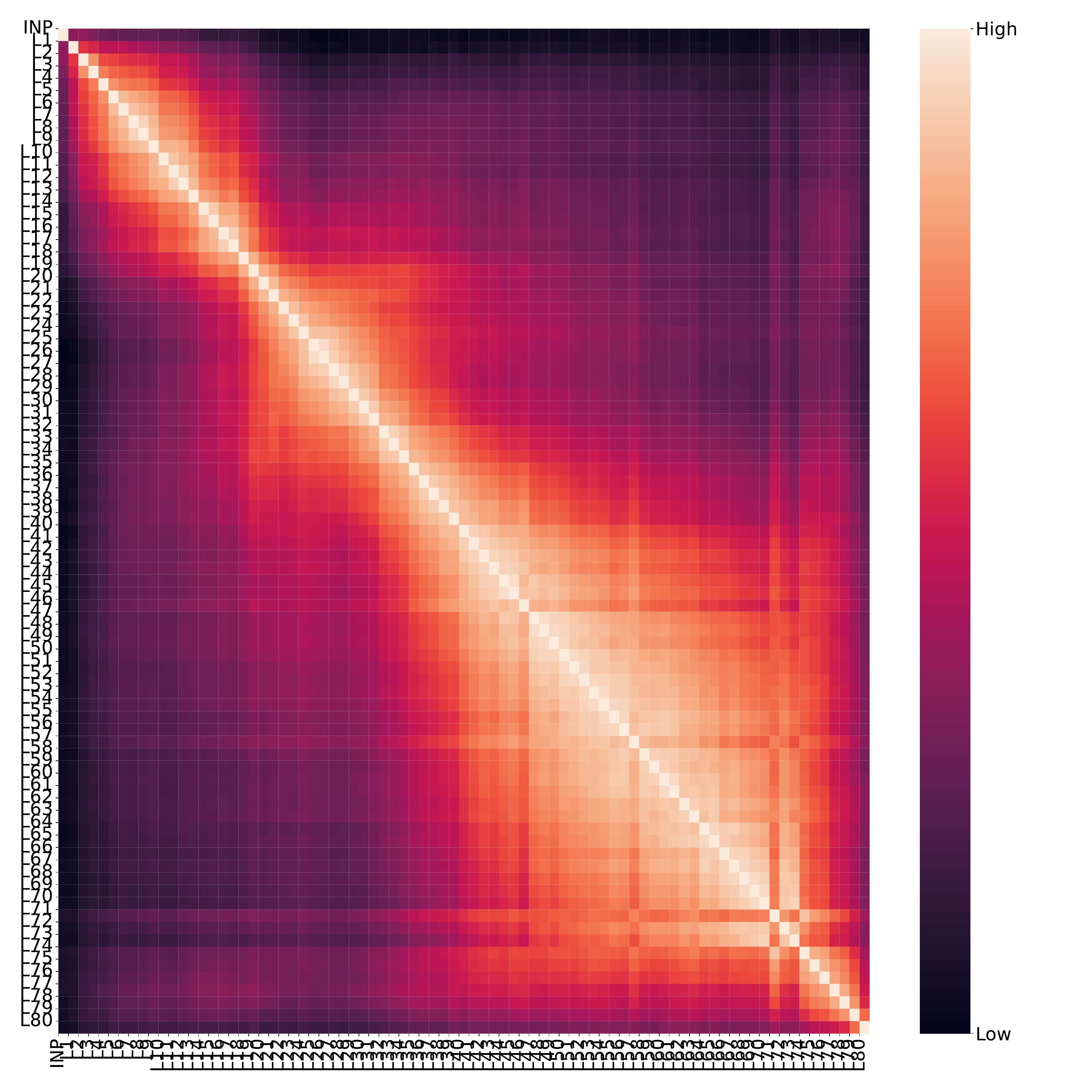}
        \caption{70B, MMLU}
    \end{subfigure}
    \begin{subfigure}[b]{0.17\textwidth}
        \includegraphics[width=\textwidth,clip,trim={0 0 24em 0}]{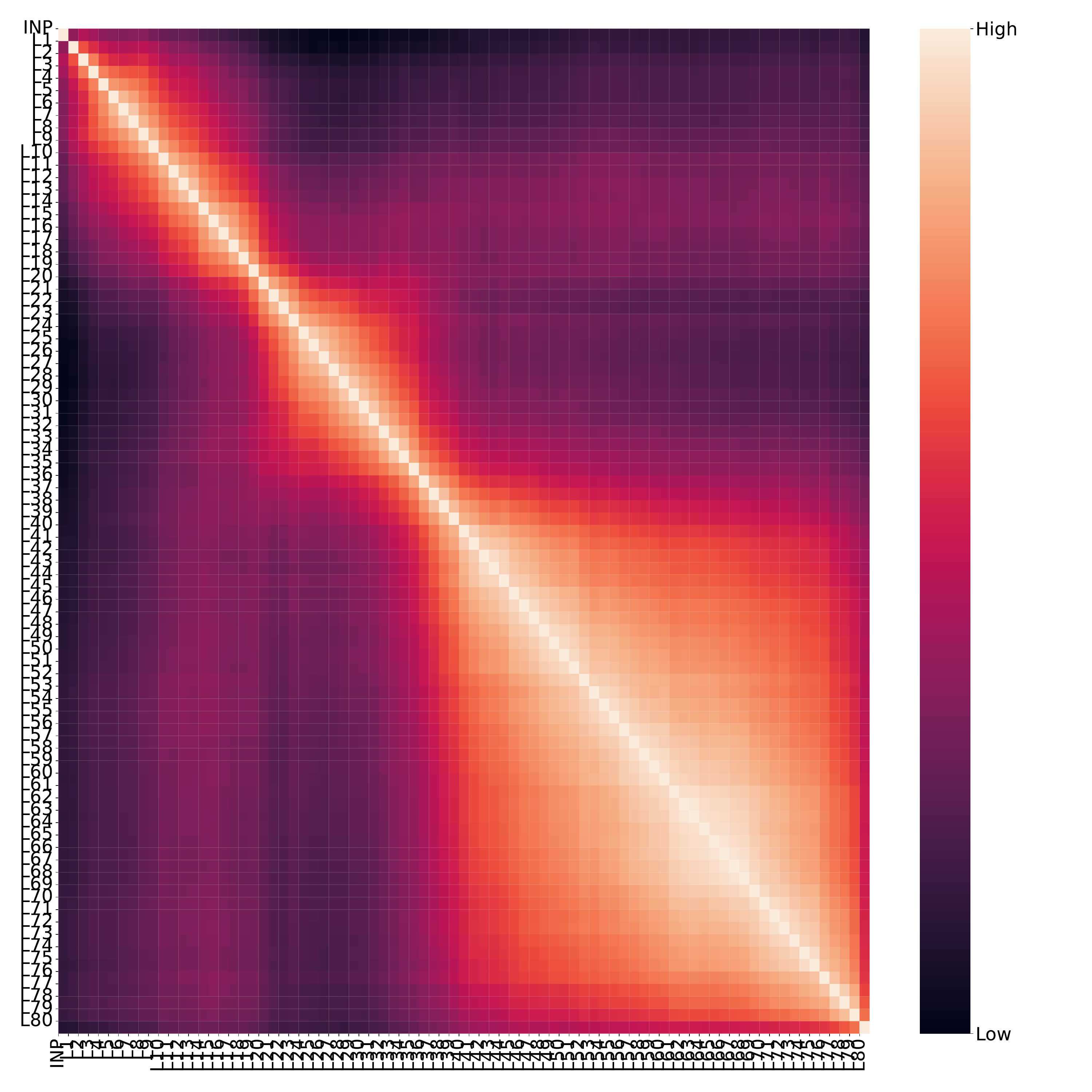}
        \caption{70B, CMMLU}
    \end{subfigure}
        \caption{Inter-layer similarity samples by Edge-Edit of Llama3.1 8B and Llama3.1 70B for MMLU and CMMLU. Bright color represents high similarity, while dark color represents low similarity.}
    \label{fig:layer-sim-llama}
\end{figure*}

\begin{figure*}[t]
    \centering
    \begin{subfigure}[b]{0.17\textwidth}
        \includegraphics[width=\textwidth,clip,trim={0 0 24em 0}]{assets/mmlu/qwen/ted_0_l2_edge_edit.pdf}
        \caption{7B, MMLU}
    \end{subfigure}
    \begin{subfigure}[b]{0.17\textwidth}
        \includegraphics[width=\textwidth,clip,trim={0 0 24em 0}]{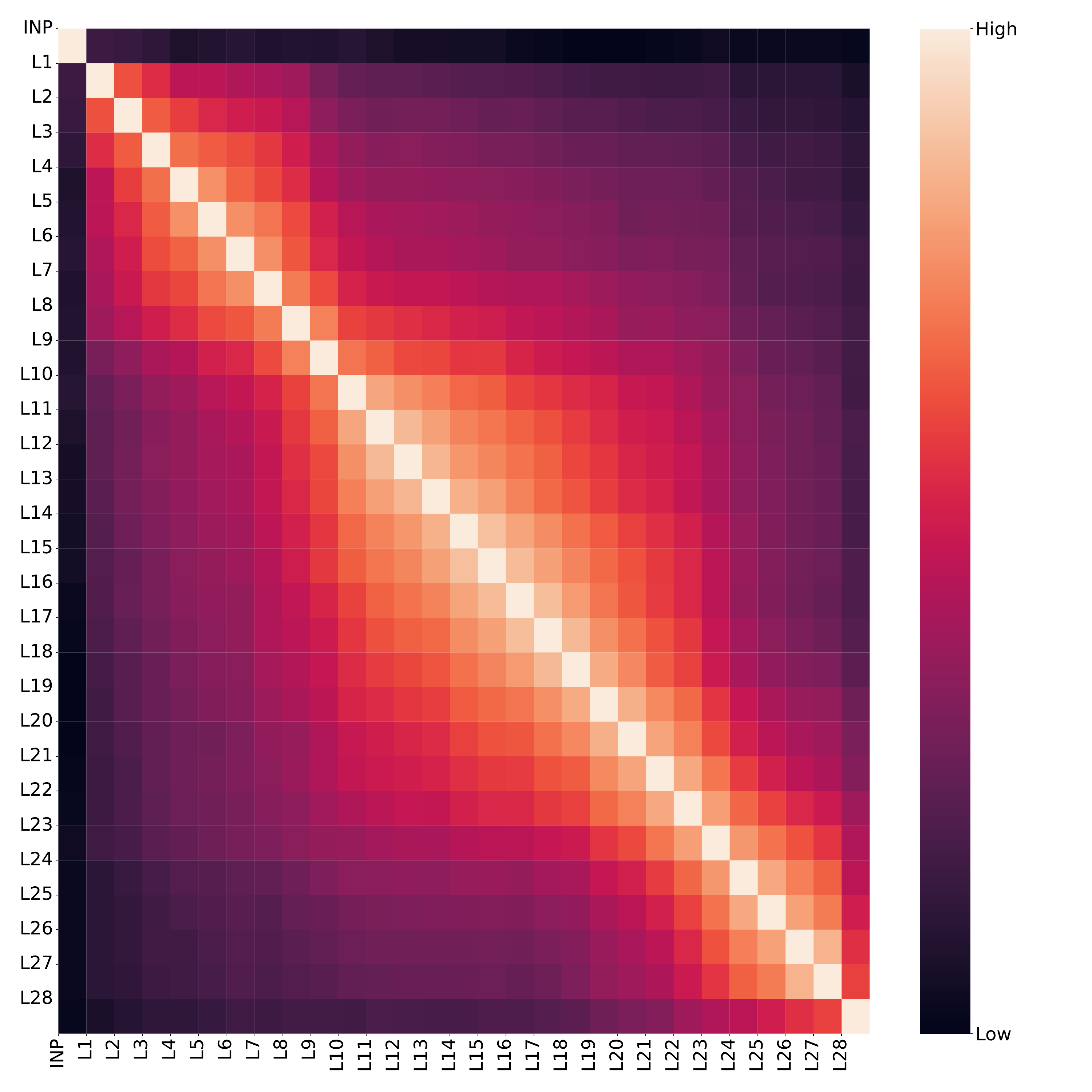}
        \caption{7B, CMMLU}
    \end{subfigure}
    \begin{subfigure}[b]{0.17\textwidth}
        \includegraphics[width=\textwidth,clip,trim={0 0 24em 0}]{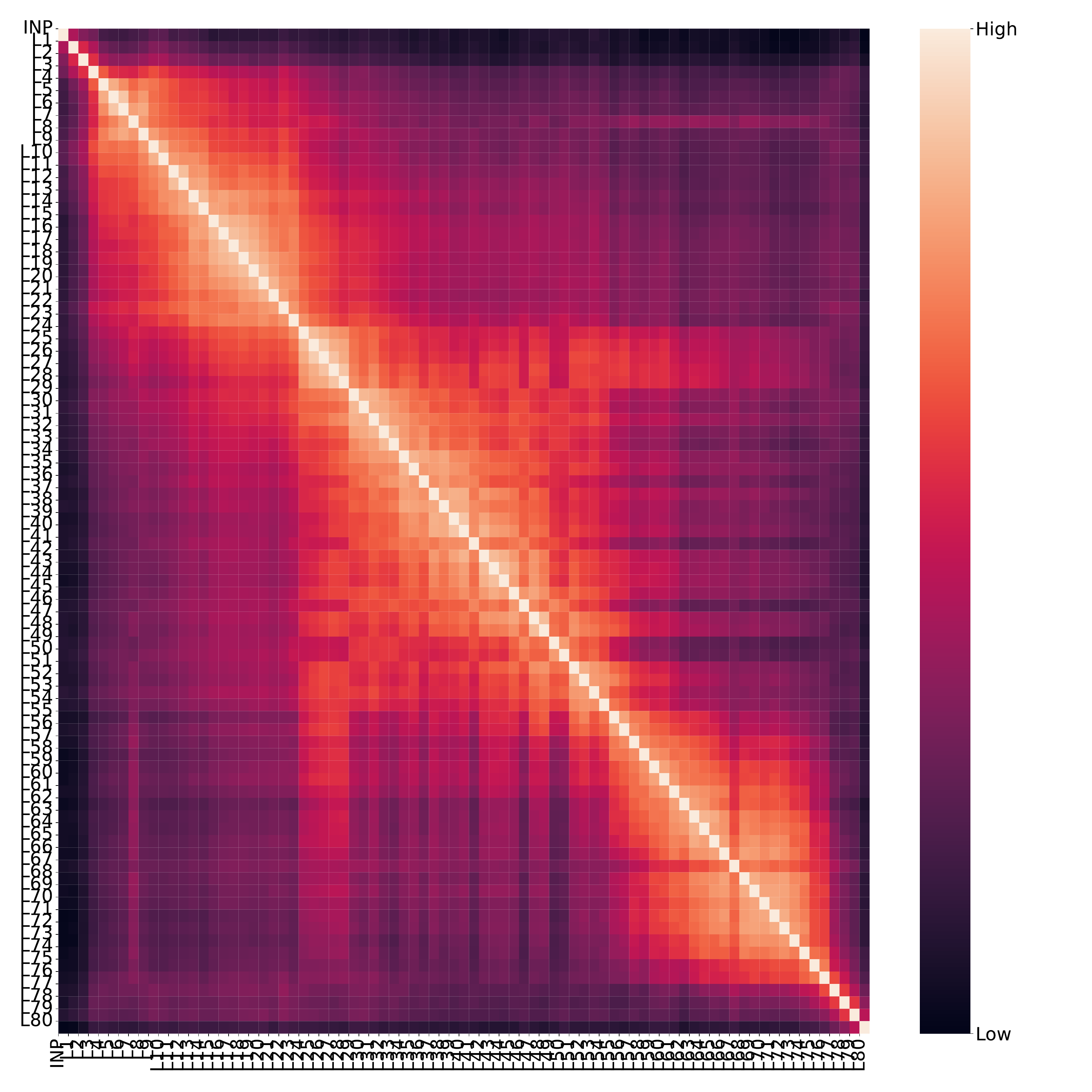}
        \caption{72B, MMLU}
    \end{subfigure}
    \begin{subfigure}[b]{0.17\textwidth}
        \includegraphics[width=\textwidth,clip,trim={0 0 24em 0}]{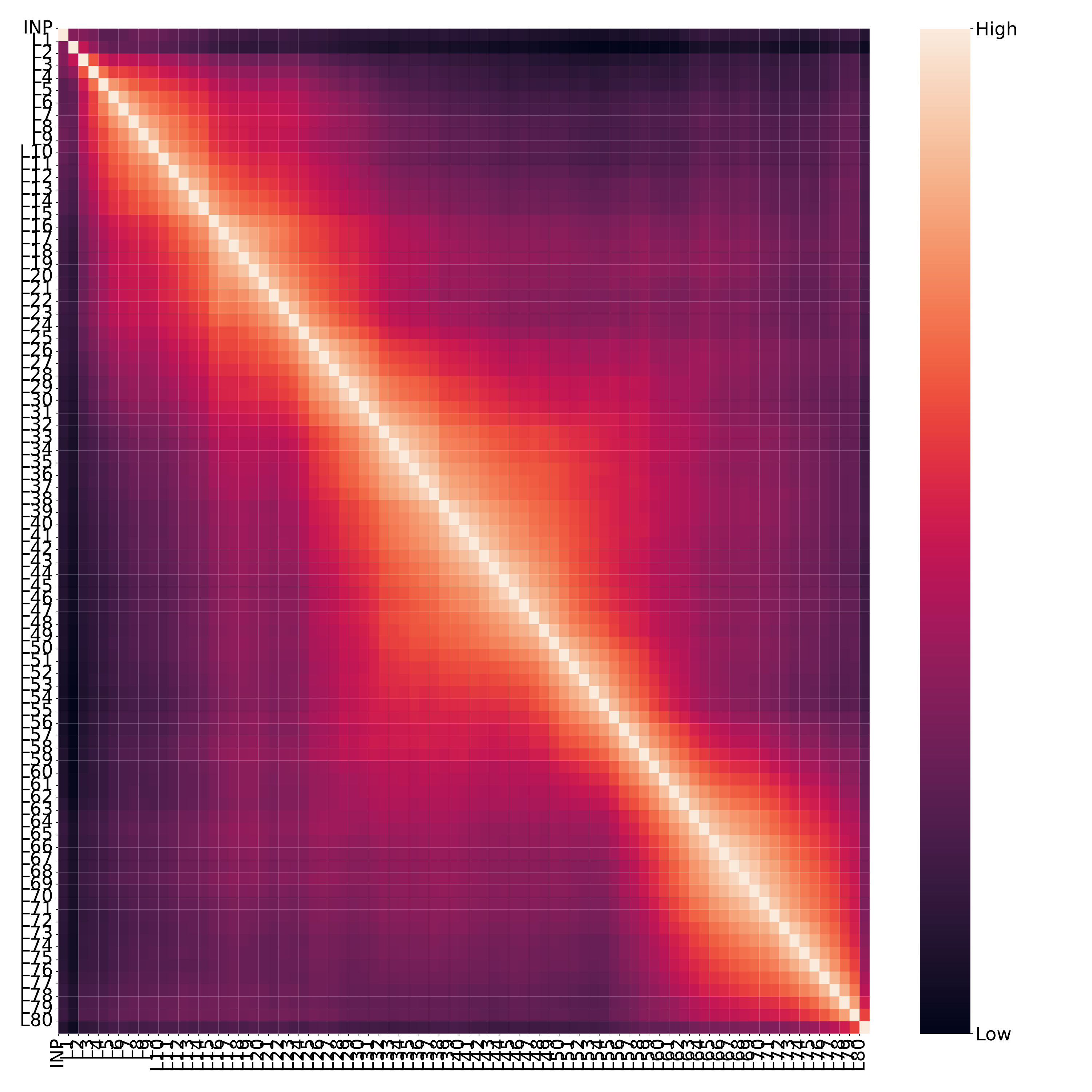}
        \caption{72B, CMMLU}
    \end{subfigure}
        \caption{Inter-layer similarity samples by Edge-Edit of Qwen2.5 7B and Qwen2.5 72B for MMLU and CMMLU. Bright color represents high similarity, while dark color represents low similarity.}
    \label{fig:layer-sim-qwen}
\end{figure*}

\paragraph{OLMo2 checkpoints.}
To investigate the evolution of the island patterns during training, we measure the Edge-Edit similarity for several checkpoints, which are illustrated in Figure~\ref{fig:sim_mmlu_olmo}.
The checkpoint stage1-1k shows relatively large islands in the higher layers, and the model reduces their size and increases similarity between adjacent layers as training progresses.
This tendency suggests that the model learns to collaborate between adjacent layers at late training steps, based on internal structures of semantic representations.

\begin{figure}[!t]
    \centering
    \begin{subfigure}[b]{0.16\linewidth}
        \includegraphics[width=\textwidth,clip,trim={0 0 24em 0}]{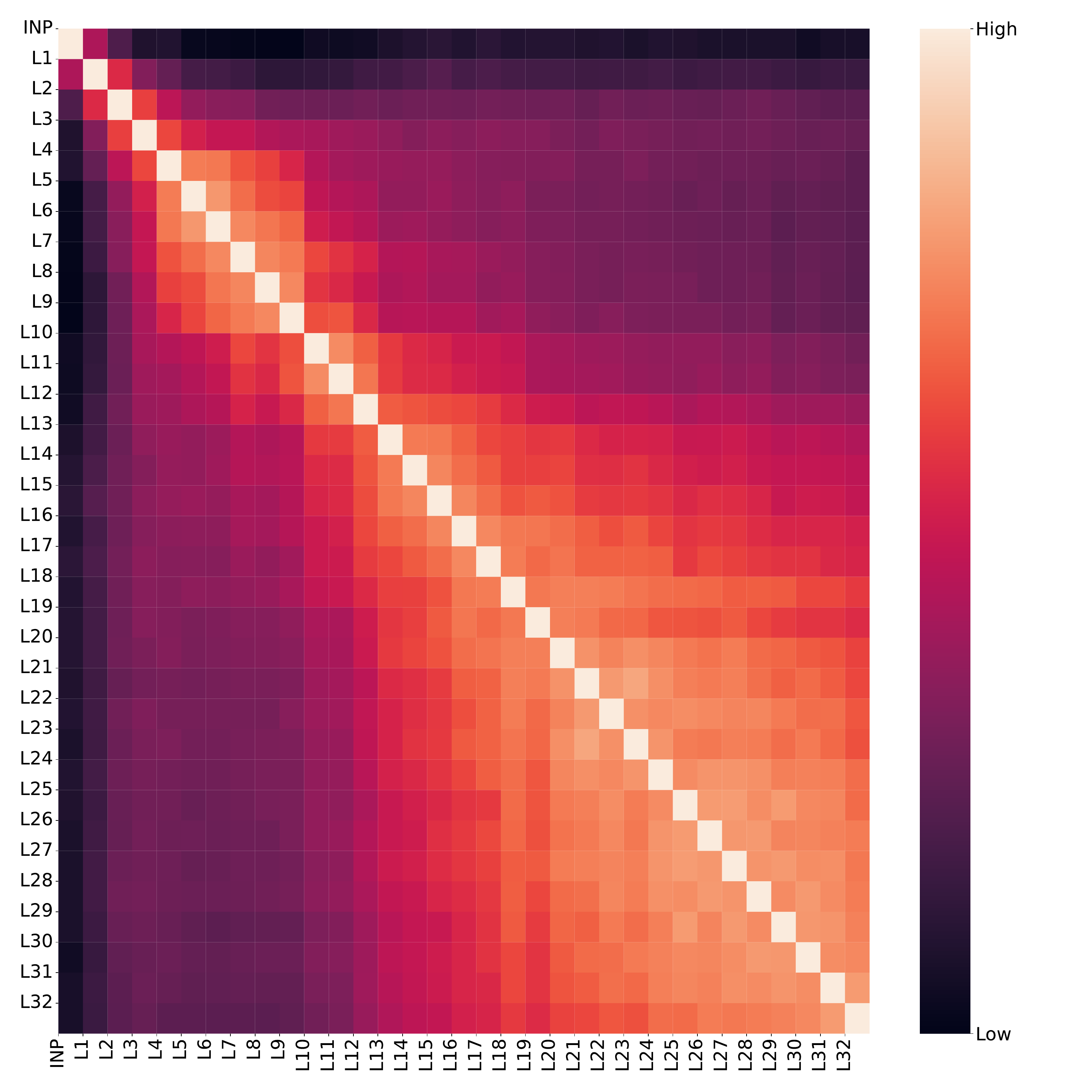}
        \caption{stage1-1k}
    \end{subfigure}
    \begin{subfigure}[b]{0.16\linewidth}
        \includegraphics[width=\textwidth,clip,trim={0 0 24em 0}]{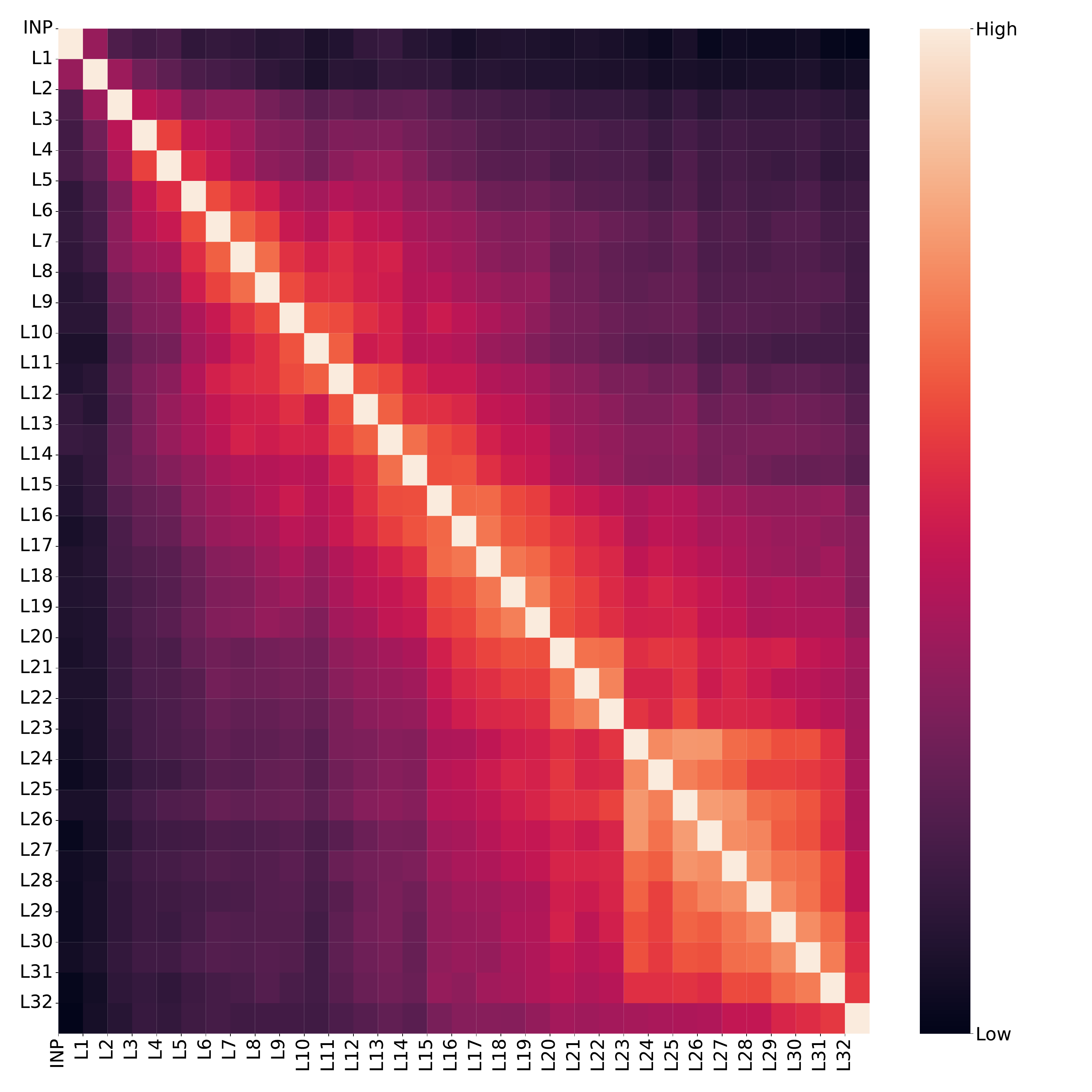}
        \caption{stage1-10k}
    \end{subfigure}
    \begin{subfigure}[b]{0.16\linewidth}
        \includegraphics[width=\textwidth,clip,trim={0 0 24em 0}]{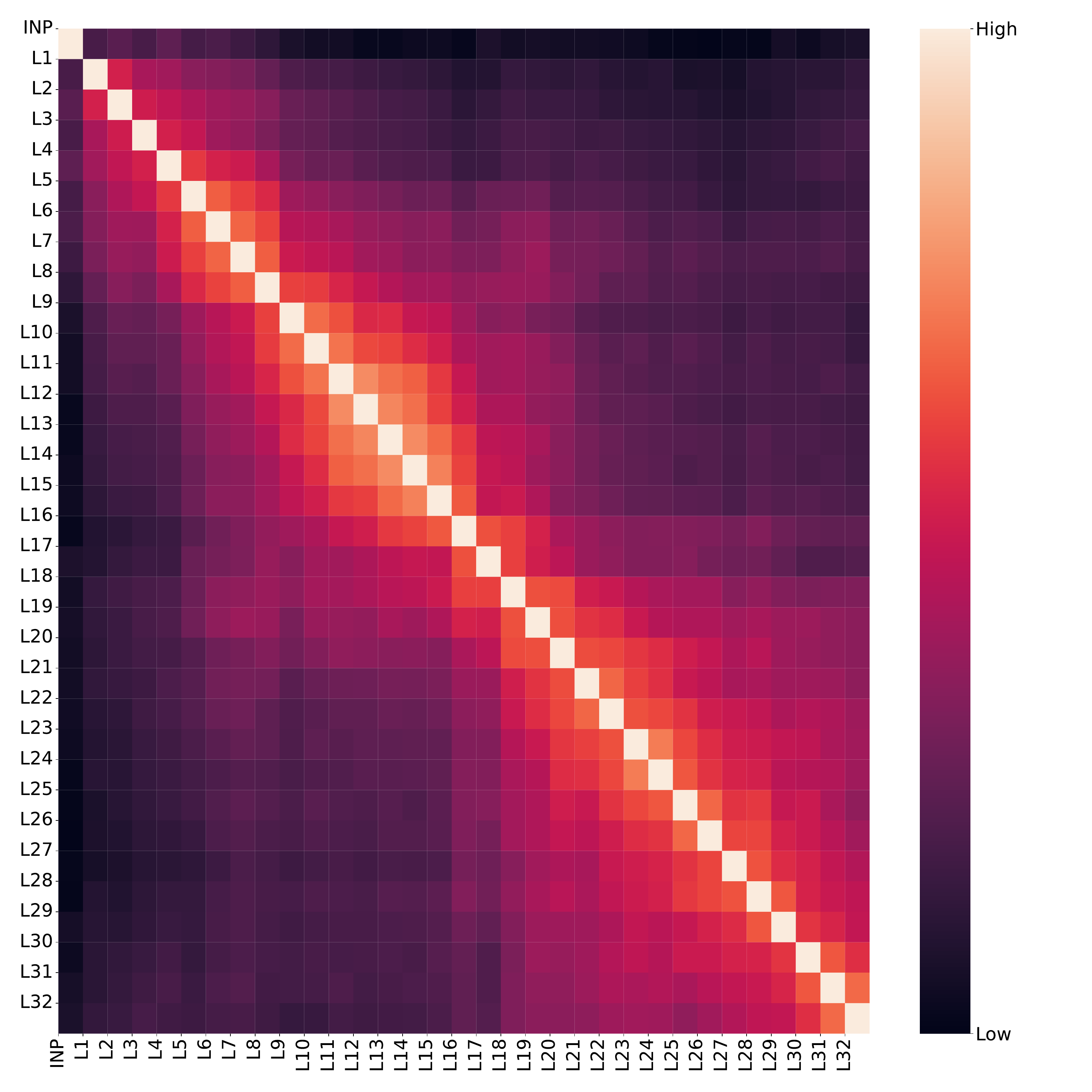}
        \caption{stage1-500k}
    \end{subfigure}
    \begin{subfigure}[b]{0.16\linewidth}
        \includegraphics[width=\textwidth,clip,trim={0 0 24em 0}]{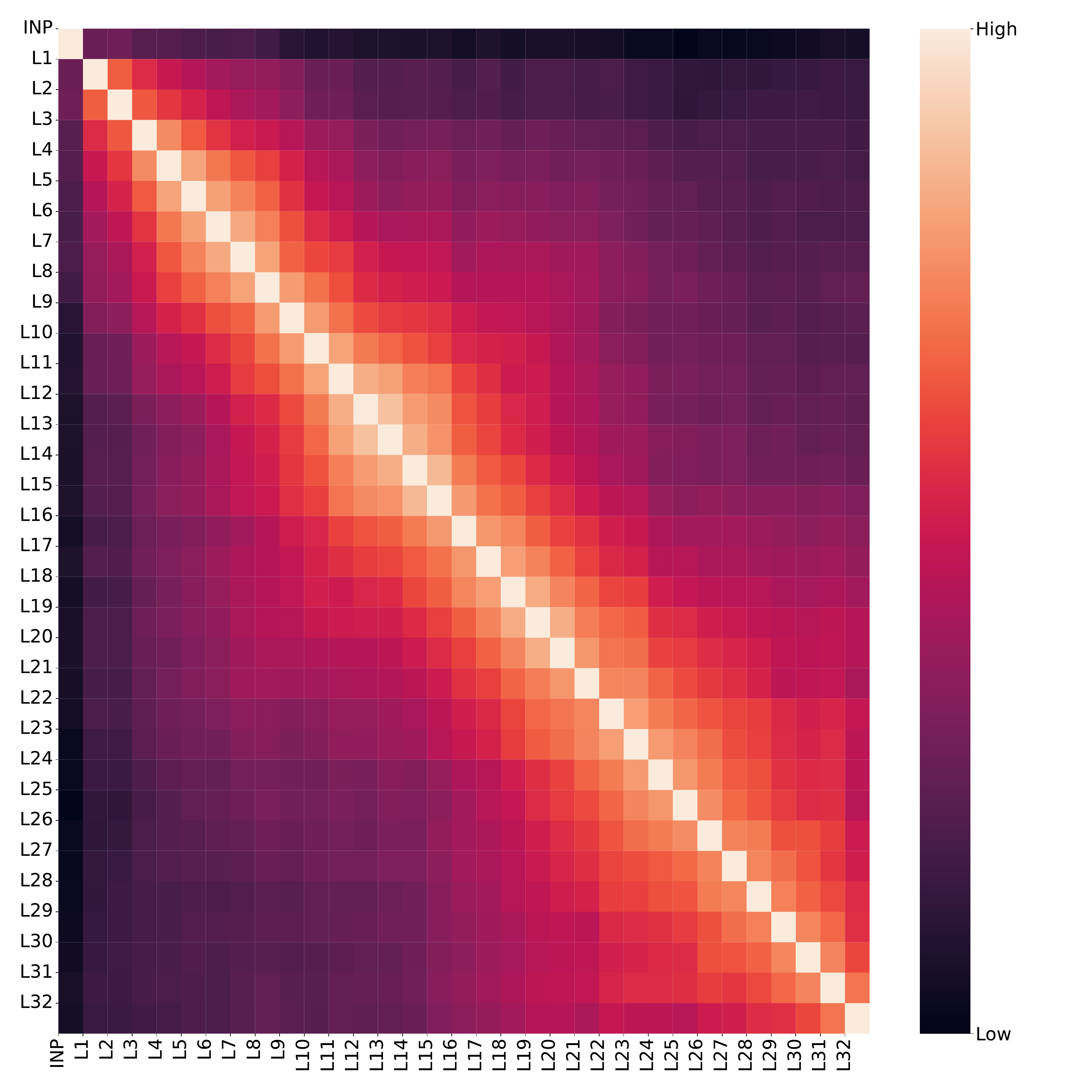}
        \caption{stage1-928k}
    \end{subfigure}
    \begin{subfigure}[b]{0.16\linewidth}
        \includegraphics[width=\textwidth,clip,trim={0 0 24em 0}]{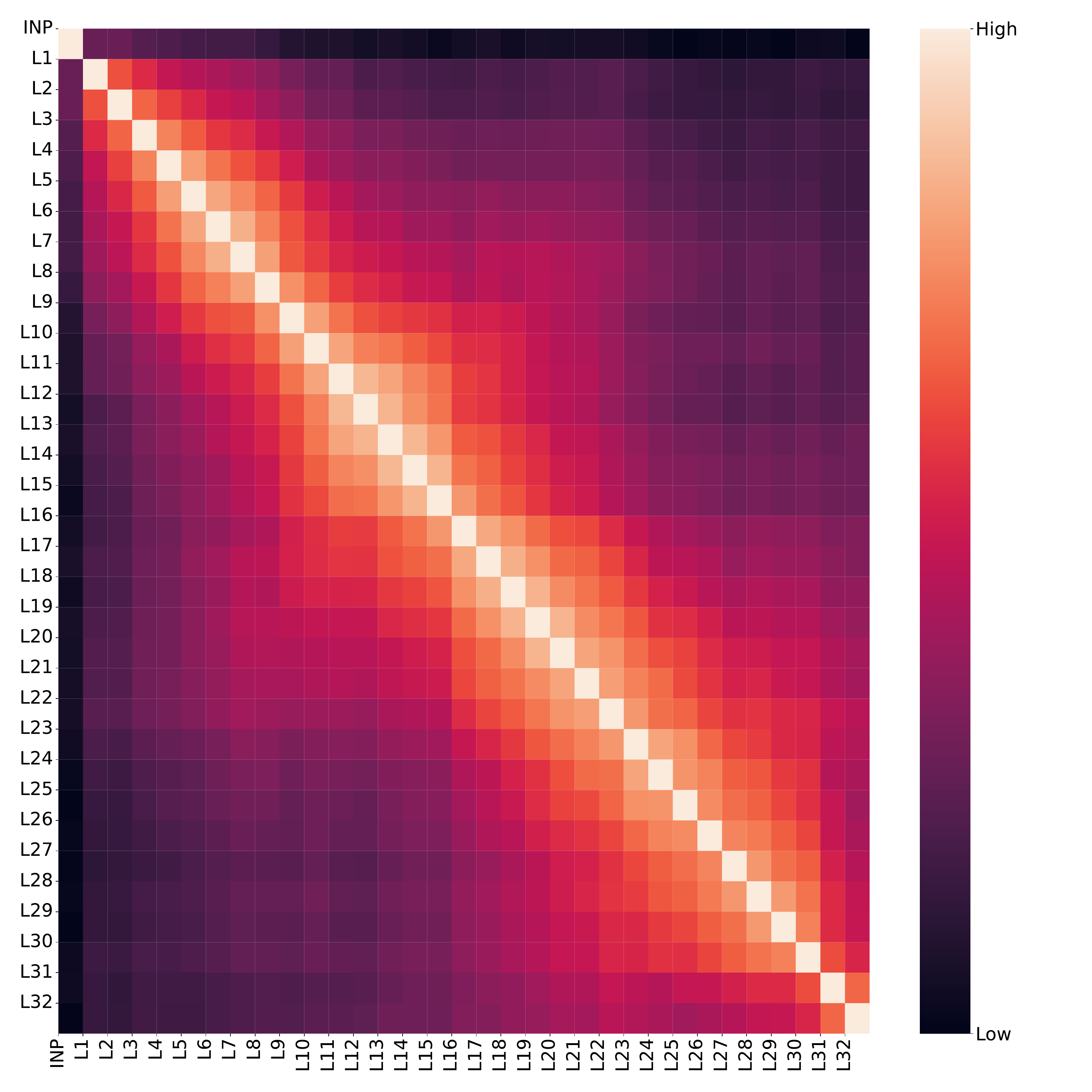}
        \caption{stage2-1k}
    \end{subfigure}
    \begin{subfigure}[b]{0.16\linewidth}
        \includegraphics[width=\textwidth,clip,trim={0 0 24em 0}]{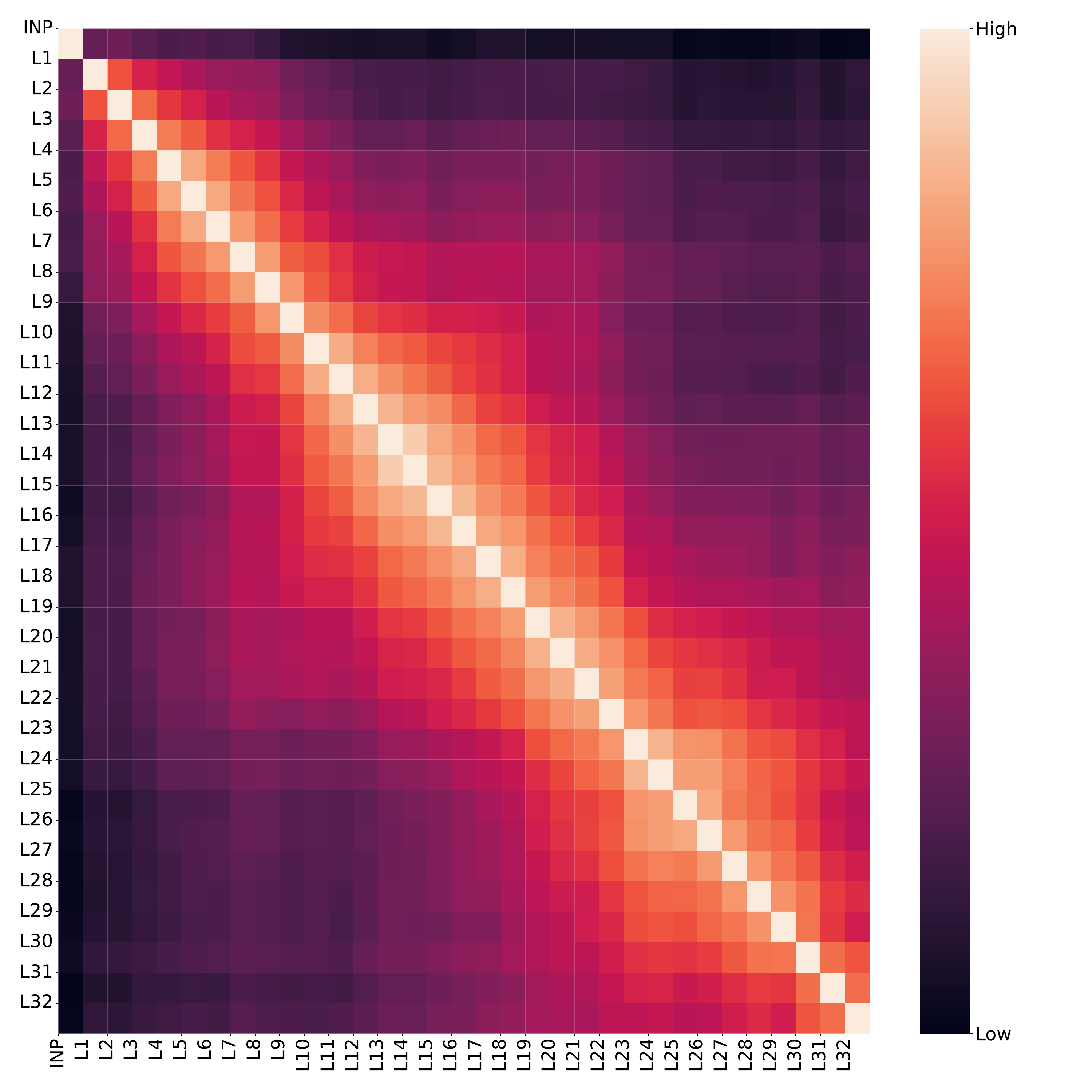}
        \caption{final}
    \end{subfigure}
    \caption{Inter-layer similarity samples of OLMo2 7B for each checkpoint on MMLU. Bright color represents high similarity, while dark color represents low similarity.}
    \label{fig:sim_mmlu_olmo}
\end{figure}

\section{Layer Similarity Pattern Consistency Across Samples}
\label{appendix:clustering_consistency}
We apply spectral clustering~\citep{shi-malik, vonLuxburg2007} to partition layers into several clusters and evaluate whether the ``islands'' patterns are consistent across samples.
For the resulting clusters, we compute the Adjusted Rand Index (ARI)~\citep{Hubert1985ComparingP} to measure cluster similarity between samples and employ the conductance~\citep{SINCLAIR198993} to assess the independence of each cluster.

\paragraph{Clustering evaluation metrics}
We formally define the conductance metric as follows.
Given $l$ layers of a model,
let $\sV = \left\{\ell_1,\dots,\ell_l\right\}$ be the set of nodes, $\sL$ be the set of layers in a resulting cluster, and $\overline{\sL}$ be the complement.
The conductance $\varphi$ of the cluster is defined as:
\begin{equation}
    \varphi(\sL) = \frac{\displaystyle a(\sL,\,\overline{\sL})}{\min\left(\text{vol}(\sL),\, \text{vol}(\overline{\sL})\right)}
    \;,
\end{equation}
where
\begin{equation}
    \text{vol}(\sA) = \displaystyle a(\sA,\,\sV)\,,\quad \displaystyle a(\sA,\sB) = \sum_{i \in \sA,\, j \in \sB}\text{score}_{*}(i, j)\,.
\end{equation}
Lower conductance means a sharper border between clusters.

\paragraph{Result}
Table~\ref{tab:ari_conductance} shows that clustering is consistent across samples for a given $k$ for each metric, model, and dataset, and that $k=2$ and $k=3$ form sharp clusters for all metrics except Cos-Struct.

\begin{table*}[t]
\caption{Adjusted Rand Index (ARI) and Conductance (Cond.) on Llama 3.1 8B and Qwen2.5 7B. Bold denotes the best performance within each method.}
\label{tab:ari_conductance}
\centering
\resizebox{0.75\textwidth}{!}{
\begin{tabular}{llrrrrrrrr}
\toprule
 \multirow[*]{3}{*}{Method} & \multirow[*]{3}{*}{k} & \multicolumn{4}{c}{Llama3.1 8B} & \multicolumn{4}{c}{Qwen2.5 7B} \\
 &  & \multicolumn{2}{c}{MMLU} & \multicolumn{2}{c}{CMMLU} & \multicolumn{2}{c}{MMLU} & \multicolumn{2}{c}{CMMLU} \\
 &  & ARI $\uparrow$ & Cond. $\downarrow$ & ARI $\uparrow$ & Cond. $\downarrow$ & ARI $\uparrow$ & Cond. $\downarrow$ & ARI $\uparrow$ & Cond. $\downarrow$ \\
\midrule
\multirow[*]{3}{*}{CKA} & 2 & $.84_{\pm .171}$ & $.70_{\pm .046}$ & $.60_{\pm .256}$ & $\mathbf{.60_{\pm .075}}$ & $.92_{\pm .098}$ & $.86_{\pm .019}$ & $\mathbf{.88_{\pm .175}}$ & $\mathbf{.43_{\pm .027}}$ \\
 & 3 & $\mathbf{1.0_{\pm .016}}$ & $\mathbf{.64_{\pm .001}}$ & $\mathbf{.67_{\pm .177}}$ & $.61_{\pm .024}$ & $.92_{\pm .077}$ & $\mathbf{.81_{\pm .017}}$ & $.53_{\pm .267}$ & $.50_{\pm .057}$ \\
 & 4 & $.97_{\pm .039}$ & $.74_{\pm .001}$ & $.65_{\pm .209}$ & $.70_{\pm .017}$ & $\mathbf{.99_{\pm .060}}$ & $.87_{\pm .007}$ & $.62_{\pm .185}$ & $.57_{\pm .028}$ \\
\cmidrule(lr){1-10}
\multirow[*]{3}{*}{Cos-Base} & 2 & $1.0_{\pm .017}$ & $\mathbf{.48_{\pm .001}}$ & $\mathbf{.97_{\pm .051}}$ & $\mathbf{.47_{\pm .010}}$ & $.83_{\pm .220}$ & $\mathbf{.64_{\pm .046}}$ & $.82_{\pm .144}$ & $\mathbf{.53_{\pm .022}}$ \\
 & 3 & $\mathbf{1.0_{\pm .000}}$ & $.63_{\pm .000}$ & $.95_{\pm .054}$ & $.64_{\pm .001}$ & $\mathbf{.99_{\pm .040}}$ & $.65_{\pm .000}$ & $\mathbf{.95_{\pm .065}}$ & $.66_{\pm .003}$ \\
 & 4 & $\mathbf{1.0_{\pm .000}}$ & $.73_{\pm .000}$ & $.96_{\pm .045}$ & $.73_{\pm .000}$ & $.98_{\pm .078}$ & $.74_{\pm .000}$ & $.84_{\pm .132}$ & $.75_{\pm .001}$ \\
\cmidrule(lr){1-10}
\multirow[*]{3}{*}{Cos-Struct} & 2 & $\mathbf{1.0_{\pm .000}}$ & $1.0_{\pm .000}$ & $\mathbf{1.0_{\pm .000}}$ & $\mathbf{.92_{\pm .009}}$ & $\mathbf{1.0_{\pm .000}}$ & $.70_{\pm .000}$ & $.75_{\pm .257}$ & $\mathbf{.92_{\pm .044}}$ \\
 & 3 & $.97_{\pm .048}$ & $\mathbf{.74_{\pm .007}}$ & $.85_{\pm .175}$ & $.96_{\pm .040}$ & $.98_{\pm .074}$ & $\mathbf{.67_{\pm .013}}$ & $.75_{\pm .249}$ & $.97_{\pm .039}$ \\
 & 4 & $.75_{\pm .236}$ & $.80_{\pm .036}$ & $.97_{\pm .147}$ & $.98_{\pm .035}$ & $\mathbf{1.0_{\pm .000}}$ & $.73_{\pm .000}$ & $\mathbf{1.0_{\pm .000}}$ & $.94_{\pm .001}$ \\
\cmidrule(lr){1-10}
\multirow[*]{3}{*}{Tree-Edit} & 2 & $.92_{\pm .097}$ & $\mathbf{.21_{\pm .012}}$ & $.28_{\pm .297}$ & $\mathbf{.44_{\pm .090}}$ & $\mathbf{.90_{\pm .097}}$ & $\mathbf{.50_{\pm .015}}$ & $.36_{\pm .375}$ & $\mathbf{.26_{\pm .100}}$ \\
 & 3 & $\mathbf{.95_{\pm .066}}$ & $.36_{\pm .013}$ & $.29_{\pm .227}$ & $.47_{\pm .051}$ & $.89_{\pm .175}$ & $.60_{\pm .014}$ & $\mathbf{.54_{\pm .211}}$ & $.34_{\pm .075}$ \\
 & 4 & $.88_{\pm .115}$ & $.48_{\pm .013}$ & $\mathbf{.39_{\pm .194}}$ & $.54_{\pm .034}$ & $.87_{\pm .085}$ & $.66_{\pm .009}$ & $.45_{\pm .189}$ & $.46_{\pm .060}$ \\
\cmidrule(lr){1-10}
\multirow[*]{3}{*}{Edge-Edit} & 2 & $\mathbf{.99_{\pm .032}}$ & $\mathbf{.45_{\pm .014}}$ & $.48_{\pm .336}$ & $.57_{\pm .132}$ & $.56_{\pm .383}$ & $\mathbf{.59_{\pm .043}}$ & $.44_{\pm .426}$ & $\mathbf{.56_{\pm .039}}$ \\
 & 3 & $.91_{\pm .068}$ & $.57_{\pm .014}$ & $\mathbf{.93_{\pm .083}}$ & $\mathbf{.55_{\pm .013}}$ & $\mathbf{.93_{\pm .064}}$ & $.60_{\pm .012}$ & $\mathbf{.87_{\pm .090}}$ & $.59_{\pm .025}$ \\
 & 4 & $.86_{\pm .142}$ & $.68_{\pm .016}$ & $.80_{\pm .132}$ & $.64_{\pm .013}$ & $.79_{\pm .213}$ & $.69_{\pm .008}$ & $.65_{\pm .209}$ & $.67_{\pm .015}$ \\
\bottomrule
\end{tabular}
}
\end{table*}

\section{Frequent Subtree Mining}
\label{appendix:freqt}

To find what structures are built on the islands, we perform frequent subtree mining~\citep{abe-etal-optimized, zaki} on an instance in MMLU as a case study.
We use FREQT\footnote{\url{http://chasen.org/~taku/software/freqt/}} to run frequent subtree mining and extract ordered subtrees of eight nodes from a set of ordered trees.
We extract subtrees that appeared at least twice across the collection of trees constructed for each layer, i.e., subtrees observed in a minimum of two layers.
The input tokens and their indices are provided in the supplemental materials. %

\paragraph{Frequent subtrees in islands.}
Table~\ref{tab:freqt} shows that both models construct subtrees of depth eight that consist of continuous tokens in the middle and higher layers, and similar patterns appear in the later positions in the higher layers.
This subtree emergence pattern suggests that models construct subtrees sequentially from left to right, whereby initially formed structural representations subsequently become obsolete in the process of residual streams.
The islands are consistent with the phases observed in the intrinsic dimensionality analysis~\cite{cheng2025emergence}, which revealed that models process linguistic information (e.g., syntax and semantics) in high-intrinsic-dimensionality phases.
Moreover, for MMLU, Qwen2.5 7B relates choice tokens (e.g., ``A'') to each other in the first few layers and does not reuse these structures in later layers.

\paragraph{Frequent subtrees across non-adjacent layers.}
Frequent subtree patterns also reveal the reuse of structures in non-adjacent layers.
Table~\ref{tab:subtree_nonadjacent} shows such instances of frequent subtrees that did not appear in several layers.
The reuse of structures between adjacent layers suggests that those layers cooperate with each other during inference, as discussed in attention heads~\citep{wang2023interpretability}, and our analysis suggests that \structlensname\ reveals non-adjacent layer collaboration in terms of internal structures.

Our findings also show that the frequent subtree patterns are different between the two models, potentially influenced by training data, indicating that bottom-up analysis approaches are appropriate to assess the internal structure of LMs.

\begin{table}[t]
    \centering
    \caption{Frequent subtree samples. The tree is represented as a strict S-expression. The number before ``\_'' denotes the index of the token in the input. We replaced ``('' and ``)'' in the input tokens with ``['' and ``]'' to run subtree mining correctly.}
    \label{tab:freqt}
    \begin{subtable}[t]{0.8\linewidth}
    \centering
    \caption{Llama3.1 8B, MMLU}
    \resizebox{\linewidth}{!}{
    \begin{tabular}{l}
    \toprule
        Subtree \\
        \midrule
        (15\_,(25\_.(40\_.(47\_.(114\_.(121\_approximately(1024\_approximately))))))(37\_]))\\
        Layers: 1, 2, 3\\
        \midrule
        (1\_The(2\_following(3\_are(4\_multiple(5\_choice(6\_questions(7\_about(8\_college))))))))\\
        Layers: 4, 5, 7, 8, 9, 10, 11, 12, 13, 14, 15, 16 \\
        \midrule
        (520\_io(521\_Is(522\_chem(523\_ic(524\_Heart(525\_Disease(530\_HD)(531\_])))))))\\
        Layers: 19, 20, 21, 22, 23, 24, 25, 26, 27, 28, 29, 30, 31, 32\\
    \bottomrule
    \end{tabular}
    }
    \end{subtable}
    \begin{subtable}[t]{0.8\linewidth}
        \centering
        \caption{Qwen2.5 7B, MMLU}
        \resizebox{\linewidth}{!}{
        \begin{tabular}{l}
    \toprule
        Subtree\\
        \midrule
        (35\_[A(40\_[B(47\_[C(53\_[D(142\_[D(246\_[D(324\_[D))))))(101\_[A)) \\
        Layers: 0, 1, 2, 3, 4\\
        \midrule
        (27\_side(28\_effect(29\_of(33\_is(36\_](37\_muscle(49\_muscle))(41\_])))))) \\
        Layers: 8, 9, 10, 11, 12, 13, 14, 15, 16, 17, 18, 19, 20 \\
        \midrule
        (1013\_while(1014\_the(1015\_heart(1016\_rate(1017\_[(1018\_the(1019\_number(1020\_of)))))))) \\
        Layers: 21, 22, 23, 24, 25, 26, 27, 28\\
    \bottomrule
    \end{tabular}
    }
    \end{subtable}
    \begin{subtable}[t]{0.8\linewidth}
        \centering
        \caption{Llama3.1 8B, Multinews}
        \resizebox{\textwidth}{!}{
    \begin{tabular}{l}
    \toprule
        Subtree \\
        \midrule
        (16\.(28\_Philadelphia(324\_Philadelphia(360\_Philadelphia(792\_Philadelphia(1961\_Philadelphia (2639\_Philadelphia))))))(39\_Nov))\\
        Layers: 0, 2, 3\\
        \midrule
        (1\_You(2\_are(3\_given(4\_several(5\_news(15\_news(18\_News(19\_:))))))))\\
        Layers: 4, 5, 6, 7, 8, 9, 10, 11, 12, 13, 14, 15, 16 \\
        \midrule
        (277\_police(288\_Police(1138\_Police(1287\_Police(1628\_Police(1698\_police(1699\_ultimately)) (2202\_police))))))\\
        Layers: 18, 19, 20, 21, 22, 23, 24, 25, 26, 27, 28, 29, 30, 31, 32\\
    \bottomrule
    \end{tabular}
    }
    \end{subtable}
    \begin{subtable}[t]{0.8\linewidth}
    \centering
    \caption{Qwen2.5 7B, Multinews}
    \resizebox{\textwidth}{!}{
    \begin{tabular}{l}
    \toprule
        Subtree \\
        \midrule
        (110\_But(208\_For(242\_They(449\_The(1783\_The(2183\_The(2507\_The))))))(1899\_But))\\
        Layers: 2, 3, 4, 5, 6\\
        \midrule
        (391\_majority(392\_of(393\_people(394\_participating(395\_in(396\_this(397\_movement))))(398\_have))))\\
        Layers: 8, 9, 10, 11, 12, 13, 14, 15, 16, 17, 18, 19 \\
        \midrule
        (111\_the(112\_expected(113\_police(114\_eviction(115\_had(116\_not(117\_happened(118\_by))))))))\\
        Layers: 22, 23, 24, 25, 26, 27, 28\\
    \bottomrule
    \end{tabular}
    }
    \end{subtable}
\end{table}

\begin{table}[t!]
    \centering
    \caption{Frequent subtree patterns found in non-adjacent layers. This shows the top two patterns with the longest periods during which the structure was not used. The absence interval represents the maximum number of layers between when a structure observed at a layer disappears and when it reappears. For example, when a structure is observed in layers 3, 6, and 8, the absence interval is 3, corresponding to the number of layers from layer 3 to layer 6.}
    \label{tab:subtree_nonadjacent}
    \resizebox{0.8\linewidth}{!}{
    \begin{tabular}{l}
    \toprule
    Subtree \\
    \midrule
    Llama3.1 8B\\
    \midrule
       (72\_,(86\_.(414\_.(471\_…(570\_…(788\_…(862\_…))))(505\_.))))  \\
       Layers: 1, 5, 32 \\
       Absence interval: 27 \\
    \midrule
        (14\_A(36\_[A(41\_[B(48\_[C(54\_[D(143\_[D(1352\_[D)))(131\_[C))))) \\
        Layers: 1, 2, 3, 4, 5, 6, 7, 8, 29 \\
        Absence interval: 21\\
    \midrule
    Qwen2.5 7B \\
    \midrule
    (393\_una(477\_sauna(566\_sauna(585\_sauna(633\_sauna(752\_sauna(790\_sauna))))(674\_sauna)))) \\
    Layers: 1, 2, 17, 18, 19, 20 \\
    Absence interval: 15 \\
    \midrule
    (407\_by(408\_short(409\_-term(410\_passive(411\_exposure(412\_to(413\_extreme(414\_heat)))))))) \\
    Layers: 10, 11, 21, 22, 23, 24, 25, 26, 27, 28 \\
    Absence interval: 10\\
    \bottomrule
    \end{tabular}
    }
\end{table}

\section{Models' behavior and Structural Transformation}
\label{appendix:logitlens-structlens}
We investigate the ``islands'' phenomenon of Edge-Edit, reported in Section~\ref{sec:layer-sim-experiment}, from a model's behavioral perspective at each layer, using the logit lens.
Focusing on the final token outputs of the logit lens in each layer in Figure~\ref{fig:logitlens_mmlu}, Llama3.1 8B demonstrates instruction-following behavior of selecting A/B/C/D for MMLU beginning at layer 18, and similar trends are observed for Qwen2.5 7B at layer 22.
Table~\ref{tab:cluster-sample} examines whether this explicit transition is observed on the border of islands by spectral clustering~\citep{shi-malik, vonLuxburg2007}. %
For Llama3.1 8B, layer 18 is the critical transition point, while layer 21 is the corresponding boundary for Qwen2.5 7B.
These results demonstrate that structural transformations observed in \structlensname\ in higher layers are related to the transition of token predictions. %

\begin{figure}[t]
    \centering
    \begin{subfigure}[b]{0.49\linewidth}
        \includegraphics[width=\textwidth,clip,trim={0 0 24em 0}]{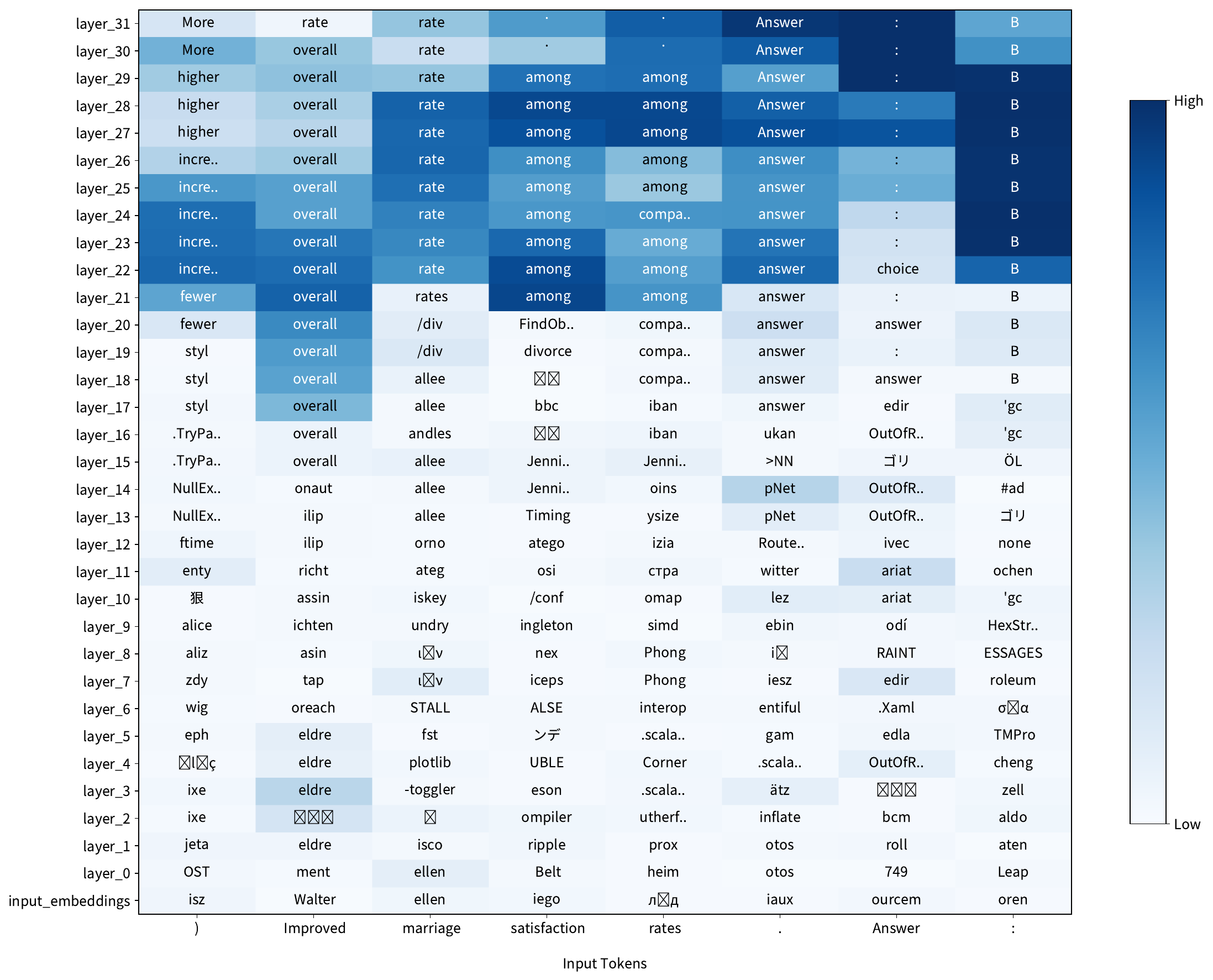}
        \caption{Llama3.1 8B}
    \end{subfigure}
    \begin{subfigure}[b]{0.49\linewidth}
        \includegraphics[width=\textwidth,clip,trim={0 0 24em 0}]{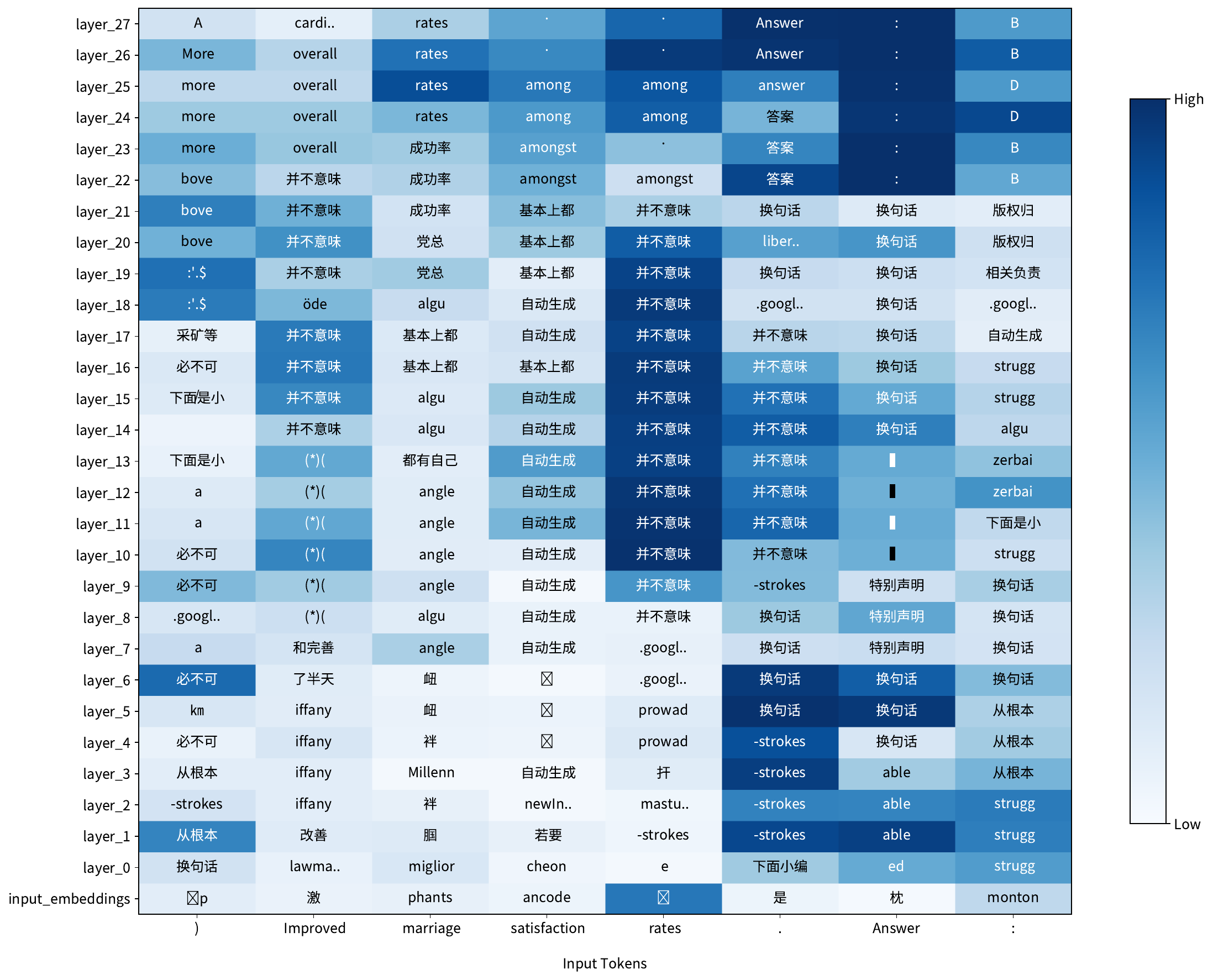}
        \caption{Qwen2.5 7B}
    \end{subfigure}
    \caption{Logit lens visualization on MMLU. We visualize the token predictions for each of the last eight tokens in the input. The lower rows represent predictions from the lower layers, while the upper rows show predictions from the higher layers. Color intensity represents prediction probability.}
    \label{fig:logitlens_mmlu}
\end{figure}

\begin{table}[!t]
    \centering
    \caption{Sample of clustering results of Edge-Edit for an instance of MMLU ($k = 3$). Layer 0 indicates the input embeddings. Cond. denotes Conductance (see Appendix~\ref{appendix:clustering_consistency}).}
    \label{tab:cluster-sample}
    \resizebox{0.8\linewidth}{!}{
    \begin{tabular}{lp{12em}rp{12em}r}
    \toprule
         & \multicolumn{2}{c}{Llama3.1 8B} & \multicolumn{2}{c}{Qwen2.5 7B} \\
         & Layers & Cond. & Layers & Cond. \\
         \midrule
        Cluster 1 & 0, 1, 2, 3. & .39 & 0, 1, 2, 3, 4, 5, 6, 7. & .64 \\
        Cluster 2 & 4, 5, 6, 7, 8, 9, 10, 11, 12, 13, 14, 15, 16, 17. & .76 & 8, 9, 10, 11, 12, 13, 14, 15, 16, 17, 18, 19, 20. & .47 \\
        Cluster 3 & 18, 19, 20, 21, 22, 23, 24, 25, 26, 27, 28, 29, 30, 31, 32. & .44 & 21, 22, 23, 24, 25, 26, 27, 28. & .54 \\
    \bottomrule
    \end{tabular}
    }
\end{table}

\begin{figure}[t]
    \centering
    \includegraphics[width=0.8\textwidth]{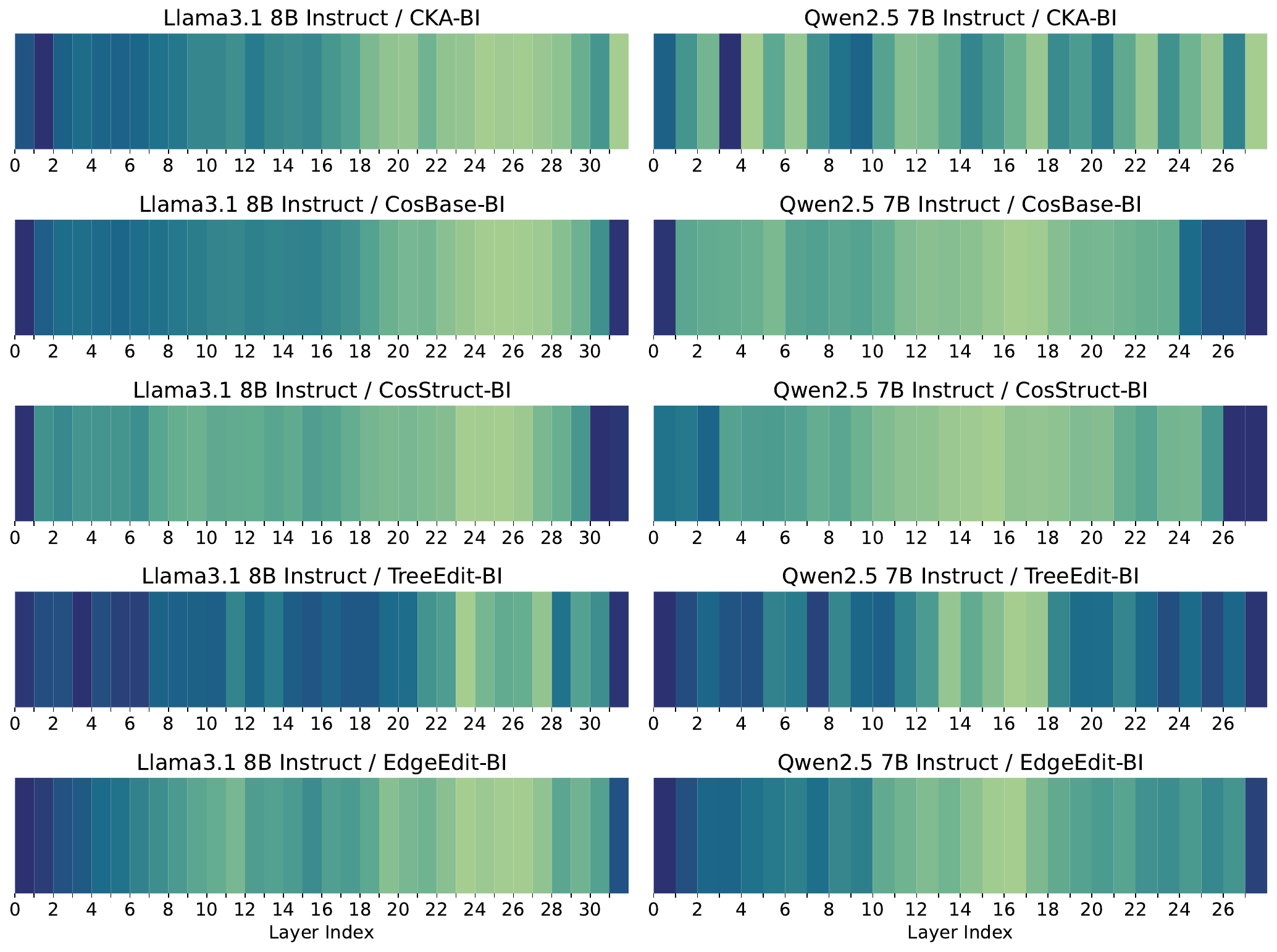}
    \caption{Visualization of layer importance. Lighter colors represent less importance, while darker colors represent greater importance.}
    \label{fig:layer_importance}
\end{figure}

\begin{table}[ht!]
\centering
\caption{Pruning results of multiple-choice QA tasks. Each value denotes the accuracy or the correctness of an answer in a particular task, with higher values indicating better performance. Values denoted by $\dagger$ are statistically significant ($p < 0.05$) compared to CosBase-BI.}
\label{tab:result-pruning-pretrained}
\resizebox{0.7\linewidth}{!}{
\begin{tabular}{lllrrrr}
\toprule
Model & Metric & Layers removed & ARC-Easy & MMLU & CMMLU & Avg. \\
\midrule
\multirow[*]{7}{*}{Llama-3.1-8B} & Dense & - & 84.6\:\: & 65.3\:\: & 52.9\:\: & 67.6 \\
\cmidrule{2-7} \vspace{-1em} \\
  & CosBase-BI & 24 25 26 27 & \underline{76.8\:\:} & 62.0\:\: & 49.0\:\: & 62.6 \\
  & CKA-BI & 24 26 27 31 & 72.4$^{\dagger}$ & 62.3\:\: & 49.4\:\: & 61.4 \\
\cdashline{2-7} \vspace{-.8em} \\
  & CosStruct-BI & 23 24 25 26 & \textbf{77.8\:\:} & \underline{64.5$^{\dagger}$} & \underline{51.2$^{\dagger}$} & \textbf{64.5} \\
  & TreeEdit-BI & 23 24 26 27 & 76.8\:\: & \textbf{64.6$^{\dagger}$} & \textbf{51.9$^{\dagger}$} & \underline{64.4} \\
  & EdgeEdit-BI & 23 24 25 26 & \textbf{77.8\:\:} & \underline{64.5$^{\dagger}$} & \underline{51.2$^{\dagger}$} & \textbf{64.5} \\
\midrule
\multirow[*]{7}{*}{Llama-3.1-8B-Instruct} & Dense & - & 85.8\:\: & 68.4\:\: & 55.9\:\: & 70.0 \\
\cmidrule{2-7} \vspace{-1em} \\
  & CosBase-BI & 24 25 26 27 & \underline{78.3\:\:} & 66.8\:\: & 54.4\:\: & 66.5 \\
  & CKA-BI & 24 25 26 31 & 75.3$^{\dagger}$ & 66.9\:\: & 54.2\:\: & 65.5 \\
\cdashline{2-7} \vspace{-.8em} \\
  & CosStruct-BI & 23 24 25 26 & \textbf{79.3\:\:} & \textbf{67.8$^{\dagger}$} & \underline{55.1$^{\dagger}$} & \textbf{67.4} \\
  & TreeEdit-BI & 23 24 26 27 & 78.1\:\: & \underline{66.9\:\:} & \textbf{55.8$^{\dagger}$} & \underline{66.9} \\
  & EdgeEdit-BI & 23 24 25 26 & \textbf{79.3\:\:} & \textbf{67.8$^{\dagger}$} & \underline{55.1$^{\dagger}$} & \textbf{67.4} \\
\midrule
\multirow[*]{7}{*}{Qwen-2.5-7B} & Dense & - & 86.5\:\: & 74.3\:\: & 82.8\:\: & 81.2 \\
\cmidrule{2-7} \vspace{-1em} \\
  & CosBase-BI & 15 16 17 & \textbf{82.1\:\:} & 56.0\:\: & 56.0\:\: & 64.7 \\
  & CKA-BI & 4 22 27 & 79.2$^{\dagger}$ & \textbf{70.2$^{\dagger}$} & \textbf{77.7$^{\dagger}$} & \textbf{75.7} \\
\cdashline{2-7} \vspace{-.8em} \\
  & CosStruct-BI & 13 14 15 & \underline{81.9\:\:} & 54.7$^{\dagger}$ & 54.8$^{\dagger}$ & 63.8 \\
  & TreeEdit-BI & 12 25 26 & 71.7$^{\dagger}$ & \underline{62.9$^{\dagger}$} & \underline{72.5$^{\dagger}$} & \underline{69.0} \\
  & EdgeEdit-BI & 12 15 26 & 76.3$^{\dagger}$ & 60.8$^{\dagger}$ & 65.5$^{\dagger}$ & 67.5 \\
\midrule
\multirow[*]{7}{*}{Qwen-2.5-7B-Instruct} & Dense & - & 87.6\:\: & 74.3\:\: & 81.1\:\: & 81.0 \\
\cmidrule{2-7} \vspace{-1em} \\
  & CosBase-BI & 15 16 17 & \underline{83.7\:\:} & 56.1\:\: & 56.8\:\: & 65.5 \\
  & CKA-BI & 4 25 27 & 80.1$^{\dagger}$ & \textbf{70.0$^{\dagger}$} & \textbf{75.7$^{\dagger}$} & \textbf{75.3} \\
\cdashline{2-7} \vspace{-.8em} \\
  & CosStruct-BI & 13 14 15 & 82.8\:\: & 55.7\:\: & 55.9\:\: & 64.8 \\
  & TreeEdit-BI & 13 16 17 & 82.5$^{\dagger}$ & \underline{59.7$^{\dagger}$} & \underline{63.1$^{\dagger}$} & \underline{68.4} \\
  & EdgeEdit-BI & 14 15 16 & \textbf{84.2\:\:} & 53.4$^{\dagger}$ & 54.0$^{\dagger}$ & 63.9 \\
\bottomrule
\end{tabular}
}
\end{table}

\section{Layer Pruning Results}
\label{appendix:pruning-result-pretrained}
To evaluate the layer pruning performance of each metric, i.e., CosBase-BI, CKA-BI, CosStruct-BI, TreeEdit-BI, and EdgeEdit-BI, for pre-trained models, we evaluate them on multiple-choice QA tasks that require generating only a single choice.
Figure~\ref{fig:layer_importance} shows the layer importance of each layer.
Layer pruning results are shown in Table~\ref{tab:result-pruning-pretrained}.
These results indicate that the best metric for layer pruning is consistent within a model family.

\clearpage
\newpage

\end{document}